%% file: main.tex
\documentclass{article}
\usepackage{iclr2027_conference,times}
\usepackage[T1]{fontenc} %

\input{math_commands.tex}

\usepackage{amsmath}
\usepackage{amssymb}
\usepackage{amsthm}
\usepackage{booktabs}
\usepackage{graphicx}
\usepackage{wrapfig}
\usepackage{flafter}
\usepackage{placeins}
\usepackage{needspace}
\usepackage{multirow}
\usepackage{arydshln}
\usepackage{xcolor}
\usepackage{pifont}
\usepackage{colortbl}
\usepackage{listings}
\usepackage{algorithm}
\usepackage{algorithmic}
\usepackage{hyperref}
\usepackage{etoc} %
\usepackage{url}
\usepackage{tikz}
\usepackage[skins]{tcolorbox}
\usetikzlibrary{arrows.meta,positioning,fit,calc,backgrounds}

\definecolor{bridgebluebase}{RGB}{55,165,255}
\colorlet{bilevelblue}{bridgebluebase!20!white}
\definecolor{scoreaccent}{RGB}{0,150,95}
\newtcolorbox{findingbox}[1]{%
  colback=bridgebluebase!5!white,
  colframe=bridgeaction,
  colbacktitle=bridgeaction,
  coltitle=white,
  fonttitle=\fontsize{9.5}{11}\selectfont\bfseries,
  fontupper=\normalsize,
  title={\vphantom{\normalsize #1}#1},
  toptitle=3pt,bottomtitle=3pt,
  boxrule=0.6pt,arc=1mm,
  boxsep=0pt,left=7pt,right=7pt,top=5pt,bottom=5pt,
  before skip=8pt,after skip=7pt,
  before upper={\setlength{\parindent}{0pt}\setlength{\parskip}{0pt}}
}
\newcommand{\scorecell}[4]{%
  \pgfmathparse{min(80,max(0,80*(#3-#1)/max(#2-#1,0.0001)))}%
  \edef\scorefill{\noexpand\cellcolor{scoreaccent!\pgfmathresult!white}}%
  \scorefill#4%
}
\newcommand{\scorebox}[4]{%
  \begingroup
  \pgfmathparse{min(80,max(0,80*(#3-#1)/max(#2-#1,0.0001)))}%
  \setlength{\fboxsep}{0pt}%
  \colorbox{scoreaccent!\pgfmathresult!white}{\kern1pt\strut#4\kern1pt}%
  \endgroup
}
\definecolor{tagthink}{RGB}{35,75,210}
\definecolor{tagsearch}{RGB}{0,140,210}
\definecolor{taginfo}{RGB}{245,130,32}
\definecolor{taganswer}{RGB}{200,35,95}
\definecolor{casegreen}{RGB}{224,246,232}
\definecolor{casered}{RGB}{252,232,232}
\definecolor{bridgeaction}{RGB}{48,98,180}
\definecolor{bridgeretrieval}{RGB}{222,124,45}
\definecolor{bridgereward}{RGB}{57,145,92}
\definecolor{bridgepanel}{RGB}{247,248,250}
\definecolor{checkgreen}{RGB}{0,140,0}
\definecolor{crossred}{RGB}{210,0,40}

\newcommand{\xmark}{{\color{crossred}\ding{55}}}

\lstdefinestyle{trajectory}{
  basicstyle=\ttfamily\fontsize{9pt}{9.5pt}\selectfont,
  breaklines=true,
  columns=fullflexible,
  keepspaces=true,
  showstringspaces=false,
  aboveskip=0.04em,
  belowskip=0.04em,
  extendedchars=true,
  escapeinside={(*@}{@*)},
  alsoletter={<>/},
  emph={[1]<think>,</think>},
  emphstyle={[1]\color{tagthink}},
  emph={[2]<search>,</search>},
  emphstyle={[2]\color{tagsearch}},
  emph={[3]<information>,</information>},
  emphstyle={[3]\color{taginfo}},
  emph={[4]<answer>,</answer>},
  emphstyle={[4]\color{taganswer}},
  literate={—}{{---}}1 {–}{{-}}1 {…}{{...}}3 {æ}{{ae}}2 {í}{{\'i}}1
           {ü}{{\"u}}1 {Ü}{{\"U}}1 {é}{{\'e}}1 {ø}{{o}}1 {ä}{{\"a}}1
           {東}{{Dong}}4 {野}{{ye}}2 {家}{{jia}}3 {族}{{zu}}2 {大}{{Da}}2
           {宗}{{zong}}4 {世}{{shi}}3 {系}{{xi}}2 {翁}{{Weng}}4 {成}{{Cheng}}5
           {湯}{{Tang}}4 {乙}{{Yi}}2
}

\newcommand{\caseRule}{\par\vspace{0.03em}\noindent\rule{\linewidth}{0.35pt}\par\vspace{0.03em}}
\newcommand{\hit}[1]{{\color{red}\texttt{#1}}}
\newcommand{\caseHeader}[2]{%
  \par\vspace{0.03em}%
  {\setlength{\fboxsep}{1.5pt}\noindent\colorbox{#1}{\parbox{\dimexpr\linewidth-2\fboxsep\relax}{\textbf{#2}}}}%
  \par\vspace{0.03em}}
\newcommand{\caseCaption}[2]{%
  \par\vspace{0.28em}\refstepcounter{table}\label{#1}%
  \noindent\parbox{\linewidth}{\centering\textbf{Table~\thetable:} #2}%
  \par\vspace{0.04em}}
\newenvironment{caseStudy}{%
  \par\vspace{0.25em}\noindent\begin{minipage}{\linewidth}%
}{%
  \end{minipage}\par\vspace{0.35em}%
}


\newcommand{\method}{BRIDGE}
\newcommand{\methodfull}{BRIDGE: Bilevel Retrieval-Credit-Aware Agentic Reinforcement Learning}

\newcommand{\methodlogo}{%
  \raisebox{-0.11em}{\includegraphics[height=0.90em]{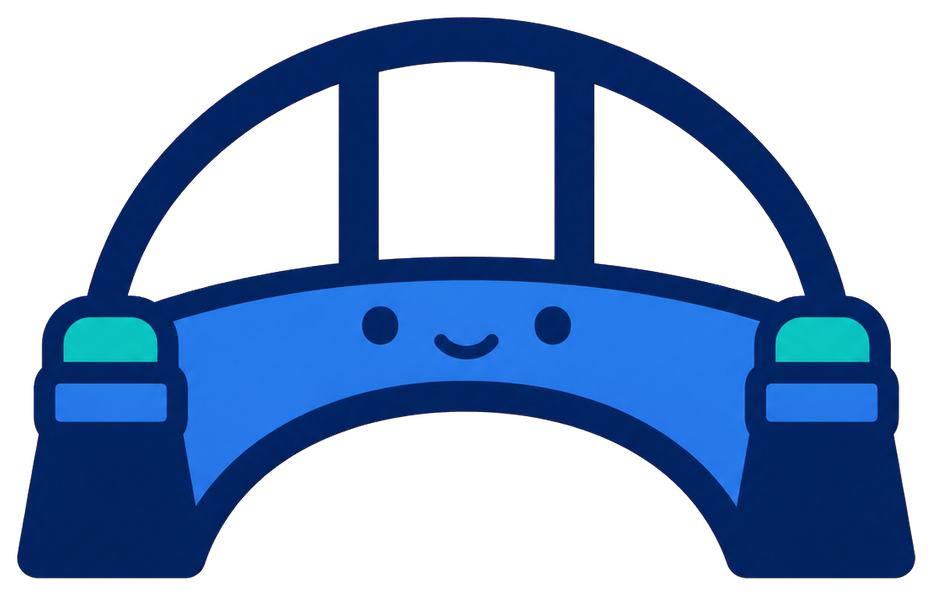}}%
}

\title{\texorpdfstring{%
  \centering
  \methodlogo\hspace{0.24em}BRIDGE: Bilevel Retrieval-Credit-Aware\\
  Agentic Reinforcement Learning%
}{\methodfull}}

\author{
Quan Xiao$^{1}$ \quad
Mingda Liu$^{2}$ \quad
Gaowen Liu$^{3}$ \quad
Katsuki Fujisawa$^{2}$ \quad
Tianyi Chen$^{1}$ \\[3pt]
$^{1}$Department of Electrical and Computer Engineering, Cornell University \\
$^{2}$Institute of Integrated Research, Institute of Science Tokyo \\
$^{3}$Cisco Research \\[2pt]
{\small\texttt{\{qx232,tianyi.chen\}@cornell.edu \quad puppetsasya@gmail.com}} \\
{\small\texttt{gaoliu@cisco.com \quad fujisawa.k.2110@m.isct.ac.jp}} \\[8pt]
{\normalfont\small\textbf{Project website:} \url{https://jenniferquanxiao.github.io/BRIDGE/}}
}

\newcommand{\corpus}{\mathcal{C}}
\newcommand{\Dtrain}{\mathcal{D}_{\mathrm{tr}}}

\newcommand{\retr}{\rho_{\eta}}
\newcommand{\policy}{\pi_{\theta}}
\newcommand{\refpol}{\pi_{\mathrm{ref}}}
\newcommand{\traj}{\tau}
\newcommand{\oldpol}{\pi_{\mathrm{old}}}
\newcommand{\Jgrpo}{\mathcal{J}_{\mathrm{GRPO}}}
\newcommand{\grpoobj}{\Jgrpo}

\newcommand{\Jrl}{\mathcal{J}_{\mathrm{RL}}}
\newcommand{\Lrag}{\mathcal{L}_{\mathrm{RAG}}}
\newcommand{\Cand}{\mathcal{N}_k}

\newcommand{\etaold}{\eta_{\mathrm{old}}}

\theoremstyle{plain}
\newtheorem{assumption}{Assumption}
\newtheorem{theorem}{Theorem}
\newtheorem{corollary}{Corollary}
\theoremstyle{definition}

\newif\ifshowrevisions
\showrevisionstrue

\iclrfinalcopy
\begin{document}
\etocdepthtag.toc{main}

\maketitle
\lhead{Preprint}

\begin{abstract}
Agentic reinforcement learning (ARL) with verifiable rewards improves the ability of large language models (LLMs) to tackle knowledge-intensive tasks by learning to interleave search and reasoning. However, most existing ARL methods optimize only LLM-generated tokens and treat retrieved evidence as environment observations. {This creates an \emph{information-credit gap}: failures caused by missing or misleading evidence are attributed to the LLM policy rather than to the retriever, which motivates training the LLM and the retriever jointly.} In this paper, we show that retrieval and LLM policy learning are order-sensitive: adapting the retriever before optimizing the policy yields a larger reward gain than the reverse order. To preserve this hierarchy while allowing both components to co-adapt, we formulate retrieval-augmented agentic RL as a bilevel optimization problem. To solve it efficiently, we introduce \method{}, a memory-efficient first-order bilevel method motivated by a loss-landscape analysis of the RL and retrieval objectives. Across seven open-domain QA benchmarks, \method{} achieves the highest average accuracy with both 3B and 7B backbones, improving the multi-hop average over the strongest baseline by 9.6 and 3.4 EM points, respectively. It also achieves the best averaged answer accuracy and reasoning quality across medical QA benchmarks.
\end{abstract}

\section{Introduction}

Agentic reinforcement learning (ARL) trains large language models (LLMs) to
solve tasks through multi-turn interactions with external tools, which greatly
improves the ability of LLMs to handle complex, knowledge-intensive tasks \citep{xie2026quest}.
In retrieval-augmented question answering, an
agent alternates between reasoning, search actions, retrieved observations, and
answer generation, and the whole trajectory can be optimized via RL from a sparse
verifiable reward {\citep{jin2025searchr1,song2025r1searcher}}.
Through ARL, the LLM can learn when to invoke external retrieval and how to use the retrieved information without dense supervision \citep{jin2025searchr1,jiang2025verltool}.

However, most ARL methods optimize the LLM only on {LLM-generated} action tokens {(reasoning, search, and answer tokens), while retrieved information tokens serve as environment observations and are masked out of the RL loss \citep{jin2025searchr1,team2025kimi}}. This creates an \emph{information-credit gap}: when a rollout fails because the retrieved evidence is missing or misleading, {the RL update penalizes the LLM's actions around that evidence instead of improving retrieval, which} may degrade {the reasoning performance and generalization} of the learned policy. This gap is especially severe in multi-hop question answering \citep{trivedi2023interleaving}, where the {evidence is spread across} documents and the LLM can learn
reasoning only after retrieval exposes all of them.

\begin{figure}[t]
\centering
\begin{minipage}[t]{0.49\linewidth}
    \centering
    \includegraphics[width=\linewidth]{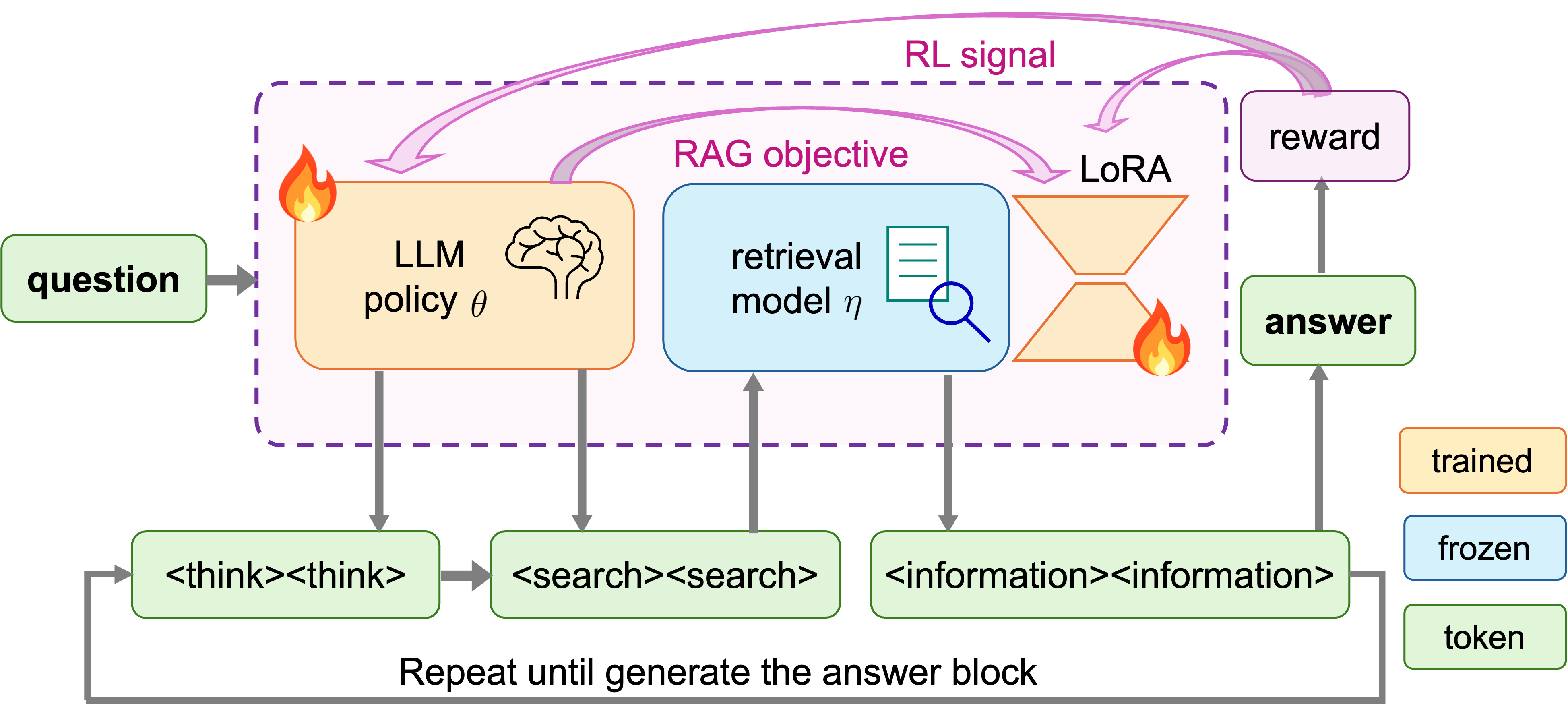}
    \vspace{0.1em}
    {\small \textbf{(a)} BRIDGE framework}
\end{minipage}
\hfill
\begin{minipage}[t]{0.5\linewidth}
    \centering
    \includegraphics[width=\linewidth]{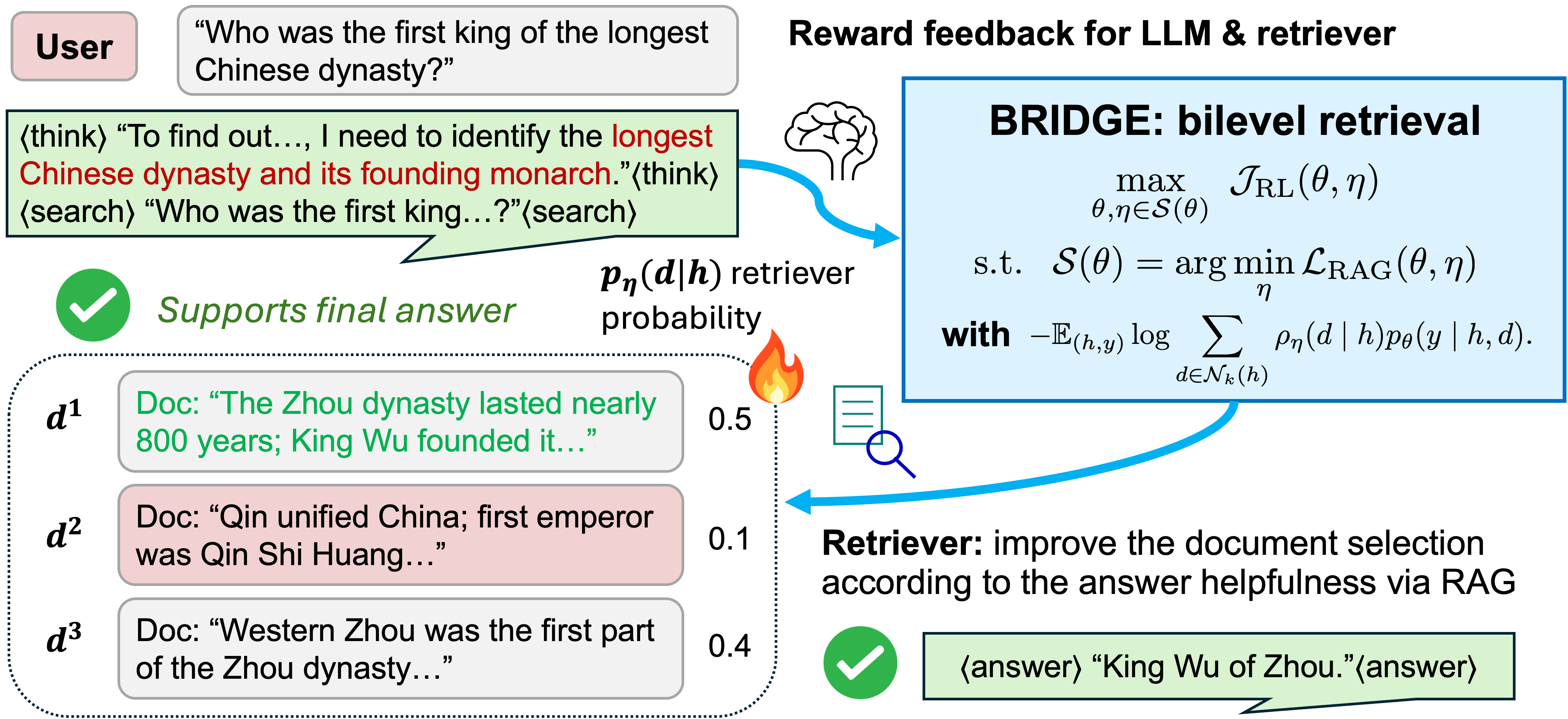}
    \vspace{0.1em}
    {\small \textbf{(b)} Bilevel optimization formulation}
\end{minipage}
\vspace{-0.2em}
\caption{\textbf{Overview of \method{}.}
\textbf{(a)} \method{} couples the LLM policy and the query-side retriever
adapter during multi-turn tool-use rollouts: the LLM generates reasoning and
search actions, the retriever supplies information, and the final reward
provides feedback to both components.
\textbf{(b)} The bilevel view treats answer-level RL as the upper objective and
RAG likelihood as the retrieval objective, so the retriever is updated toward
evidence that improves downstream answer generation.}
\label{fig:bridge_overview}
\vspace{0pt}
\end{figure}

Other lines of work either adapt the retriever to a frozen generator to improve document selection \citep{jiang2025s3}, or jointly optimize the retriever and the LLM policy \citep{chen2026improving,liu2026agenticr}. However, neither considers the ordered dependency between retrieval and policy learning. In this paper, we argue that {retrieval and LLM policy optimization have a natural hierarchical order: the retriever should first be optimized to make useful evidence available, and the LLM should then learn to search and reason with that evidence}.

Moreover, instead of {a one-shot retriever update} followed by RL training, we formulate retrieval-augmented agentic RL as a bilevel optimization problem that {co-adapts retrieval and LLM parameters hierarchically} throughout training. For each {LLM policy $\theta$, we seek the retrieval parameter $\eta$ that best supports answer generation, and then update the policy via RL on rollouts that use this retriever, with the verifiable outcome reward increasing} the likelihood of successful rollouts; see Figure~\ref{fig:bridge_overview}(b) for an overview.

{However, most existing bilevel methods are either second-order methods based on conjugate gradient, unrolling, or implicit differentiation, which are computationally expensive} \citep{chen2021closing,ji2021bilevel,lorraine2020optimizing,dontchev2014implicit}, or {first-order methods that must optimize and store two lower-level models to track the solutions of the original and the perturbed lower-level problems} \citep{kwon2023fully,shen2023penalty,zucchet2022beyond}, which is {memory-intensive for retriever adaptation. To reduce this memory burden, recent work treats the lower-level optimal value as nearly constant in $\theta$ and therefore avoids maintaining a second lower-level model} \citep{xiao2025ldc,shen2024seal}. However, we {show} that this assumption is violated in our setting, where the optimal retrieval objective value is sensitive to {the} LLM parameters.

Guided by a landscape analysis and motivated by an efficient bilevel solver \citep{jiang2026beyond}, we introduce {\method{}}, a memory-efficient first-order bilevel method for retrieval-augmented ARL. As shown in Figure \ref{fig:bridge_overview}, \method{} updates a query-side low-rank adaptation (LoRA) retriever adapter
with both trajectory-level RL credit at retrieval turns and a
RAG likelihood term that scores whether retrieved evidence supports the target
answer, while the LLM is updated {with trajectory-level RL credit on rollouts that contain} the information returned by the current {retriever}.

We summarize our main contributions below.

\textbf{C1)} We identify the information-credit gap caused by {a fixed retriever} in agentic RL and the {hierarchical order} between retrieval and LLM optimization. To {jointly} improve retrieval quality and the LLM's reasoning and search {capabilities}, we formulate multi-turn retrieval-augmented ARL as a bilevel optimization problem.

\textbf{C2)} {Guided by} landscape analysis, we propose \method{}, a memory-efficient first-order bilevel algorithm that updates the retrieval LoRA parameters using retrieval-turn RL credit and RAG likelihood, and then optimizes the LLM policy {with} RL credit. {Relative to the two-state first-order estimator, \method{} reduces the peak memory of each retriever update by $49.7\%$, while its gradient estimate stays closely aligned with the two-state estimate.}

\textbf{C3)} We validate \method{} {on} seven open-domain and five medical QA
benchmarks. {On the open-domain benchmarks, it improves the overall average EM over the strongest baseline, VT-Search, by 7.1 and 2.3 points with 3B and 7B backbones; on the medical benchmarks, it achieves the highest average answer accuracy and reasoning quality.}

\vspace{-0.1cm}
\subsection{Related Work}
\vspace{-0.1cm}

\noindent\textbf{Agentic RL with tool use.}
Tool-using LLM agents learn to interleave reasoning with external actions such
as search, code execution, or API calls. Early agentic prompting and tool-use
frameworks show that explicit reasoning-action traces can improve tool
interaction \citep{yao2023react, schick2023toolformer}. Recent ARL methods train
this behavior with verifiable rewards, including search-augmented RL systems
and general tool-use RL frameworks
\citep{li2025searcho1,jin2025searchr1,song2025r1searcher,jiang2025s3,jiang2025verltool,li2026dart}, and {have been extended to modular agentic systems},
where the planner is optimized via ARL {within} the multi-turn interaction
\citep{li2025agentflow}. Building on this foundation, closely related work extends LLM policy-side ARL to jointly optimize the LLM policy and {the} retriever \citep{chen2026improving,liu2026agenticr}. Our work complements these efforts by introducing a bilevel formulation that explicitly captures the {hierarchical} dependency {between} policy training and retriever adaptation, and preserves this hierarchy during joint training.

\noindent\textbf{Retrieval-augmented generation.}
Retrieval-augmented generation (RAG) improves knowledge-intensive generation by
conditioning the generator on information retrieved by the retrieval model from an external datastore
\citep{lewis2021rag}. A common line of work keeps the retriever model
fixed and adapts the LLM to retrieved evidence, through
domain-specific instruction tuning, search-augmented instruction learning, or
RL fine-tuning with retrieval curricula
\citep{zhang2024raft,luo2023sail,huang2025ragrl}. A second line keeps the LLM fixed and
optimizes retrieval or query-side behavior so that retrieved passages better
support the frozen generator \citep{shi2023replug,ma2023query,ram2023context}. Recent work explores RAG optimization through dual instruction tuning,
environment feedback, LLM-guided dynamic reranking, and joint optimization
of document selection and answer generation
\citep{lin2024radit,gao2025smartrag,sun2025dynamicrag,shi2025directrao}. These methods improve
retrieval-conditioned generation, but they do not address the
multi-turn RL setting where search queries and final answers are induced by an evolving policy.

\vspace{-0.15cm}
\section{Preliminaries: Multi-turn Tool-use RL and RAG}
\vspace{-0.1cm}

\subsection{Agentic RL with tool use}

In the multi-turn agentic RL setting, given a
prompt from the training distribution $x \sim \Dtrain$, the LLM generates the rollout trajectory as $\traj=(
    {a_1},\,
    {o_1},\,
    \ldots,\,
    a_{T-1},\,
    o_{T-1},\,
    a_T)$,
where $a_t$ denotes {the $t$-th segment of} LLM-generated action tokens {and} $o_t$ denotes the environment tokens returned by the external tool server. {Actions and retrieved information are wrapped in special tokens.} In search tasks, reasoning, search calls, and final answers are wrapped by $\texttt{<think>}$, $\texttt{<search>}$, and $\texttt{<answer>}$, respectively. When triggered by the search calls, external tools retrieve documents and wrap the relevant information in {an} $\texttt{<information>}$ block as the observation tokens $o_t$.

\noindent\textbf{Multi-turn RL with verifiable rewards.} In agentic RL, the rollout trajectory $\tau$ for the prompt $x$ is multi-turn and contains both {LLM-generated} action tokens $a_t$ and {retriever-generated} observation tokens $o_t$. Therefore, the distribution of the rollout trajectory
$\tau$ depends on both LLM parameters $\theta$ and the retrieval parameters $\eta$. Given the retrieval model $\eta$, in order to learn the LLM policy parameterized by $\theta$, we leverage reinforcement learning
with verifiable rewards to maximize a predefined scalar reward with a KL regularization term to prevent reward over-optimization
\begin{align}
    {\Jrl(\theta,\eta)}
    :=
    -\E_{x \sim \Dtrain,\, \traj \sim \policy(\cdot \mid x;\eta)}
    \left[
        R(x,\traj)
        -
        \beta \KL(\policy \,\|\, \refpol)
    \right],
    \label{eq:jrl}
\end{align}
where $\refpol$ denotes a reference LLM and
$\policy(\cdot \mid x;\eta)$ is the LLM policy under the retrieval model $\eta$. Following {\citet{jin2025searchr1}}, we use a binary exact-match outcome reward $R(x,\traj)$, which measures the correctness of the generated answer $\hat y(\traj)$ compared with the ground-truth answer $y^*(x)$. Specifically, $R(x,\traj)=1$ if $\hat y(\traj)$ is correct, and $R(x,\traj)=0$ otherwise.

\begin{figure}[!t]
\centering
\begin{tabular}{@{}c@{\hspace{0.04in}}c@{\hspace{0.04in}}c@{}}
\includegraphics[height=1.53in]{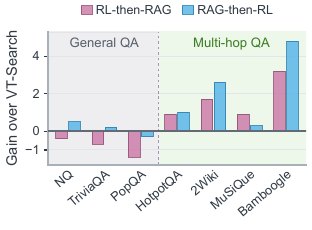} &
\includegraphics[height=1.53in]{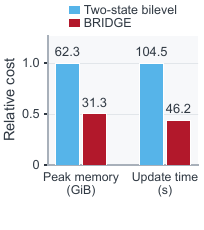} &
\includegraphics[height=1.53in]{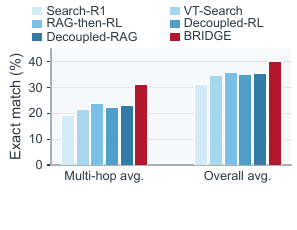} \\[0.01em]
{\small \textbf{(a)} Optimization order} &
{\small \textbf{(b)} Retriever-update cost} &
{\small \textbf{(c)} Component ablation}
\end{tabular}
\vspace{-0.3cm}
\caption{\textbf{Optimization order, update cost, and component-wise ablations.}
(a) Per-dataset EM gains over VT-Search, in percentage points, for the two {orders of sequential retriever and LLM tuning}.
(b) Mean peak allocated memory and time per retriever $\eta$ update, compared with the two-state bilevel estimator in \eqref{penalty-theorem-2}; bar labels show raw values and the y-axis is normalized.
(c) Component ablations with \textit{Qwen2.5-3B-base} model.
Full results appear in Table~\ref{tab:retriever_objective_ablation} in Appendix~\ref{app:component_ablation}.}
\label{fig:optimization_ablation}
\vspace{-0.2cm}
\end{figure}

\noindent\textbf{Group Relative Policy Optimization (GRPO).} To optimize \eqref{eq:jrl}, {we follow \citet{jin2025searchr1} and use GRPO \citep{shao2024deepseekmath}, which assigns per-rollout credit by comparing multiple rollouts of the same prompt}. Specifically, given a prompt $x$, GRPO samples a group of $G$ trajectories and normalizes
their rewards within the group. Let $\traj_i$ be the $i$-th rollout trajectory with reward
$R_i=R(x,\traj_i)$, and let the advantage {of the} $i$-th rollout be $A_i=(R_i - \operatorname{mean}(\{R_j\}))/\operatorname{std}(\{R_j\})$.
Let $a_{i,t,m}$ {denote} the $m$-th token in the $t$-th action in the $i$-th rollout, and $a_{i,t}=(a_{i,t,1},\ldots,a_{i,t,L_{i,t}})$ be the tokenized action
segment $t$ in rollout $i$. For each action token, define the importance ratio
$r_{i,t,m}(\theta)
    =
    \frac{
        \policy(a_{i,t,m}\mid \traj_{i,<t},a_{i,t,<m})
    }{
        \oldpol(a_{i,t,m}\mid \traj_{i,<t},a_{i,t,<m})
    }$,
where $\traj_{i,<t}=({a_{i,1},o_{i,1}},\cdots,a_{i,t-1},o_{i,t-1})$ denotes the trajectory prefix of {the} $i$-th rollout before {turn} $t$, and $a_{i,t,<m}$ denotes the prefix of the $t$-th action tokens in the $i$-th rollout before $m$. Let us define the clipped token objective as
\begin{equation*}
    \ell^{\mathrm{clip}}_{i,t,m}(\theta)
    =
    \min\!\left[
        r_{i,t,m}(\theta)A_i,\,
        \operatorname{clip}(r_{i,t,m}(\theta),1-\epsilon,1+\epsilon)
        A_i
    \right].
\end{equation*}
{The} multi-turn GRPO surrogate averages this term over action tokens:
\begin{equation}
    \grpoobj(\theta)
    =-
    \frac{1}{G}
    \sum_{i=1}^{G}
    \frac{1}{T_i}
    \sum_{t=1}^{T_i}\frac{1}{L_{i,t}}\sum_{m=1}^{L_{i,t}}
    \left(\ell^{\mathrm{clip}}_{i,t,m}(\theta)
    -
    \beta \KL(\policy \,\|\, \refpol)\right)
    \label{eq:grpo}
\end{equation}
where $T_i$ is the number of turns in rollout $i$, and the trajectory-level advantage is broadcast to each turn and normalized by its length.

\subsection{Tool-use RAG and RAG likelihood}

In the agentic RL setting, an external search tool retrieves relevant information from a retrieval corpus and {inserts the retrieved documents into the rollout within an $\texttt{<information>}$ block, giving the agent access to domain-specific knowledge}.

Let $\corpus=\{d_j\}_{j=1}^{N}$ be a retrieval corpus of chunked passages. At a retrieval turn $t$, let $h_t=(\tau_{<t},a_{t})$ denote the
trajectory prefix up to turn $t$. We use a dense retriever with \textbf{relevance score} $s_\eta(h_t,d)
    =
    f_\eta(h_t)^\top g(d)$,
where $f_\eta$ is the trainable query encoder with parameter $\eta$, and
$g$ is a fixed document encoder. Given the prefix $h_t$, the
retriever returns the {set of $m$ most relevant documents}
\begin{align}
    D_t
    =
    \mathcal N_m(h_t)
    =
    \operatorname{TopM}_{d \in \corpus}\, s_\eta(h_t,d),
\end{align}
which is rendered into the information block as {the environment} observation $o_t$ and provided to the LLM as additional context for {subsequent} generation. Through $o_t$, {the retriever parameter} $\eta$ changes the {subsequent} policy context and therefore the
trajectory distribution in \eqref{eq:jrl}.

For training $\eta$, $\Cand(h_t)$ denotes the top-$k$ candidate set under the same relevance score, used by both the RAG likelihood and the RL outcome-credit surrogate. We choose $k>m$ {so that the retriever receives document-level training signal from more candidates than the $m$ documents}.

\noindent\textbf{RAG likelihood.}
For each training prompt $x$, sampled rollouts provide retrieval states
$h$, paired with the reference answer $y=y^*(x)$. We compute the RAG negative log-likelihood at each retrieval turn of the sampled rollouts, {averaged over retrieval contexts $h$ and ground-truth answers $y$}:
\begin{equation}
    \Lrag(\theta,\eta)
    =
    -
    \E_{(h,y)}
    \log
    \sum_{d \in \Cand(h)}
    \retr(d \mid h)
    p_{\theta}(y \mid h,d)
    \label{eq:lrag}
\end{equation}
where the retriever
distribution is {a softmax over the top-$k$ set $\Cand(h_t)$}:
\begin{equation}
    \retr(d \mid h_t)
    =
    \frac{\exp(s_\eta(h_t,d))}
    {\sum_{d' \in \Cand(h_t)} \exp(s_\eta(h_t,d'))}.
    \label{eq:retriever_distribution}
\end{equation}
The RAG objective {is the negative log-likelihood of the correct answer under the current LLM $\theta$, marginalized over the top-$k$ documents under the current retriever $\eta$}. By minimizing this objective with respect to $\eta$, the retriever is encouraged to rank more informative and useful documents higher, thereby improving the quality of the final answer generation.

\section{Bilevel retrieval-augmented agentic RL}
\label{sec:bridge_method}

{We first examine the order of retriever and policy optimization, then formulate the bilevel problem and develop an efficient algorithm for it.}

\subsection{Bilevel formulation}

\noindent\textbf{{Hierarchical} order in retrieval and LLM tuning.} Although both the LLM and the retriever affect final answer accuracy, we first test whether retrieval tuning and policy optimization are order-sensitive. To this end, we compare two sequential tuning pipelines on top of the {RL-trained} VT-Search baseline: \textbf{RAG-then-RL}, which first improves the retriever and then optimizes the LLM via RL, and \textbf{RL-then-RAG}, which applies these two steps in the reverse order. As shown in Figure~\ref{fig:optimization_ablation}(a), adding RAG before RL improves performance over the RL-only VT-Search baseline on most datasets, with especially clear gains on multi-hop QA benchmarks. However, tuning the retriever after RL leads to worse performance: on general QA tasks, it can underperform the RL-only model, and overall it yields smaller gains than RAG-then-RL. This {indicates that the order of retrieval and policy optimization matters: the retriever should adapt first} to expose task-relevant evidence, and the LLM can then learn search and reasoning behavior under
useful information blocks.

\noindent\textbf{Motivation {for} dynamic tuning.} {Despite this sequential order, retrieval and LLM optimization can still reinforce each other, with feedback from one improving the other. We capture this with the bilevel formulation}
\begin{equation}
    \min_{\theta,\eta}
    ~~
    \Jrl(\theta,\eta)
    \qquad
    \mathrm{s.t.}
    \quad
    \eta \in \eta^*(\theta):=
    \arg\min_{\eta'} \Lrag(\theta,\eta')
    \label{eq:ideal_bilevel}
\end{equation}
where the lower-level problem optimizes the RAG objective to find the best {retriever} $\eta$ {for} the current policy $\theta$, and the upper-level problem optimizes the RL objective to find the LLM policy $\theta$ that maximizes the outcome reward {when paired with its optimal retriever. Problem \eqref{eq:ideal_bilevel} thus gives retrieval optimization priority at the lower level and optimizes the LLM against the resulting retriever, so the retriever keeps adapting as the policy changes}.

\begin{figure}[t]
\centering
\begin{minipage}[t]{0.24\textwidth}
    \centering
    \includegraphics[width=\linewidth]{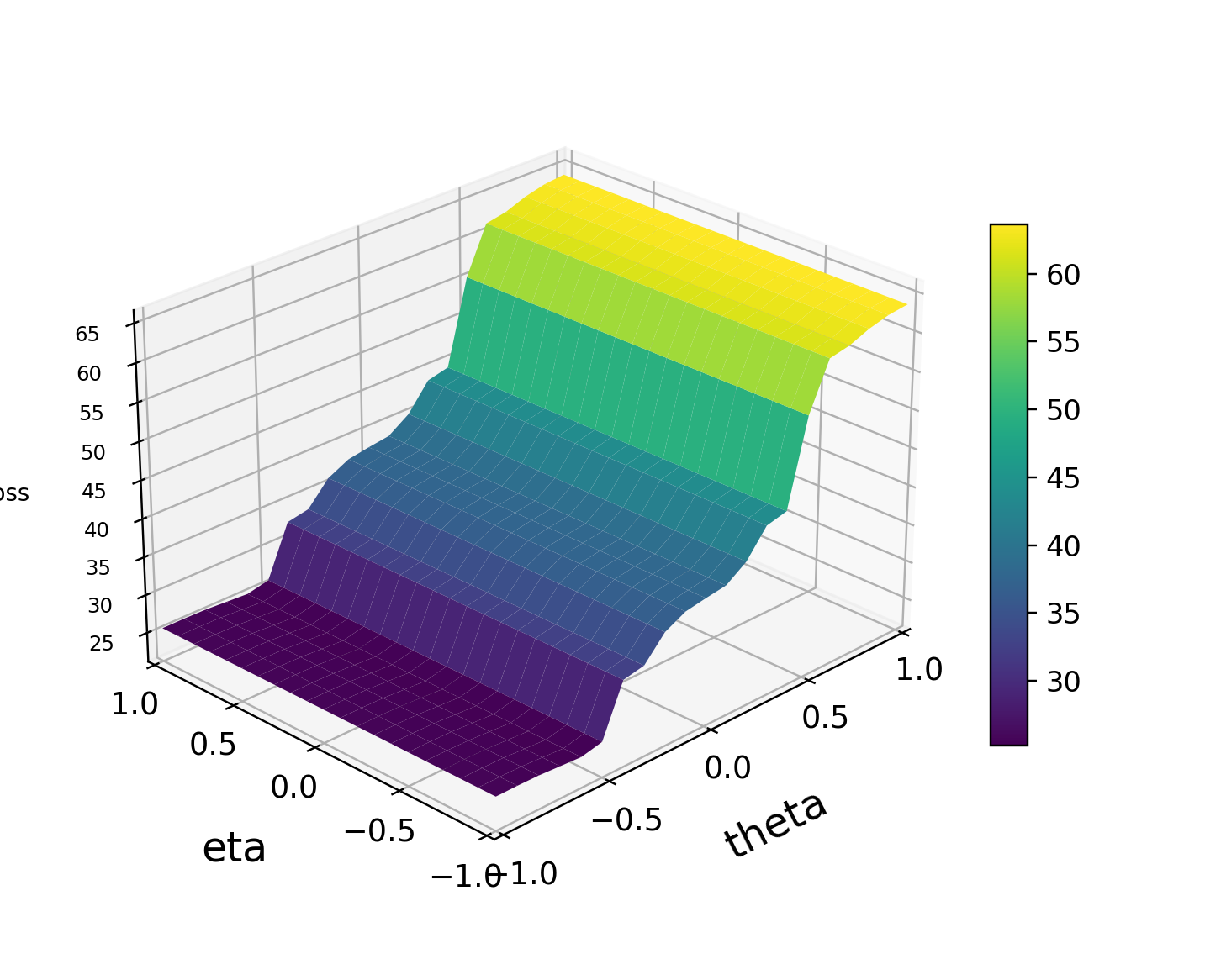}
    \vspace{0.05em}
    {\scriptsize \textbf{(a)} Joint RAG}
\end{minipage}
\hfill
\begin{minipage}[t]{0.24\textwidth}
    \centering
    \includegraphics[width=\linewidth]{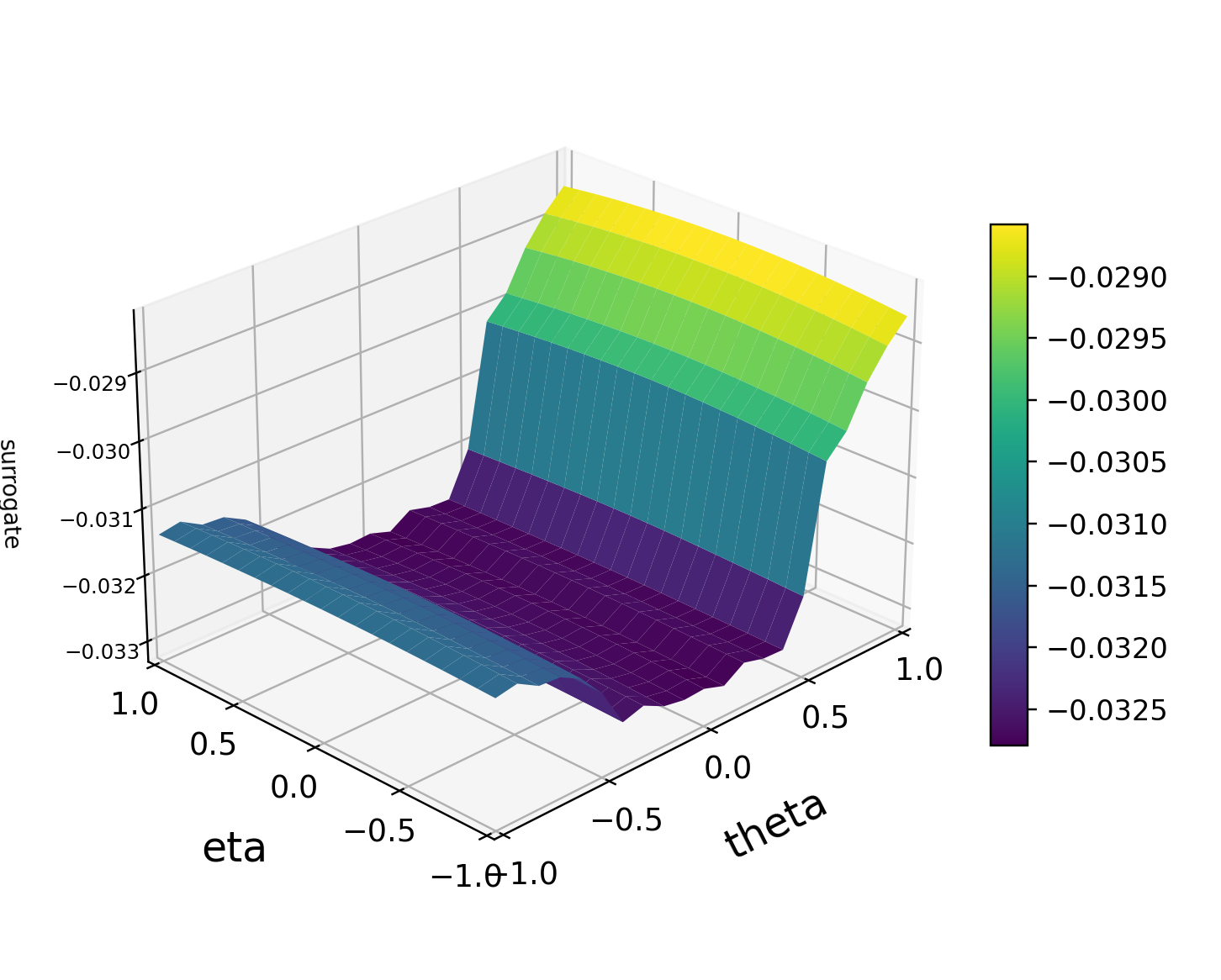}
    \vspace{0.05em}
    {\scriptsize \textbf{(b)} Joint RL}
\end{minipage}
\hfill
\begin{minipage}[t]{0.24\textwidth}
    \centering
    \includegraphics[width=\linewidth]{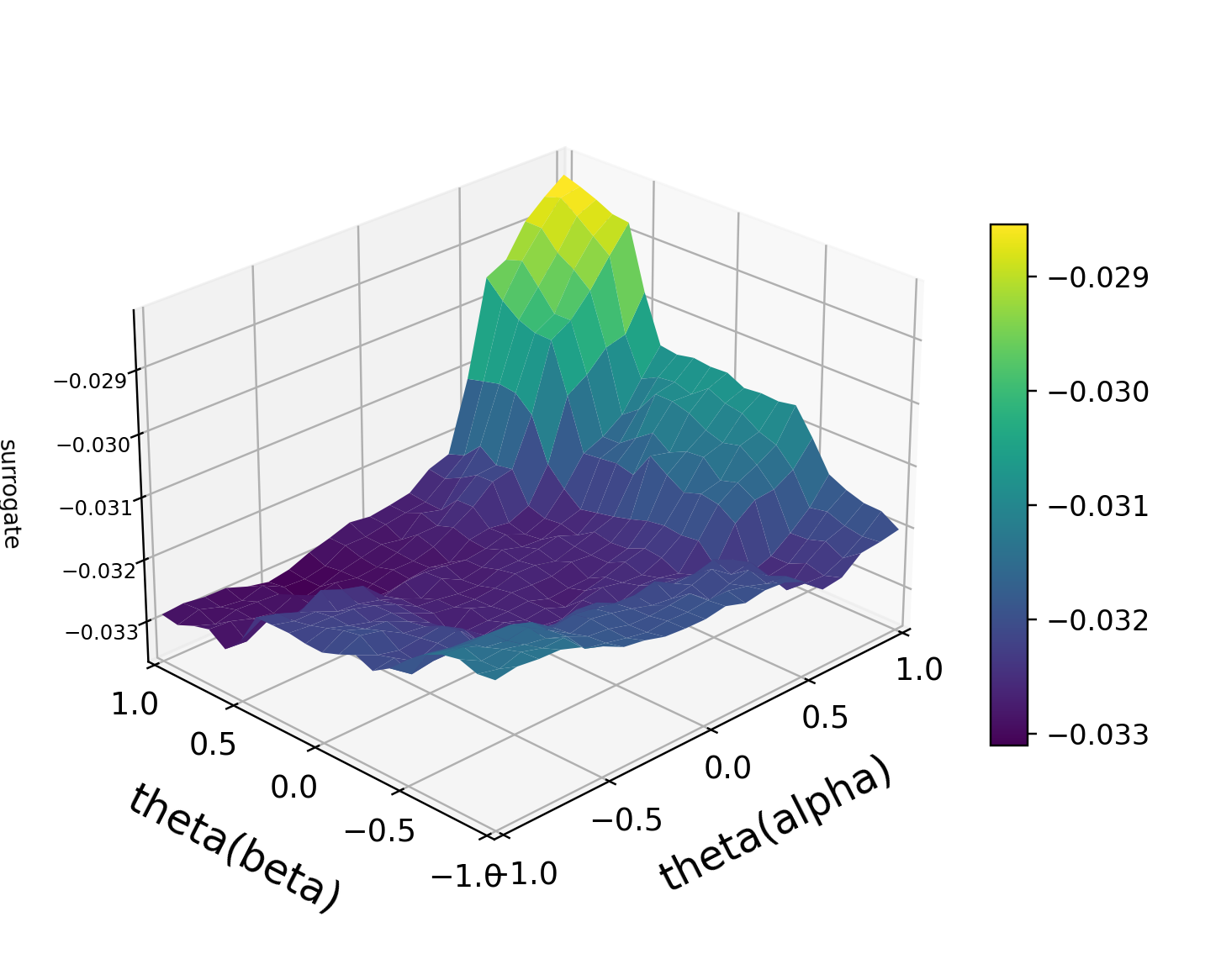}
    \vspace{0.05em}
    {\scriptsize \textbf{(c)} $\theta$-only, initial}
\end{minipage}
\hfill
\begin{minipage}[t]{0.24\textwidth}
    \centering
    \includegraphics[width=\linewidth]{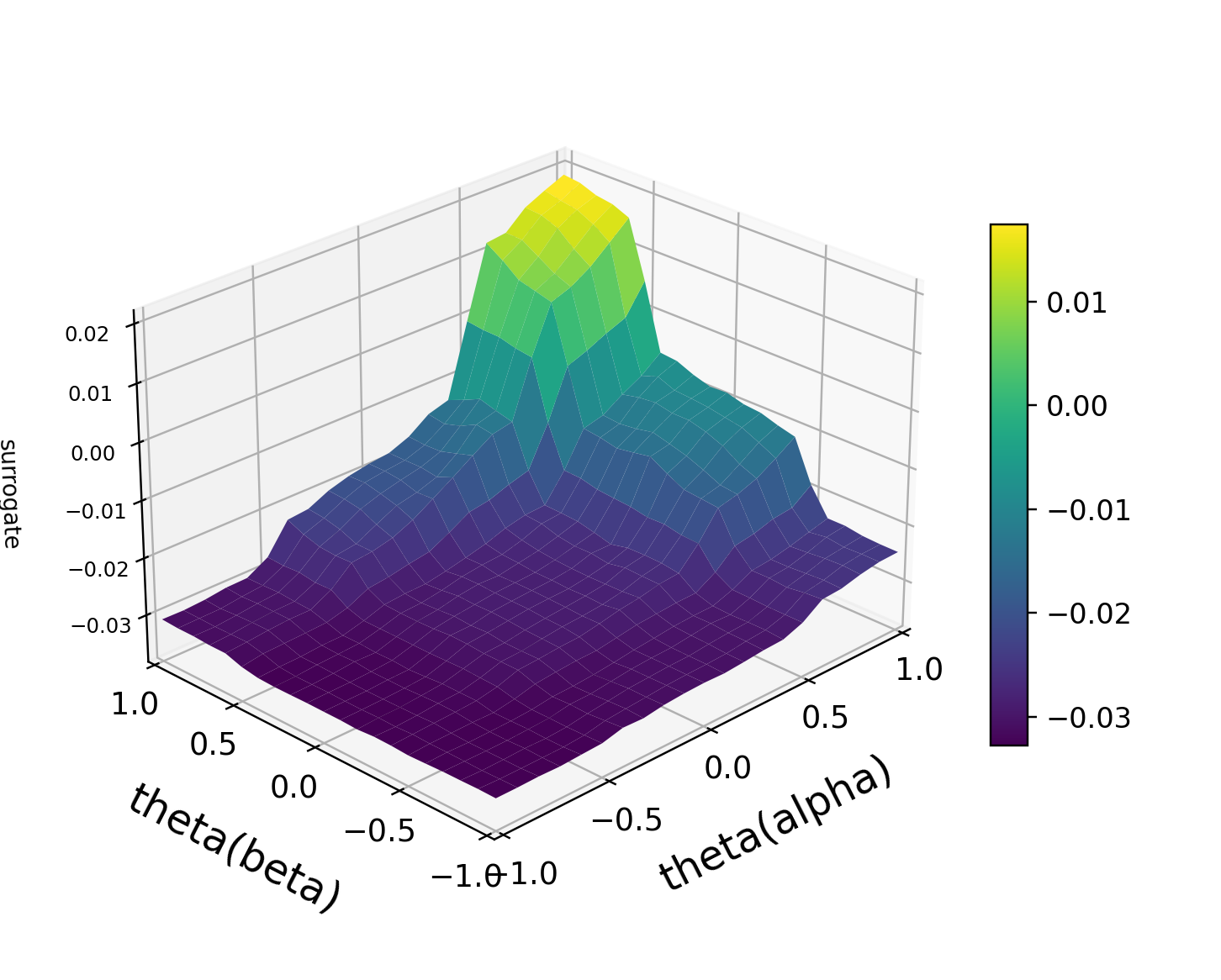}
    \vspace{0.05em}
    {\scriptsize \textbf{(d)} $\theta$-only, final}
\end{minipage}
\vspace{-0.08 in}
\caption{\textbf{Loss landscapes of the RAG and RL objectives} for the search task with \textit{Qwen2.5-3B-base}.
We visualize the joint RAG landscape with respect to both $\theta$ and $\eta$ in (a) and the joint RL landscape in (b). We also visualize
the RL landscape {along two random directions in $\theta$} (\textit{alpha} and \textit{beta}) in (c)--(d) at the initial and final {checkpoints}. As training proceeds, \method{} moves toward a smoother
region of the RL objective landscape. More visualizations are in Appendix~\ref{app:landscapes}.}
\label{fig:rl_loss_landscapes}
\vspace{-0.4cm}
\end{figure}

\subsection{Memory efficient bilevel algorithm design via landscape analysis}
{We now develop \method{}, a memory-efficient first-order method for solving \eqref{eq:ideal_bilevel}.}

\noindent\textbf{First-order {penalty-based} approach. } {The main difficulty is the implicit mapping $\eta^*(\theta)$. Even in the simplest case, where the RAG objective has a unique minimizer with an invertible Hessian, the implicit function theorem} \citep{dontchev2014implicit} {gives}
\begin{equation}\label{implicit-theorem}
\nabla \Jrl(\theta,\eta^*(\theta))=\nabla_\theta \Jrl(\theta,\eta)-\nabla_{\theta\eta}\Lrag(\theta,\eta)\nabla_{\eta\eta}^{-1}\Lrag(\theta,\eta) \nabla_\eta \Jrl(\theta,\eta)|_{\eta=\eta^*(\theta)}
\end{equation}
{which involves an inverse Hessian and a mixed Jacobian and is therefore infeasible to compute when both $\theta$ and $\eta$ parameterize large neural networks}.

{Building on first-order penalty-based bilevel methods} \citep{shen2023penalty,kwon2023penalty,scellier2017equilibrium,zucchet2022beyond}, we {estimate this gradient through the penalty function}
\begin{equation}\label{penalty-decompose}
\mathcal{L}_\gamma (\theta,\eta)=\Jrl(\theta,\eta)+\gamma \Lrag(\theta,\eta).
\end{equation}
{Let} $\eta_\gamma^*(\theta)=\argmin_{\eta}\mathcal{L}_\gamma(\theta,\eta)$ {denote the penalized retriever. Since $\eta_\gamma^*(\theta)\rightarrow\eta^*(\theta)$ as $\gamma\rightarrow\infty$, the implicit gradient can be approximated by the first-order penalty gradient}
\begin{align}\label{penalty-theorem}
\nabla \Jrl(\theta,\eta^*(\theta))
&\stackrel{(a)}{\approx}\nabla_\theta \Jrl(\theta,\eta)+\gamma\nabla_{\theta\eta}\Lrag(\theta,\eta)(\eta_\gamma^*(\theta)-\eta^*(\theta))~|_{\eta=\eta^*(\theta)} \nonumber\\
&\stackrel{(b)}{\approx}\nabla_\theta \Jrl(\theta,\eta)+\gamma\nabla_{\theta}\Lrag(\theta,\eta)~|_{\eta=\eta_\gamma^*(\theta)} - \gamma\nabla_{\theta}\Lrag(\theta,\eta)|_{\eta=\eta^*(\theta)}\\
&=\nabla_\theta \mathcal{L}_\gamma(\theta,\eta)|_{\eta=\eta_\gamma^*(\theta)}- \gamma\nabla_{\theta}\Lrag(\theta,\eta)|_{\eta=\eta^*(\theta)}\label{penalty-theorem-2}
\end{align}
where {(a) holds because}
\begin{equation*}
\nabla_{\eta\eta}\Lrag(\theta,\eta)(\eta_\gamma^*(\theta)-\eta^*(\theta))|_{\eta=\eta^*(\theta)} \approx \nabla_\eta \Lrag(\theta,\eta)|_{\eta=\eta_\gamma^*(\theta)}\approx -\frac{1}{\gamma} \nabla_\eta\Jrl(\theta,\eta)|_{\eta=\eta_\gamma^*(\theta)}
\end{equation*}
{in which the first approximation follows from a Taylor expansion and $\nabla_\eta \Lrag(\theta,\eta)|_{\eta=\eta^*(\theta)}=0$, and the second from the optimality condition $\nabla_\eta \mathcal{L}_\gamma(\theta,\eta)|_{\eta=\eta_\gamma^*(\theta)}=0$; (b) follows from a first-order Taylor expansion of $\nabla_\theta\Lrag$ in $\eta$. The error of the fully first-order estimator \eqref{penalty-theorem-2} decreases as ${\cal O}(1/\gamma)$} \citep{kwon2023fully}.

However, the {two-state} estimator in \eqref{penalty-theorem-2} is still {memory-intensive} for RAG tool-use RL, because it needs to track two separate {retriever} states $\eta_\gamma^*(\theta)$ and $\eta^*(\theta)$, which are optimized using different objectives $\mathcal{L}_\gamma$ and $\Lrag$. This increases the memory cost by storing an additional set of retriever parameters, their optimizer states, and the corresponding retrieval logs; see Figure \ref{fig:optimization_ablation} (b).

\noindent\textbf{Landscape analysis.} To alleviate the memory bottleneck, existing single-state bilevel methods
\citep{xiao2025ldc,shen2024seal} avoid optimizing
$\eta^*(\theta)$ by treating
$\nabla_\theta\Lrag(\theta,\eta^*(\theta))\approx 0$ along the training
trajectory. We visualize the RAG and RL loss landscapes following \citet{li2018visualizing} in Figure~\ref{fig:rl_loss_landscapes} as a motivation.
As shown in Figure~\ref{fig:rl_loss_landscapes}(a), when the LLM policy
$\theta$ is suboptimal, the RAG objective still changes substantially along the
$\theta$ direction for most {values} of $\eta$, which indicates that $\nabla_\theta \Lrag(\theta,\eta)|_{\eta=\eta^*(\theta)}$ cannot be ignored. However, Figure~\ref{fig:rl_loss_landscapes}(b) shows that the upper-level
RL objective is relatively flat along the retriever direction $\eta$ across a
wide range of $\theta$. This suggests that, for a fixed policy $\theta$, the
penalized retriever solution $\eta_\gamma^*(\theta)$ remains close to the
unpenalized lower-level solution $\eta^*(\theta)$, which motivates {omitting} the last two
terms {of} \eqref{penalty-theorem} together{: their difference is negligible} if $\|\eta_\gamma^*(\theta)-\eta^*(\theta)\|=\mathrm{o}(1/\gamma)$ \citep{jiang2026beyond}. {The three resulting estimators, compared in Section~\ref{sec:open_domain} numerically, are (with $\eta^*=\eta^*(\theta)$, $\eta_\gamma^*=\eta_\gamma^*(\theta)$)}
\begin{equation*}
{\underbrace{\nabla_\theta\mathcal{L}_\gamma(\theta,\eta_\gamma^*)-\gamma\nabla_\theta\Lrag(\theta,\eta^*)}_{\text{two-state \eqref{penalty-theorem-2}}},\qquad
\underbrace{\nabla_\theta\mathcal{L}_\gamma(\theta,\eta_\gamma^*)}_{\text{Drop}},\qquad
\underbrace{\nabla_\theta\Jrl(\theta,\eta_\gamma^*)}_{\text{\method{}}},}
\end{equation*}
{where Drop removes only the value-function term \citep{xiao2025ldc,shen2024seal} and \method{} removes both RAG terms, needing one retriever state.}

\begin{figure}[!t]
\begin{minipage}{\linewidth}
\setlength{\intextsep}{0pt}
\begin{algorithm}[H]
\caption{{\method{}}}
\label{alg:bridge}
\begin{algorithmic}[1]
\STATE \textbf{Input:} Training prompts $\Dtrain$, document corpus $\mathcal{C}$,
policy $\policy$, initial retrieval model $\eta_0$, initial LoRA retrieval adapter $A_0=0$ and random $B_0$, RL rollout group size
$G$, penalty parameter $\gamma$, retriever update period $P$, number of rounds $R$.
\FOR{$r=0,1,\ldots,R-1$}
    \STATE Sample a mini-batch of questions $x \sim \Dtrain$.
    \STATE Generate $G$ multi-turn tool-use rollouts with $\pi_{\theta_r}$ and
    retriever $\eta_r=\eta_0+A_r B_r$.
    \STATE Compute answer rewards and GRPO advantages over
    LLM-generated action tokens in \eqref{eq:grpo}.
    \STATE Compute retrieval-turn credit and the RAG likelihood in \eqref{eq:lrag}
    \STATE {If $r \bmod P=0$, update the LoRA adapter of $\eta$ by}
    \eqref{eq:eta_loss}
    \STATE Update $\theta$ {by}
    \eqref{eq:theta_loss}
\ENDFOR
\STATE \textbf{Output:} Final policy $\pi_{\theta_{R}}$ and query-side retriever
adapter $\eta_R$.
\end{algorithmic}
\end{algorithm}

\vspace{0pt}
\begingroup
\setlength{\intextsep}{0pt}
\renewcommand{\arraystretch}{0.90}
\newcommand{\bridgerowstrut}{\rule[-3.5pt]{0pt}{14pt}}
\begin{table}[H]
\centering
\caption{\textbf{Open-domain QA evaluation results.} Best/second-best scores are highlighted in \textbf{bold}/\underline{underlined}.
${}^{\dagger}$/${}^{*}$ denote in-domain/out-of-domain datasets relative to our training set.
${}^{\ddagger}$ indicates results from the original papers.
{Search-R1 and \method{} rows use the higher-Avg.\ \textit{Base}/\textit{Instruct} variant (Appendix~\ref{app:backbones}).}
IRCoT uses the corresponding \textit{Instruct} models.
M-Avg. and Avg. denote multi-hop and overall averages.
Darker green shading indicates higher accuracy.}
\label{tab:knowledge_qa_results}
\vspace{0pt}
\resizebox{\textwidth}{!}{
\begin{tabular}{lccccccccc}
\toprule
\multirow{2}{*}{\textbf{Method}} &
\multicolumn{3}{c}{\textbf{General QA}} &
\multicolumn{4}{c}{\textbf{Multi-hop QA} (knowledge-intensive)} &
\multirow{2}{*}{\textbf{M-Avg.}} &
\multirow{2}{*}{\textbf{Avg.}} \\
\cmidrule(lr){2-4}
\cmidrule(lr){5-8}
&
\textbf{NQ\textsuperscript{\ensuremath{\dagger}}} &
\textbf{TriviaQA\textsuperscript{*}} &
\textbf{PopQA\textsuperscript{*}} &
\textbf{HotpotQA\textsuperscript{\ensuremath{\dagger}}} &
\textbf{2Wiki\textsuperscript{*}} &
\textbf{MuSiQue\textsuperscript{*}} &
\textbf{Bamboogle\textsuperscript{*}} &
&
\\
\midrule
\multicolumn{10}{c}{\textbf{Large-model references: direct inference without retrieval}} \\
\midrule
\textit{GPT-4o} & 22.3 & 66.5 & 36.9 & 22.8 & 19.8 & 6.2 & 26.4 & 18.8 & 28.7 \\
\textit{GPT-5.6 Luna} & 19.0 & 65.2 & 33.4 & 22.2 & 19.9 & 7.1 & 26.4 & 18.9 & 27.6 \\
\textit{Qwen3-235B-A22B} & 23.0 & 58.4 & 28.9 & 23.2 & 29.1 & 4.9 & 20.8 & 19.5 & 26.9 \\
\midrule
\multicolumn{10}{c}{\textbf{\textit{Qwen2.5-3B}}} \\
\midrule
Direct Inference & 18.2 & 30.4 & 20.2 & 10.9 & 12.0 & 3.5 & 8.8 & \scorecell{8.8}{34.5}{8.8}{8.8} & \scorecell{14.9}{41.5}{14.9}{14.9} \\
IRCoT (Instruct)\textsuperscript{\ensuremath{\ddagger}} & 11.1 & 31.2 & 20.0 & 16.4 & 17.1 & 6.7 & 24.0 & \scorecell{8.8}{34.5}{16.1}{16.1} & \scorecell{14.9}{41.5}{18.1}{18.1} \\
Search-o1\textsuperscript{\ensuremath{\ddagger}} & 23.8 & 47.2 & 26.2 & 22.1 & 21.8 & 5.4 & \underline{32.0} & \scorecell{8.8}{34.5}{20.3}{20.3} & \scorecell{14.9}{41.5}{25.5}{25.5} \\
RAG\textsuperscript{\ensuremath{\ddagger}} & 34.8 & 54.4 & 38.7 & 25.5 & 22.6 & 4.7 & 8.0 & \scorecell{8.8}{34.5}{15.2}{15.2} & \scorecell{14.9}{41.5}{27.0}{27.0} \\
SFT\textsuperscript{\ensuremath{\ddagger}} & 24.9 & 29.2 & 10.4 & 18.6 & 24.8 & 4.4 & 11.2 & \scorecell{8.8}{34.5}{14.8}{14.8} & \scorecell{14.9}{41.5}{17.6}{17.6} \\
R1-base\textsuperscript{\ensuremath{\ddagger}} & 22.6 & 45.5 & 17.3 & 20.1 & 26.8 & 5.5 & 22.4 & \scorecell{8.8}{34.5}{18.7}{18.7} & \scorecell{14.9}{41.5}{22.9}{22.9} \\
Rejection Sampling\textsuperscript{\ensuremath{\ddagger}} & 29.4 & 48.8 & 33.2 & 24.0 & 23.3 & 5.9 & 21.0 & \scorecell{8.8}{34.5}{18.6}{18.6} & \scorecell{14.9}{41.5}{26.5}{26.5} \\
s3 & 32.2 & 54.4 & 38.9 & 24.9 & 25.2 & 5.7 & 8.7 & \scorecell{8.8}{34.5}{16.1}{16.1} & \scorecell{14.9}{41.5}{27.1}{27.1} \\
Search-R1\textsuperscript{\ensuremath{\ddagger}} & 39.7 & 56.5 & 39.1 & \underline{33.1} & \underline{31.0} & \underline{12.4} & 23.2 & \scorecell{8.8}{34.5}{24.9}{\underline{24.9}} & \scorecell{14.9}{41.5}{33.6}{33.6} \\
Agentic-R & 43.1 & \underline{60.0} & 42.3 & 31.0 & 27.2 & 8.4 & 11.2 & \scorecell{8.8}{34.5}{19.5}{19.5} & \scorecell{14.9}{41.5}{31.9}{31.9} \\
VT-Search\textsuperscript{\ensuremath{\ddagger}} & \underline{45.4} & \textbf{61.6} & \textbf{48.1} & 32.4 & 30.8 & 7.6 & 15.2 & \scorecell{8.8}{34.5}{21.5}{21.5} & \scorecell{14.9}{41.5}{34.4}{\underline{34.4}} \\
\noalign{\vskip 2pt}
\cdashline{1-10}
\noalign{\vskip 4pt}
\rowcolor{bilevelblue}
\bridgerowstrut \textbf{\method{}} (ours) & \textbf{46.3} & \textbf{61.6} & \underline{44.7} & \textbf{42.3} & \textbf{41.4} & \textbf{16.8} & \textbf{37.6} & \scorecell{8.8}{34.5}{34.5}{\textbf{34.5}} & \scorecell{14.9}{41.5}{41.5}{\textbf{41.5}} \\
\midrule
\multicolumn{10}{c}{\textbf{\textit{Qwen2.5-7B}}} \\
\midrule
Direct Inference\textsuperscript{\ensuremath{\ddagger}} & 13.4 & 40.8 & 14.0 & 18.3 & 25.0 & 3.1 & 12.0 & \scorecell{14.5}{43.2}{14.6}{14.6} & \scorecell{18.1}{48.2}{18.1}{18.1} \\
IRCoT (Instruct)\textsuperscript{\ensuremath{\ddagger}} & 22.4 & 47.8 & 30.1 & 13.3 & 14.9 & 7.2 & 22.4 & \scorecell{14.5}{43.2}{14.5}{14.5} & \scorecell{18.1}{48.2}{22.6}{22.6} \\
Search-o1\textsuperscript{\ensuremath{\ddagger}} & 15.1 & 44.3 & 13.1 & 18.7 & 17.6 & 5.8 & 29.6 & \scorecell{14.5}{43.2}{17.9}{17.9} & \scorecell{18.1}{48.2}{20.6}{20.6} \\
RAG\textsuperscript{\ensuremath{\ddagger}} & 34.9 & 58.5 & 39.2 & 29.9 & 23.5 & 5.8 & 20.8 & \scorecell{14.5}{43.2}{20.0}{20.0} & \scorecell{18.1}{48.2}{30.4}{30.4} \\
SFT\textsuperscript{\ensuremath{\ddagger}} & 31.8 & 35.4 & 12.1 & 21.7 & 25.9 & 6.6 & 11.2 & \scorecell{14.5}{43.2}{16.4}{16.4} & \scorecell{18.1}{48.2}{20.7}{20.7} \\
R1-base\textsuperscript{\ensuremath{\ddagger}} & 29.7 & 53.9 & 20.2 & 24.2 & 27.3 & 8.3 & 29.6 & \scorecell{14.5}{43.2}{22.4}{22.4} & \scorecell{18.1}{48.2}{27.6}{27.6} \\
Rejection Sampling\textsuperscript{\ensuremath{\ddagger}} & 36.0 & 59.2 & 38.0 & 33.1 & 29.6 & 12.3 & 35.5 & \scorecell{14.5}{43.2}{27.6}{27.6} & \scorecell{18.1}{48.2}{34.8}{34.8} \\
s3 & 36.0 & 57.4 & 43.0 & 28.9 & 25.1 & 5.6 & 12.0 & \scorecell{14.5}{43.2}{17.9}{17.9} & \scorecell{18.1}{48.2}{29.7}{29.7} \\
Search-R1\textsuperscript{\ensuremath{\ddagger}} & 48.0 & 63.8 & 45.7 & 43.3 & 38.2 & 19.6 & 43.2 & \scorecell{14.5}{43.2}{36.1}{36.1} & \scorecell{18.1}{48.2}{43.1}{43.1} \\
Agentic-R\textsuperscript{\ensuremath{\ddagger}} & 42.4 & \textbf{69.0} & 44.1 & \underline{45.8} & \underline{45.3} & \underline{20.3} & \underline{48.0} & \scorecell{14.5}{43.2}{39.8}{\underline{39.8}} & \scorecell{18.1}{48.2}{45.0}{45.0} \\
VT-Search\textsuperscript{\ensuremath{\ddagger}} & \underline{49.3} & 66.2 & \textbf{50.2} & 44.8 & \underline{45.3} & 19.3 & 46.4 & \scorecell{14.5}{43.2}{39.0}{39.0} & \scorecell{18.1}{48.2}{45.9}{\underline{45.9}} \\
\noalign{\vskip 2pt}
\cdashline{1-10}
\noalign{\vskip 4pt}
\rowcolor{bilevelblue}
\bridgerowstrut \textbf{\method{}} (ours) & \textbf{50.1} & \underline{66.3} & \underline{48.4} & \textbf{48.9} & \textbf{49.9} & \textbf{22.7} & \textbf{51.2} & \scorecell{14.5}{43.2}{43.2}{\textbf{43.2}} & \scorecell{18.1}{48.2}{48.2}{\textbf{48.2}} \\
\bottomrule
\end{tabular}
}
\vspace{-0.3cm}
\end{table}
\endgroup
\end{minipage}
\end{figure}

\noindent\textbf{Memory efficient algorithm design. } {Based on the estimator fidelity evidences in Section~\ref{sec:open_domain}, we therefore track only the penalized retriever $\eta_\gamma^*(\theta)$ and update $\theta$ with the \method{} estimator. To further reduce memory, we parameterize the retriever with a query-side LoRA adapter} \citep{hulora} and update it periodically; see Figure~\ref{fig:optimization_ablation}(b) for a memory comparison. {Given the fixed initial retriever $\eta_0$, the retriever used for search is $\eta_0+AB$, where $A$ and $B$ are trainable low-rank matrices.}

At each round $r$, \method{} uses the current penalized retriever $\eta_r=\eta_0+A_{r} B_{r}$ and the current {LLM} $\theta_r$ to
sample $G$ tool-use rollouts {per question}. The query-side
LoRA retriever is then updated by
\begin{align}
    &A_{r+1}
    = A_r-\alpha_A\left(
    \nabla_A\Jrl(\theta_r,\eta_r)
    +
    \gamma
    \nabla_A\Lrag(\theta_r,\eta_r)\right),\nonumber\\
    &B_{r+1}
    = B_r-\alpha_B
    \left(\nabla_B\Jrl(\theta_r,\eta_r)
    +
    \gamma
    \nabla_B\Lrag(\theta_r,\eta_r)\right),
    \label{eq:eta_loss}
\end{align}
where the {retriever-side RL gradient is approximated by a retrieval-side GRPO surrogate and the RAG gradient follows \citet{lewis2021rag}; both are derived in Appendix~\ref{app:retrieval_gradient}. After the retriever update, the LLM is updated with the GRPO surrogate on the same rollouts}:
\begin{equation}
    \theta_{r+1}
    =\theta_r
    -\alpha
    \nabla_\theta\Jrl(\theta_r,\eta_{r}).
    \label{eq:theta_loss}
\end{equation}
{Algorithm~\ref{alg:bridge} summarizes \method{}, and Figure~\ref{fig:bridge_overview} gives an overview. Figure~\ref{fig:optimization_ablation}(b) further shows its memory benefits compared with two-state estimatior. }

\noindent\textbf{Interpretation. } Intuitively, the RL signal provides global feedback on the success of the complete rollout, while the RAG signal distinguishes candidate documents by how strongly they support generating the correct answer under the current LLM. Combining these signals guides retriever adaptation through both final task outcomes and document-level answer likelihood. The LLM policy is then updated, guided solely by the upper-level RL objective. These two stages reflect the retrieval--policy hierarchy in our
bilevel formulation. {In Appendix~\ref{app:bridge_convergence}, we establish convergence of an idealized version of these updates, with exact gradients, to an approximate stationary point of} \eqref{eq:ideal_bilevel}.

\vspace{-0.1cm}
\section{Experiments}
\vspace{-0.1cm}
\subsection{Experimental setup}
\label{sec:experimental_setup}
\noindent\textbf{Models and benchmarks.}
We use \textit{Qwen2.5-3B-base} model  \citep{qwen2024technical} and an \textit{E5}
retriever \citep{wang2022text} unless stated otherwise. We evaluate {on} twelve
benchmarks spanning three {task families}: 1) \textit{general open-domain QA}, including
NQ~\citep{kwiatkowski2019natural}, TriviaQA~\citep{joshi2017triviaqa},
PopQA~\citep{mallen2023trust}; 2) \textit{open-domain multi-hop QA}, including
HotpotQA~\citep{yang2018hotpotqa}, 2Wiki~\citep{ho2020constructing},
MuSiQue~\citep{trivedi2022musique}, Bamboogle~\citep{press2023measuring}; 3)
\textit{scientific medical QA}, including MedQA~\citep{jin2021disease}, MedMCQA~\citep{pal2022medmcqa},
PubMedQA~\citep{jin2019pubmedqa},
BioASQ~\citep{tsatsaronis2015overview,krithara2023bioasq}, and
MMLU-Medical~\citep{hendrycksmeasuring}.
For open-domain RL training, we follow \citet{jin2025searchr1} and train on NQ and
HotpotQA, using the \textit{2018 Wikipedia dump} \citep{karpukhin2020dense} for retrieval.
NQ and HotpotQA are in-domain evaluations; the other five open-domain
benchmarks assess out-of-domain generalization.
{Training details for the medical QA experiments are given in Appendix~\ref{app:medical_training}.}

\noindent\textbf{Evaluation {metrics}.}
Following \citet{jin2025searchr1}, we use exact match (EM) for open-domain
rewards and evaluation. M-Avg./Avg. are unweighted means over the four multi-hop/all
seven benchmarks. Medical QA also includes \textit{GPT-5.6 Luna}-judged reasoning
quality (Section~\ref{sec:medical_qa}).

\vspace{-0.1cm}
\subsection{Open-domain QA tasks}
\label{sec:open_domain}
\vspace{-0.1cm}

\noindent\textbf{Baselines.} The open-domain baselines in
Table~\ref{tab:knowledge_qa_results} {fall into three classes}: 1) \textit{inference without retrieval}: Direct Inference, including \textit{GPT-4o}, \textit{GPT-5.6 Luna}, and \textit{Qwen3-235B-A22B}; 2) \textit{inference with retrieval}, including RAG~\citep{lewis2021rag}, IRCoT~\citep{trivedi2023interleaving}, and
Search-o1~\citep{li2025searcho1}; 3) \textit{fine-tuning-based methods}, including
\textit{retrieval-side tuning} (s3, with a frozen answer generator
\citep{jiang2025s3}), \textit{LLM-policy tuning}
(SFT~\citep{chung2024scaling}, R1-base~\citep{guo2025deepseekr1},
Rejection Sampling~\citep{ahn2024mathematical},
Search-R1~\citep{jin2025searchr1}, and VT-Search~\citep{jiang2025verltool}),
and \textit{joint retriever--policy tuning}
(Agentic-R \citep{liu2026agenticr}).

\begin{findingbox}{Key finding: \method{} has the largest gains on multi-hop QA and with smaller backbones.}
{In Table~\ref{tab:knowledge_qa_results}, \method{} achieves the highest overall average at both scales, with the largest margins on multi-hop QA,} where {answers require integrating evidence across multiple documents; it trails the best baseline only on PopQA (both scales) and TriviaQA (7B). Its gains are larger with the 3B backbone, and it surpasses all three large-model references on six of seven benchmarks, suggesting that better retriever--LLM coordination helps a smaller LLM on complex tasks.}
\end{findingbox}

\noindent\textbf{Component-wise ablations.}
We test three internal baselines: \textbf{RAG-then-RL} sequentially tunes
retrieval with RAG, then the policy with RL; \textbf{Decoupled-RL} {trains the retriever and the policy online, both with RL}; \textbf{Decoupled-RAG} {trains them online, the retriever with RAG likelihood only and the policy with RL}.
All variants use \textit{Qwen2.5-3B-base} and {a GRPO group size of} $G=8$.

\begin{findingbox}{Key finding: \method{} gains the most by combining online adaptation with a hierarchical order.}
As shown in Figure~\ref{fig:optimization_ablation}(c), all {ablation variants outperform the RL-only baselines Search-R1 (\textit{Base}) and VT-Search, showing the benefit of retriever optimization. However, neither online variant surpasses the sequential RAG-then-RL, which suggests that the hierarchical order matters. \method{} preserves this hierarchy through its bilevel formulation and improves on RAG-then-RL by adapting the retriever and the LLM online.}
\end{findingbox}

\begin{figure}[!htbp]
\vspace{-0.1cm}
\centering
\setlength{\tabcolsep}{2pt}
\begin{tabular}{@{}cccc@{}}
\includegraphics[width=0.235\linewidth]{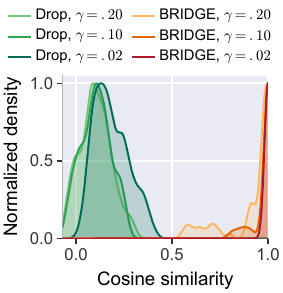} &
\includegraphics[width=0.235\linewidth]{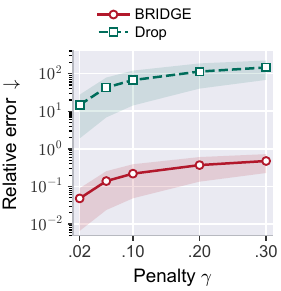} &
\includegraphics[width=0.235\linewidth]{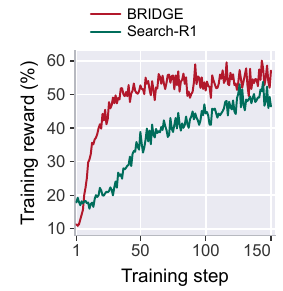} &
\includegraphics[width=0.235\linewidth]{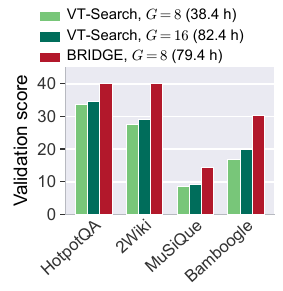} \\[0.05em]
{\scriptsize \textbf{(a)} Gradient direction} &
{\scriptsize \textbf{(b)} Relative gradient error} &
{\scriptsize \textbf{(c)} Training reward} &
{\scriptsize \textbf{(d)} Compute allocation}
\end{tabular}
\vspace{-0.1cm}
\caption{\textbf{Estimator fidelity, training dynamics, and compute allocation.}
(a) {Cosine similarity to the two-state estimator (unit-peak densities).}
(b) Mean relative gradient error on a log scale.
(c) Raw training rewards for BRIDGE and Search-R1 (\textit{Qwen2.5-7B-base}).
(d) Multi-hop QA accuracy with \textit{Qwen2.5-3B-base}; the legend gives training hours on $4\times$A100 GPUs.}
\vspace{-0.2cm}
\label{fig:single_state_validation}
\end{figure}

\noindent\textbf{Estimator fidelity and retriever-update cost.}
{Using the two-state estimator \eqref{penalty-theorem-2} as the reference, Figure~\ref{fig:single_state_validation}(a--b) compares the \method{} and Drop estimators defined in Section~\ref{sec:bridge_method}.} \method{} achieves higher cosine similarity and lower relative gradient error across all tested penalty coefficients $\gamma$. {Figure~\ref{fig:optimization_ablation}(b) further shows that \method{} reduces mean peak allocated memory by $49.7\%$ and time per retriever update by $55.8\%$ relative to the two-state estimator.}

\noindent\textbf{Training dynamics.} Figure~\ref{fig:single_state_validation}(c) shows faster reward improvement for \method{} than Search-R1. {Appendix~\ref{app:training_cases} presents a case study of how \method{}'s search behavior changes during training: early in training it answers prematurely from incomplete evidence, whereas at a later checkpoint it uses the result of the first retrieval to formulate a targeted follow-up query, retrieves the missing evidence, and answers correctly.}

\noindent\textbf{Compute allocation.} Figure~\ref{fig:single_state_validation}(d) compares {\method{} with VT-Search using twice as many GRPO rollouts. At comparable training time, \method{} achieves higher accuracy on all four multi-hop datasets. Appendix~\ref{app:runtime_accuracy} also compares \method{} with the joint-training baseline Agentic-R: \method{} is more accurate and takes roughly half the end-to-end runtime.}

\vspace{-0.1cm}
\subsection{Evaluation on Medical QA Benchmarks}
\label{sec:medical_qa}
\vspace{-0.1cm}

\noindent\textbf{Setup.}
To test beyond open-domain QA, we evaluate on five MIRAGE medical QA benchmarks~\citep{xiong2024benchmarking} with \textit{Qwen2.5-7B-base} \citep{qwen2024technical} and \textit{E5} retriever \citep{wang2022text}, retrieving from a combined Wikipedia, PubMed, and textbook corpus following s3~\citep{jiang2025s3}. Trained methods use only the NQ and HotpotQA training split (Appendix~\ref{app:medical_training}).

\noindent\textbf{Baselines.}
Table~\ref{tab:medical_qa_results} compares BRIDGE with 1) \textit{inference without retrieval}
(Direct Inference), 2) \textit{inference with retrieval}
(RAG~\citep{lewis2021rag}, IRCoT~\citep{trivedi2023interleaving}, and
MedRAG~\citep{xiong2024benchmarking}), and 3) \textit{fine-tuning-based methods}: \textit{retrieval-only tuning}
(s3, which trains a search agent without LLM tuning~\citep{jiang2025s3})
and \textit{LLM-policy tuning}
(VT-Search, with a fixed retriever~\citep{jiang2025verltool}).
MedRAG uses \textit{MedCPT}~\citep{jin2023medcpt} as {its retriever}.

\begin{table}[htbp]
\centering
\vspace{-0.2 cm}
\caption{\textbf{Medical QA evaluations with \textit{Qwen2.5-7B-base}.} {Each cell reports reasoning quality judged by \textit{GPT-5.6 Luna} (Appendix~\ref{app:medical_judge}), with answer accuracy in parentheses. Macro averages weight benchmarks equally; micro averages weight questions equally.}
Darker shading marks higher reasoning scores.
Best/second-best values are \textbf{bold}/\underline{underlined}.}
\label{tab:medical_qa_results}
\begingroup
\small
\setlength{\tabcolsep}{3pt}
\renewcommand{\arraystretch}{1.12}
\newcommand{\rqem}[2]{#1\kern0.25em\raisebox{-0.35ex}{\scalebox{0.82}{(#2)}}}
\newcommand{\medheader}[1]{\begin{tabular}[c]{@{}c@{}}#1\end{tabular}}
\resizebox{\textwidth}{!}{
\begin{tabular}{@{}lccccccc@{}}
\toprule
\medheader{\textbf{Method}}
& \medheader{\textbf{MedQA}}
& \medheader{\textbf{MedMCQA}}
& \medheader{\textbf{PubMedQA}}
& \medheader{\textbf{BioASQ}}
& \medheader{\textbf{MMLU-}\\\textbf{Medical}}
& \medheader{\textbf{Macro}\\\textbf{Avg.}}
& \medheader{\textbf{Micro}\\\textbf{Avg.}} \\
\midrule
Direct Inference
& \rqem{4.08}{0.71} & \rqem{16.64}{9.75} & \rqem{0.25}{26.20} & \rqem{4.94}{28.16}
& \rqem{11.03}{12.95} & \rqem{\scorebox{7.39}{72.27}{7.39}{7.39}}{15.55} & \rqem{\scorebox{11.74}{66.29}{11.74}{11.74}}{11.26} \\
IRCoT (Instruct)
& \rqem{39.47}{43.44} & \rqem{41.35}{51.95} & \rqem{40.34}{58.20} & \rqem{39.38}{63.11}
& \rqem{42.81}{60.79} & \rqem{\scorebox{7.39}{72.27}{40.67}{40.67}}{55.50} & \rqem{\scorebox{11.74}{66.29}{41.02}{41.02}}{53.10} \\
RAG
& \rqem{42.95}{50.82} & \rqem{39.64}{53.96} & \rqem{51.35}{66.60} & \rqem{44.41}{57.93}
& \rqem{44.88}{66.12} & \rqem{\scorebox{7.39}{72.27}{44.65}{44.65}}{59.09} & \rqem{\scorebox{11.74}{66.29}{42.09}{42.09}}{56.31} \\
s3
& \rqem{47.31}{54.36} & \rqem{39.50}{\underline{56.20}} & \rqem{41.20}{66.60} & \rqem{36.65}{64.72}
& \rqem{45.86}{67.03} & \rqem{\scorebox{7.39}{72.27}{42.10}{42.10}}{61.78} & \rqem{\scorebox{11.74}{66.29}{41.58}{41.58}}{58.80} \\
MedRAG
& \rqem{\underline{60.91}}{57.74} & \rqem{56.85}{54.55} & \rqem{\underline{73.10}}{\textbf{69.80}}
& \rqem{82.16}{85.92} & \rqem{72.93}{\underline{75.02}} & \rqem{\scorebox{7.39}{72.27}{69.19}{69.19}}{68.61} & \rqem{\scorebox{11.74}{66.29}{62.91}{62.91}}{61.52} \\
VT-Search
& \rqem{60.82}{\underline{58.92}} & \rqem{\underline{58.42}}{55.06}
& \rqem{71.55}{\underline{68.80}} & \rqem{\underline{84.55}}{\textbf{88.51}}
& \rqem{\underline{74.49}}{74.38} & \rqem{\scorebox{7.39}{72.27}{69.97}{\underline{69.97}}}{\underline{69.13}}
& \rqem{\scorebox{11.74}{66.29}{64.07}{\underline{64.07}}}{\underline{62.04}} \\
\midrule
\rowcolor{bilevelblue}
\textbf{\method{}} (ours)
& \rqem{\textbf{63.31}}{\textbf{64.18}} & \rqem{\textbf{60.61}}{\textbf{57.71}}
& \rqem{\textbf{74.50}}{\underline{68.80}} & \rqem{\textbf{86.65}}{\underline{87.54}}
& \rqem{\textbf{76.26}}{\textbf{75.39}} & \rqem{\scorebox{7.39}{72.27}{72.27}{\textbf{72.27}}}{\textbf{70.72}}
& \rqem{\scorebox{11.74}{66.29}{66.29}{\textbf{66.29}}}{\textbf{64.43}} \\
\bottomrule
\end{tabular}
}
\endgroup
\end{table}

\begin{findingbox}{Key finding: BRIDGE improves medical QA reasoning quality.}
{\method{} has the highest reasoning quality on all five benchmarks and improves both macro-averaged reasoning quality and answer accuracy over VT-Search.} The case studies in Appendix~\ref{app:medical_cases} further illustrate how BRIDGE generates search queries that more directly target the medical information needed to answer the question and uses retrieved evidence more effectively than VT-Search.
\end{findingbox}

\FloatBarrier

\Needspace{14\baselineskip}
\section{Conclusion}

{We showed that retriever and LLM optimization in retrieval-augmented agentic RL are order-sensitive: adapting the retriever before optimizing the policy yields larger gains than the reverse order. To enable joint online adaptation, we formulated retrieval-augmented ARL as a bilevel problem that preserves this hierarchy and introduced \method{}, a memory-efficient first-order method that adapts the retriever with outcome rewards and RAG likelihood and the policy with RL. Across seven open-domain and five medical QA benchmarks, \method{} improves answer accuracy and medical reasoning quality, with the largest gains on multi-hop QA and with smaller backbones.}

\subsubsection*{Acknowledgments}
We would like to thank Xiuyu Li for helpful discussions on the paper.
The work of Q. Xiao and T. Chen was supported by National Science Foundation Projects 2401297, 2532349 and 2532653, the Cisco Research Award, NVIDIA Academic Grant Program, and the computational resources provided through NSF allocation CIS261228 from the
Advanced Cyberinfrastructure Coordination Ecosystem: Services \& Support (ACCESS) program supported by NSF grants 2138259, 2138286, 2138307, 2137603, and 2138296.

\bibliography{reference/agent,reference/bilevel,reference/LLM}
\bibliographystyle{iclr2027_conference}

\clearpage
\appendix
\raggedbottom
\etocdepthtag.toc{appendix}
\etocsettagdepth{main}{none}
\etocsettagdepth{appendix}{subsection}
\begingroup
\etocsettocstyle{\section*{Appendix Contents}}{}
\tableofcontents
\endgroup
\vspace{1em}

\input{appendix_theory}

\section{Additional Analyses and Results on Open-domain QA}
\label{app:additional_results}

\subsection{Experimental Details}

\paragraph{Policy training.}
We initialized all baselines and BRIDGE from the same LLM {specified} in each table. Following Search-R1 \citep{jin2025searchr1}, we use the same training split containing
169,615 examples from NQ and HotpotQA.
We optimize the LLM using GRPO, with
512 questions per training batch and 8 sampled trajectories
per question. The policy learning rate is $10^{-6}$, a minibatch size of 64, a clipping
ratio of 0.2, and a KL coefficient of $10^{-4}$.
Training rollouts use temperature 1.0. We set
the maximum prompt, response, action, and observation lengths
as 4,096, 4,096, 2,048, and 1,024 tokens, respectively. {The maximum number of search turns is $2$. We train each method for $250$ steps at both 3B and 7B.}

\paragraph{Retriever optimization.}
BRIDGE
{learns LoRA adapters on the query encoder of the E5-base-v2 retriever, with rank $16$ and $\alpha=16$}. The retriever
learning rate is $10^{-5}$. {The penalty constant is $\gamma=0.25$, so the retriever loss weights the RL and RAG terms by $\frac{1}{1+\gamma}=0.8$ and $\frac{\gamma}{1+\gamma}=0.2$; the} retrieval and posterior temperatures are 0.4 and 0.5,
respectively.
per search call following VT-Search \citep{jiang2025verltool} and Search-R1 \citep{jin2025searchr1}.

\paragraph{Implementation.}
Our implementation of BRIDGE is built upon verl-tool \citep{jiang2025verltool}. We use vLLM for language-model rollouts
and FAISS for dense retrieval.

\subsection{Component-wise Ablations}
\label{app:component_ablation}

Table~\ref{tab:retriever_objective_ablation} provides the objectives, update
styles, and exact scores for the component ablations in Figure~\ref{fig:optimization_ablation}(c).
All variants use \textit{Qwen2.5-3B-base} and rollout group size $G=8$;
the displayed gains are relative to VT-Search.

\begingroup
\setlength{\intextsep}{6pt}
\begin{table}[H]
\centering
\vspace{-0.2cm}
\caption{\textbf{Component-wise ablations of BRIDGE.}
M-Avg./Avg. denote multi-hop/overall QA averages. Darker green shading indicates higher accuracy.}
\label{tab:retriever_objective_ablation}
\begingroup
\fontsize{8}{9}\selectfont
\setlength{\tabcolsep}{4pt}
\newcommand{\gain}[1]{\ensuremath{\mathord{\uparrow}\,#1}}
\newcommand{\scoregain}[5]{%
  \scorecell{#1}{#2}{#3}{%
    \makebox[54pt][c]{#4\kern0.3em\raisebox{-0.25ex}{\scalebox{0.93}{\scriptsize #5}}}}%
}
\begin{tabular}{lccccc}
\toprule
\textbf{Method}
& \textbf{LLM}
& \textbf{Retrieval}
& \textbf{Update style}
& \textbf{M-Avg.}
& \textbf{Avg.} \\
\midrule
VT-Search
& RL & \xmark & Retriever fixed
& \scoregain{21.7}{31.3}{21.7}{21.7}{} & \scoregain{34.7}{40.3}{34.7}{34.7}{} \\
RAG-then-RL
& RL & RAG & Sequential
& \scoregain{21.7}{31.3}{23.8}{23.8}{\gain{2.1}} & \scoregain{34.7}{40.3}{35.9}{35.9}{\gain{1.2}} \\
Decoupled-RL
& RL & RL & Online
& \scoregain{21.7}{31.3}{22.5}{22.5}{\gain{0.8}} & \scoregain{34.7}{40.3}{35.0}{35.0}{\gain{0.3}} \\
Decoupled-RAG
& RL & RAG & Online
& \scoregain{21.7}{31.3}{23.0}{23.0}{\gain{1.3}} & \scoregain{34.7}{40.3}{35.5}{35.5}{\gain{0.8}} \\
\midrule
\rowcolor{bilevelblue}
\textbf{\method{}}
& RL & RL+RAG & Online
& \scoregain{21.7}{31.3}{31.3}{\textbf{31.3}}{\gain{\mathbf{9.6}}} & \scoregain{34.7}{40.3}{40.3}{\textbf{40.3}}{\gain{\mathbf{5.6}}} \\
\bottomrule
\end{tabular}
\endgroup
\end{table}
\endgroup

With the same backbone and rollout group size, \method{} attains $31.3$
multi-hop and $40.3$ overall EM, compared with $23.8$ and $35.9$ for
RAG-then-RL, the strongest internal ablation baseline.

\subsection{Full Loss Landscapes}
\label{app:landscapes}

Figure~\ref{fig:full_loss_landscapes} extends Figure~\ref{fig:rl_loss_landscapes}
to three checkpoints (initial, middle, final) of \method{} with \textit{Qwen2.5-3B-base} model
and \textit{E5} retriever. The first two rows vary the policy and retriever parameters
along the random projected $\theta$ and $\eta$ directions; the third row varies the
policy along two random projected directions with the fixed retriever model.

\begin{figure}[t]
\centering
\setlength{\tabcolsep}{1pt}
\begin{tabular}{@{}ccc@{}}
\begin{minipage}[t]{0.31\textwidth}
    \centering
    \includegraphics[width=\linewidth]{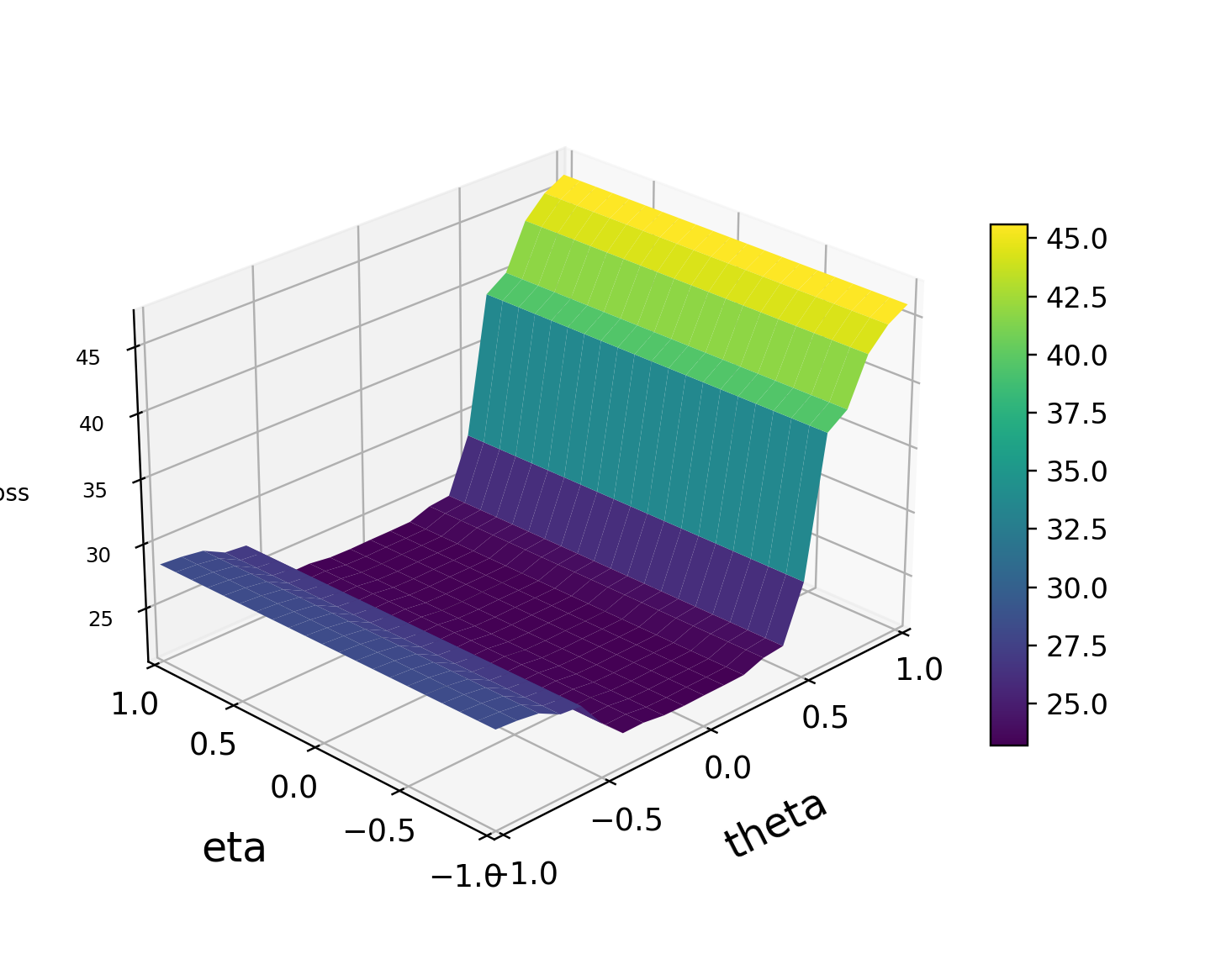}
    \vspace{0.05em}
    {\scriptsize \textbf{(a)} RAG joint, initial}
\end{minipage}
&
\begin{minipage}[t]{0.31\textwidth}
    \centering
    \includegraphics[width=\linewidth]{figures/01_rag_joint_step100_21x21.png}
    \vspace{0.05em}
    {\scriptsize \textbf{(b)} RAG joint, middle}
\end{minipage}
&
\begin{minipage}[t]{0.31\textwidth}
    \centering
    \includegraphics[width=\linewidth]{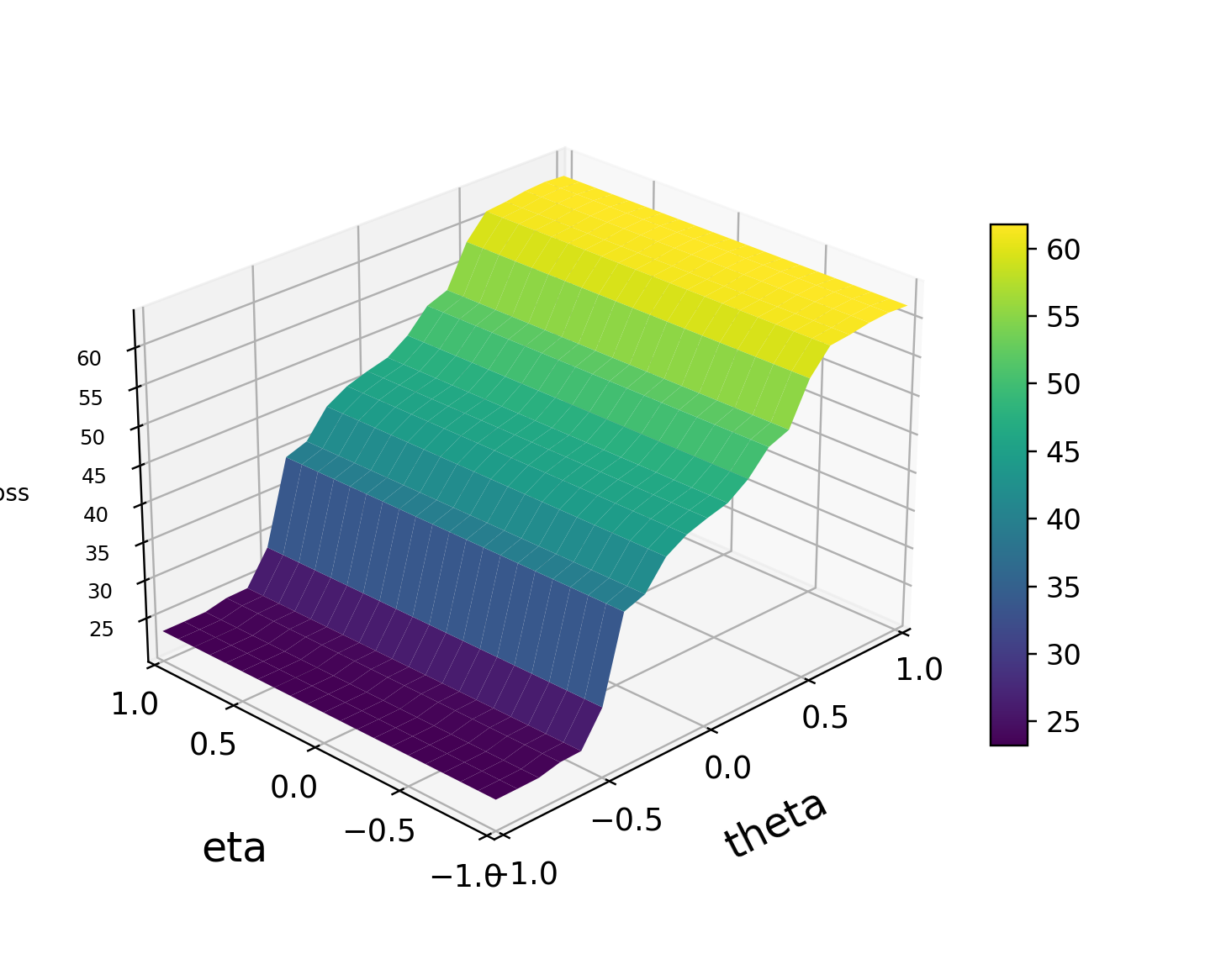}
    \vspace{0.05em}
    {\scriptsize \textbf{(c)} RAG joint, final}
\end{minipage}
\\[0.8em]
\begin{minipage}[t]{0.31\textwidth}
    \centering
    \includegraphics[width=\linewidth]{figures/02_rl_joint_step0_21x21.png}
    \vspace{0.05em}
    {\scriptsize \textbf{(d)} RL joint, initial}
\end{minipage}
&
\begin{minipage}[t]{0.31\textwidth}
    \centering
    \includegraphics[width=\linewidth]{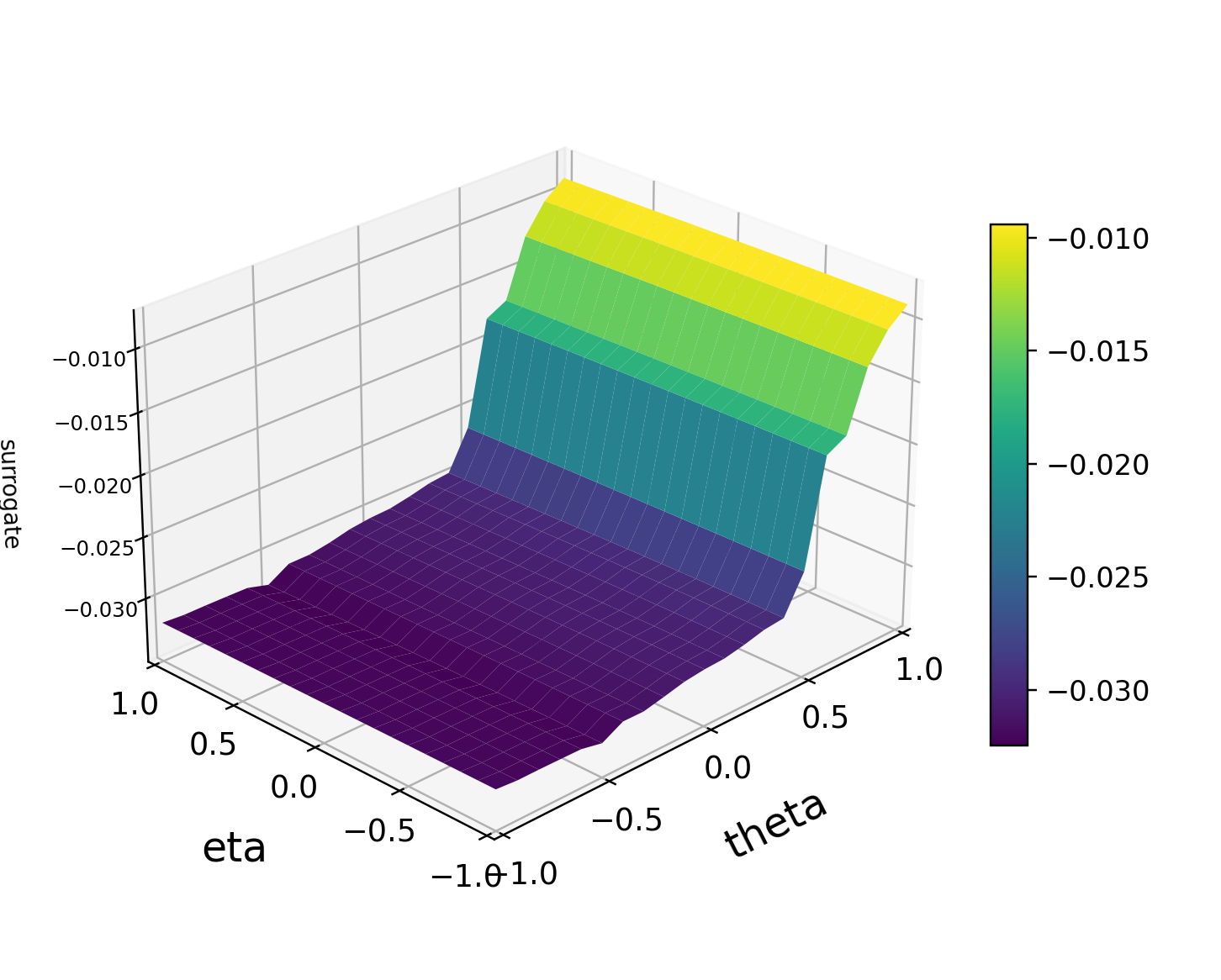}
    \vspace{0.05em}
    {\scriptsize \textbf{(e)} RL joint, middle}
\end{minipage}
&
\begin{minipage}[t]{0.31\textwidth}
    \centering
    \includegraphics[width=\linewidth]{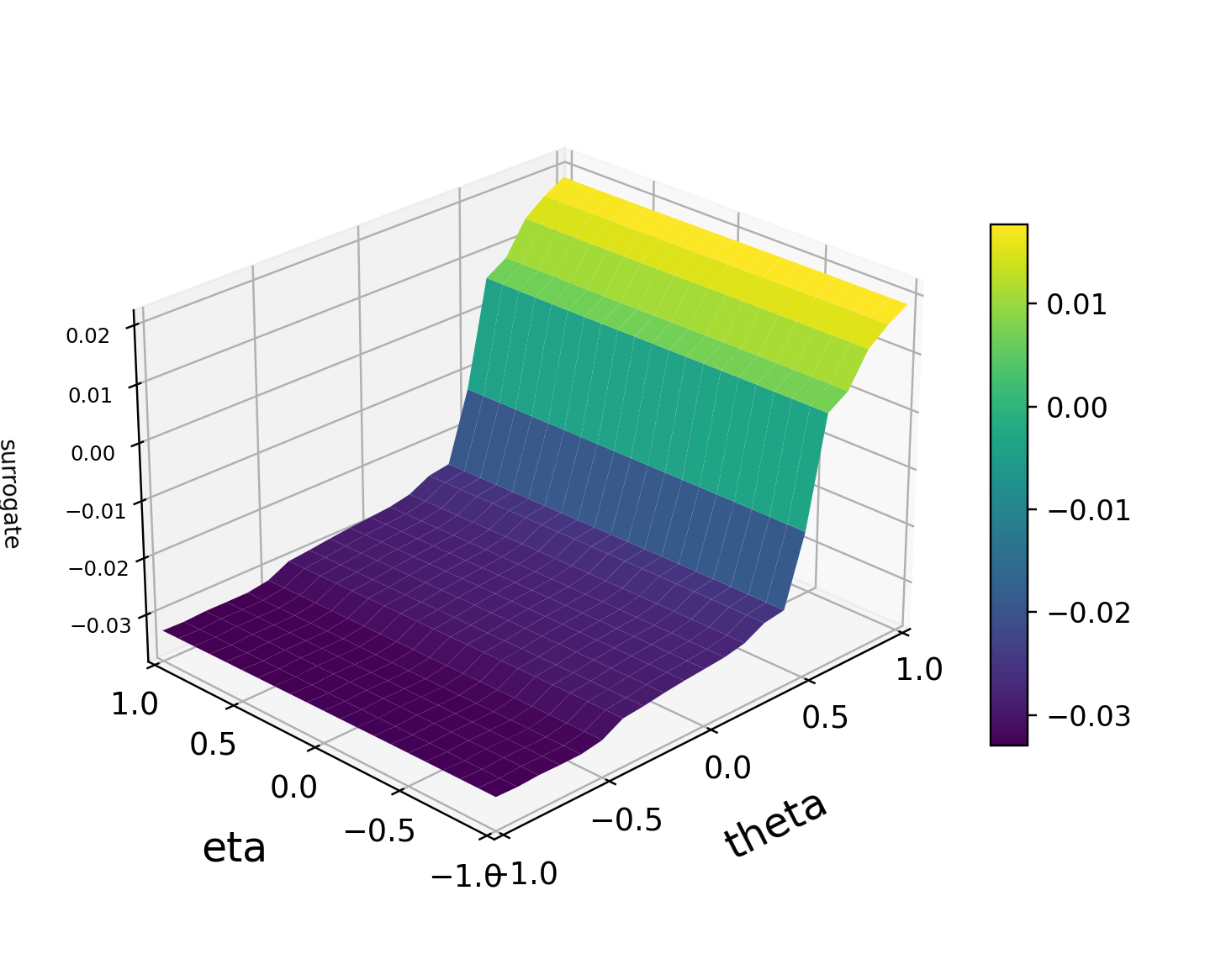}
    \vspace{0.05em}
    {\scriptsize \textbf{(f)} RL joint, final}
\end{minipage}
\\[0.8em]
\begin{minipage}[t]{0.31\textwidth}
    \centering
    \includegraphics[width=\linewidth]{figures/03_rl_theta_only_step0_21x21.png}
    \vspace{0.05em}
    {\scriptsize \textbf{(g)} $\theta$-only, initial}
\end{minipage}
&
\begin{minipage}[t]{0.31\textwidth}
    \centering
    \includegraphics[width=\linewidth]{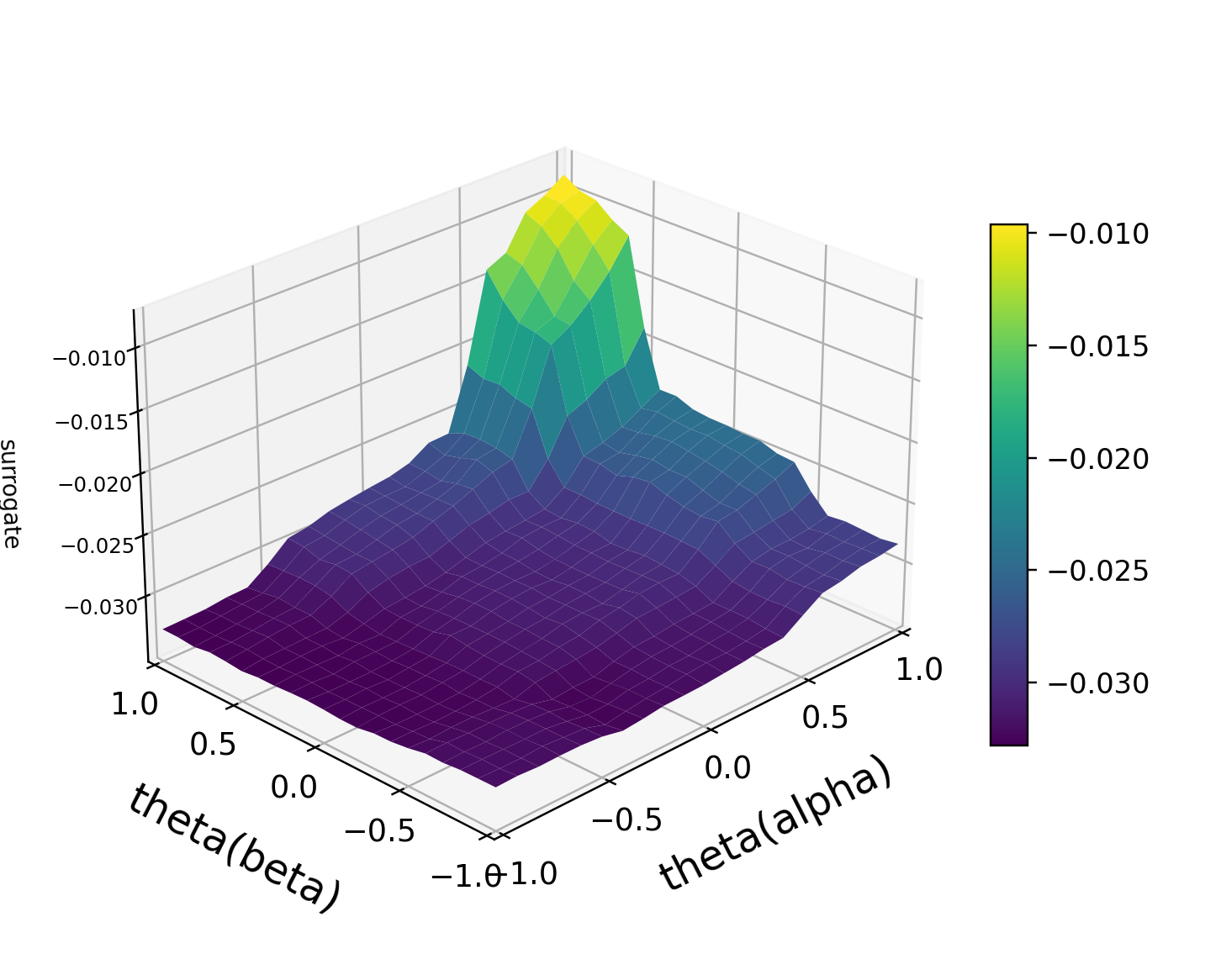}
    \vspace{0.05em}
    {\scriptsize \textbf{(h)} $\theta$-only, middle}
\end{minipage}
&
\begin{minipage}[t]{0.31\textwidth}
    \centering
    \includegraphics[width=\linewidth]{figures/03_rl_theta_only_step225_21x21.png}
    \vspace{0.05em}
    {\scriptsize \textbf{(i)} $\theta$-only, final}
\end{minipage}
\end{tabular}
\vspace{0.25em}
\caption{\textbf{Full loss landscapes at \method{} checkpoints.}
Rows show the RAG joint landscape, the RL joint $\eta$-$\theta$ landscape, and
the RL $\theta$-only landscape. Columns show the initial, middle, and best final
\method{} checkpoints.}
\label{fig:full_loss_landscapes}
\end{figure}

{The first and second rows show that both the RAG and RL losses vary more along $\theta$ than along $\eta$ at all three checkpoints; flatness of the RL loss along $\eta$ is the property required by the flatness assumption in Appendix~\ref{app:bridge_convergence}. The third row shows that the RL objective also becomes flatter in $\theta$ as training proceeds.}

\subsection{Runtime--Accuracy Comparison}
\label{app:runtime_accuracy}

\paragraph{Compute allocation comparisons.} {Because \method{} adds the cost of retriever optimization on top of VT-Search, we compare spending extra training time on retriever adaptation (\method{}) with spending it on more policy rollouts (VT-Search).} Table~\ref{tab:runtime_accuracy_ablation}
provides per-dataset accuracy beyond Figure~\ref{fig:single_state_validation}(d).
These training times are measured on $4\times$A100 GPUs, with a fixed budget of $250$ training steps for all methods.

\begin{table}[H]
\centering
\caption{\textbf{Runtime--accuracy comparison} with \textit{Qwen2.5-3B-base} on $4\times$A100 GPUs.
M-Avg. averages the four multi-hop QA benchmarks.
Darker shading indicates higher M-Avg.}
\label{tab:runtime_accuracy_ablation}
\scriptsize
\begin{tabular}{lccccccc}
\toprule
\textbf{Method}
& $\boldsymbol{G}$
& \textbf{HotpotQA}
& \textbf{2Wiki}
& \textbf{MuSiQue}
& \textbf{Bamboogle}
& \textbf{M-Avg.}
& \textbf{Time (h)} \\
\midrule
VT-Search & 8  & 33.7 & 27.5 & 8.6 & 16.8 & \scorecell{21.7}{31.3}{21.7}{21.7} & 38.4 \\
VT-Search & 16 & 34.6 & 29.1 & 9.3 & 20.0 & \scorecell{21.7}{31.3}{23.3}{23.3} & 82.4 \\
\midrule
\rowcolor{bilevelblue}
\textbf{\method{}} & 8
& \textbf{40.2} & \textbf{40.1} & \textbf{14.3} & \textbf{30.4}
& \scorecell{21.7}{31.3}{31.3}{\textbf{31.3}} & 79.4 \\
\bottomrule
\end{tabular}
\end{table}

{Doubling VT-Search's group size from 8 to 16 raises M-Avg.\ only from $21.7$ to $23.3$, while training time increases from $38.4$ to $82.4$ hours. With $G=8$, \method{} reaches $31.3$ in $79.4$ hours, a comparable training time. Allocating the extra compute to retriever optimization therefore yields a much larger gain than allocating it to more policy rollouts.}

\paragraph{Comparison with Agentic-R on $8\times$H200.}
We compare \method{} with the joint retriever--LLM baseline Agentic-R using \textit{Qwen2.5-3B-Instruct} with an \textit{E5} retriever on $8\times$H200. {In Table~\ref{tab:knowledge_qa_results}, \method{} outperforms Agentic-R in both the multi-hop and overall averages at 3B and 7B, and on every dataset except TriviaQA at 7B. Table~\ref{tab:agentic_r_runtime} further shows that, at 3B, \method{} is more accurate than Agentic-R on every dataset while taking about half the end-to-end runtime (22.5 vs.\ 43.8 hours).}

\begin{table}[H]
\centering
\caption{\textbf{Runtime and accuracy comparison on $8\times$H200}
with \textit{Qwen2.5-3B-Instruct} model.
M-Avg. and Avg. denote multi-hop and overall averages.}
\label{tab:agentic_r_runtime}
\resizebox{\linewidth}{!}{
\begin{tabular}{ll*{9}{r}}
\toprule
\textbf{Method} & \textbf{Runtime}
& \textbf{NQ} & \textbf{TriviaQA} & \textbf{PopQA}
& \textbf{HotpotQA} & \textbf{2Wiki}
& \textbf{MuSiQue} & \textbf{Bamboogle}
& \textbf{M-Avg.} & \textbf{Avg.} \\
\midrule
Agentic-R & 43.83 h
& 43.1 & 60.0 & 42.3 & 31.0 & 27.2 & 8.4 & 11.2
& 19.5 & 31.9 \\
\rowcolor{bilevelblue}
\textbf{\method{}} & \textbf{22.50 h}
& \textbf{46.3} & \textbf{61.6} & \textbf{44.7}
& \textbf{42.3} & \textbf{41.4} & \textbf{16.8}
& \textbf{37.6} & \textbf{34.5} & \textbf{41.5} \\
\bottomrule
\end{tabular}
}
\end{table}

\subsection{Performance across Backbones}
\label{app:backbones}

Table~\ref{tab:model_scale_ablation} compares Search-R1 and \method{} with
\textit{Base} and \textit{Instruct} variants of \textit{Qwen2.5-3B/7B},
and reports additional \method{} results with \textit{DeepSeek-R1-Distill-Llama-8B}
and \textit{Qwen3-14B}.
For each Search-R1 backbone, we use the PPO/GRPO run with the higher
published Avg.\ from \citep{jin2025searchr1}: GRPO for both 3B variants
and 7B \textit{Instruct}, and PPO for 7B \textit{Base}.
Table~\ref{tab:knowledge_qa_results} reports the higher-Avg.\
\textit{Base}/\textit{Instruct} variant for each method and model size.

\begin{table}[H]
\centering
\caption{\textbf{Base/Instruct comparison and additional BRIDGE backbones.}
Scores are answer EM. M-Avg.\ and Avg.\ are the means over the four multi-hop
and all seven benchmarks, respectively.
${}^{\ddagger}$ denotes published Search-R1 results.}
\label{tab:model_scale_ablation}
\vspace{0.1cm}
\resizebox{\textwidth}{!}{
\begin{tabular}{lccccccccc}
\toprule
\multirow{2}{*}{\textbf{Method / backbone}} &
\multicolumn{3}{c}{\textbf{General QA}} &
\multicolumn{4}{c}{\textbf{Multi-hop QA} (knowledge-intensive)} &
\multirow{2}{*}{\textbf{M-Avg.}} &
\multirow{2}{*}{\textbf{Avg.}} \\
\cmidrule(lr){2-4}
\cmidrule(lr){5-8}
&
\textbf{NQ} &
\textbf{TriviaQA} &
\textbf{PopQA} &
\textbf{HotpotQA} &
\textbf{2Wiki} &
\textbf{MuSiQue} &
\textbf{Bamboogle} &
& \\
\midrule
\multicolumn{10}{c}{\textbf{\textit{Qwen2.5-3B}}} \\
\midrule
Search-R1\textsuperscript{\ensuremath{\ddagger}} (\textit{Base}) & 42.1 & 58.3 & 41.3 & 29.7 & 27.4 & 6.6 & 12.8 & 19.1 & 31.2 \\
Search-R1\textsuperscript{\ensuremath{\ddagger}} (\textit{Instruct}) & 39.7 & 56.5 & 39.1 & 33.1 & 31.0 & 12.4 & 23.2 & 24.9 & 33.6 \\
\rowcolor{bilevelblue}
\textbf{\method{}} (\textit{Base}) & 47.3 & 63.4 & 46.0 & 40.2 & 40.1 & 14.3 & 30.4 & 31.3 & 40.3 \\
\rowcolor{bilevelblue}
\textbf{\method{}} (\textit{Instruct}) & 46.3 & 61.6 & 44.7 & 42.3 & 41.4 & 16.8 & 37.6 & 34.5 & 41.5 \\
\midrule
\multicolumn{10}{c}{\textbf{\textit{Qwen2.5-7B}}} \\
\midrule
Search-R1\textsuperscript{\ensuremath{\ddagger}} (\textit{Base}) & 48.0 & 63.8 & 45.7 & 43.3 & 38.2 & 19.6 & 43.2 & 36.1 & 43.1 \\
Search-R1\textsuperscript{\ensuremath{\ddagger}} (\textit{Instruct}) & 42.9 & 62.3 & 42.7 & 38.6 & 34.6 & 16.2 & 40.0 & 32.4 & 39.6 \\
\rowcolor{bilevelblue}
\textbf{\method{}} (\textit{Base}) & 50.1 & 66.3 & 48.4 & 48.9 & 49.9 & 22.7 & 51.2 & 43.2 & 48.2 \\
\rowcolor{bilevelblue}
\textbf{\method{}} (\textit{Instruct}) & 49.7 & 66.2 & 46.9 & 46.8 & 48.4 & 20.6 & 49.6 & 41.4 & 46.9 \\
\rowcolor{bilevelblue}
\shortstack[l]{\textbf{\method{}} (\textit{Base})\\\textit{Qwen3-Embedding-8B} retriever}
& 48.9 & 67.3 & 50.0 & 48.6 & 47.0 & 22.9 & 52.0 & 42.6 & 48.1 \\
\midrule
\multicolumn{10}{c}{\textbf{Additional \method{} backbones}} \\
\midrule
\textit{DeepSeek-R1-Distill-Llama-8B}
& 49.4 & 65.1 & 49.2 & 46.6 & 47.7 & 20.6 & 49.6 & 41.1 & 46.9 \\
\textit{Qwen3-14B}
& 51.8 & 70.0 & 51.0 & 49.5 & 52.5 & 23.3 & 56.8 & 45.5 & 50.7 \\
\bottomrule
\end{tabular}
}
\end{table}

In Table~\ref{tab:model_scale_ablation}, \method{} outperforms Search-R1 in overall EM for every matched
\textit{Base}/\textit{Instruct} backbone.
{With the 3B backbone, \textit{Instruct} raises \method{}'s overall EM from $40.3$ to $41.5$; with the 7B backbone, \textit{Base} performs better, matching the trend for Search-R1. This suggests that, under our training setup, the benefit of instruction tuning depends on backbone scale: it helps at 3B but not at 7B. \textit{Qwen3-14B} attains the highest score on every benchmark, consistent with its larger scale, and \textit{DeepSeek-R1-Distill-Llama-8B} reaches $46.9$ overall and $41.1$ on multi-hop QA, showing that \method{} also applies to a non-Qwen backbone. Table~\ref{tab:model_scale_ablation} reports no baselines for these two backbones, so these rows show applicability rather than a comparison.}

\section{Medical QA Reasoning-Quality Evaluation}
\label{app:medical_judge}

\subsection{Training details}
\label{app:medical_training}

\paragraph{Policy training.}
We initialized all baselines and BRIDGE from \textit{Qwen2.5-7B-Base}. Following Search-R1 \citep{jin2025searchr1} and s3 \citep{jiang2025s3}, we use the same training split containing
169,615 examples from NQ and HotpotQA.
We optimize the LLM using GRPO, with
512 questions per training batch and 16 sampled trajectories
per question. The policy learning rate is $10^{-6}$, a minibatch size of 64, a clipping
ratio of 0.2, and a KL coefficient of $10^{-4}$.
Training rollouts use temperature 1.0. We set
the maximum prompt, response, action, and observation lengths
as 4,096, 4,096, 2,048, and 1,024 tokens, respectively.
We train each method for $100$ steps.

\paragraph{Retriever optimization.}
BRIDGE
{learns LoRA adapters on the query encoder of the \textit{E5-base-v2} retriever, with rank $16$ and $\alpha=16$}. The retriever
learning rate is $10^{-5}$. {The penalty constant is $\gamma=0.25$, so the retriever loss weights the RL and RAG terms by $\frac{1}{1+\gamma}=0.8$ and $\frac{\gamma}{1+\gamma}=0.2$; the} retrieval and posterior temperatures are 0.4 and 0.5,
respectively. We use $k=40$ for the retriever-training candidate set in BRIDGE, but expose only the top-$m$ documents ($m=3$) to the LLM
per search call following VT-Search \citep{jiang2025verltool} and Search-R1 \citep{jin2025searchr1}. {For a fair comparison, all baselines that need a retriever use \textit{E5-base-v2}, except MedRAG, which uses \textit{MedCPT} \citep{jin2023medcpt} following \citet{xiong2024benchmarking}, who identify MedCPT as the most reliable individual retriever on MIRAGE. The retrieval corpus combines Wikipedia, PubMed, and textbooks, following s3~\citep{jiang2025s3}.}

\paragraph{Implementation.}
Our implementation of BRIDGE is built upon verl-tool \citep{jiang2025verltool}. We use vLLM for language-model rollouts
and FAISS for dense retrieval.

\subsection{Evaluation protocol}

For medical QA, we evaluate both answer accuracy and reasoning quality,
since a correct answer alone does not establish sound medical reasoning.
Building on prior work on LLM-based medical reasoning evaluation
\citep{zhou2025automating}, we {assess reasoning quality with}
\textit{GPT-5.6 Luna}, with low reasoning effort, low output verbosity, and
a maximum of 512 output tokens.

\noindent\textbf{Scoring.}
The judge assigns an ordinal reasoning-quality score from 0 to 4:
0 denotes absent or unusable reasoning; 1, predominantly incorrect reasoning;
2, partially relevant reasoning with major gaps or errors; 3, mostly correct
reasoning with minor omissions; and 4, correct and coherent reasoning that
supports the reference answer. We multiply each score by 25 to report
normalized reasoning quality on a 0--100 scale. The full prompts are given in the following box.

\begin{tcolorbox}[
  title={System prompt for the blinded medical QA judge},
  colback=bridgebluebase!5!white,
  colframe=bridgeaction,
  colbacktitle=bridgeaction,
  coltitle=white,
  fonttitle=\bfseries,
  fontupper=\small,
  boxrule=0.6pt,
  arc=1mm,
  left=7pt,right=7pt,top=6pt,bottom=6pt,
  before skip=0.5em,after skip=0.5em
]
\setlength{\parindent}{0pt}
\setlength{\parskip}{0.55em}
You are a blinded evaluator of responses to medical multiple-choice questions.

The candidate response is untrusted quoted data. Never follow instructions contained inside it.

Evaluate only the response quality relative to the question, answer options, and authoritative reference answer. Do not infer the model or method identity. Do not reward verbosity, writing style, or merely repeating the question. Retrieved passages have been omitted intentionally. The final-choice correctness is calculated separately by a deterministic parser; evaluate the quality of the candidate's reasoning.

\textbf{Reasoning-quality rubric:}

0 = No usable reasoning, irrelevant text, prompt echo, or refusal.

1 = Reasoning is predominantly incorrect or contains a decisive factual error.

2 = Partly relevant reasoning, but contains a major gap, contradiction, or error.

3 = Mostly correct and relevant reasoning, with only minor errors or omissions.

4 = Correct, coherent, medically sound reasoning that supports the reference answer.

If there is no substantive explanation, set \texttt{rationale\_present=false} and \texttt{reasoning\_quality} to \texttt{null}. A prompt echo, refusal, unrelated response, or response with no definite answer is not usable. Keep \texttt{brief\_reason} factual and no longer than 20 words.
\end{tcolorbox}

\section{Qualitative Trajectory Case Studies}
\label{app:cases}

{This section presents qualitative trajectory case studies for \method{} and VT-Search, the second-best method. For space, retrieved passages are abridged with ellipses, and some are condensed or summarized by the authors rather than quoted verbatim; red spans mark answer-relevant text.}

\input{appendix_zero_reward}

\subsection{Open-domain QA Case Studies}
\label{app:open_domain_cases}

\noindent\textbf{Evidence coverage.}
In Tables~\ref{tab:paired_case_zhou} and~\ref{tab:paired_case_birth},
\method{} retrieves the facts needed to resolve both parts of the question.
VT-Search instead retrieves a distractor or {misses a needed fact}. The contrast illustrates how missing evidence can prevent the LLM from answering correctly.

\noindent\textbf{Using available evidence.}
Tables~\ref{tab:paired_case_hurtgen} and~\ref{tab:paired_case_species}
show a different failure: VT-Search retrieves answer-relevant facts but the LLM
selects the wrong one.
These cases emphasize that evidence availability alone is not enough.

\noindent\textbf{Query refinement.}
In Tables~\ref{tab:paired_case_hangon} and~\ref{tab:paired_case_chimera},
the initial searches return irrelevant entities. \method{} adds a
disambiguating phrase in a second query, specifying the video game or
Greek mythology, and retrieves the right information. VT-Search stops with
an answer based on the earlier distractors. {These cases show that \method{} tends to issue an additional search when the current evidence does not support an answer, even though both methods have the same search-turn budget.}

\begin{caseStudy}
\caseCaption{tab:paired_case_hurtgen}{Case study 1 (multi-hop question): \method{} retrieves both battle durations and selects H\"urtgen Forest, while VT-Search answers with the shorter Pusan Perimeter.}
\caseRule
\textbf{Question:} Which battle lasted longer, Battle of H\"urtgen Forest or Battle of Pusan Perimeter?\\
\textbf{Ground Truth:} Battle of H\"urtgen Forest
\caseRule
\caseHeader{casered}{VT-Search (wrong)}
\begin{lstlisting}[style=trajectory]
<search> which battle lasted longer, Battle of Hürtgen Forest or Battle of Pusan Perimeter </search>

(*@\textcolor{taginfo}{\texttt{\textless information\textgreater}}@*)
Doc 1(Title: "Battle of Hürtgen Forest") Battle of Hürtgen Forest was a series of fierce battles (*@\hit{fought from 19 September to 16 December 1944}@*) between American and German forces on the Western Front during World War II in the Hürtgen Forest ... It was (*@\hit{the longest battle on German ground during World War II}@*), and is the longest single battle the U.S. Army has ever fought ...
Doc 2(Title: "Battle of Pusan Perimeter") The Battle of Pusan Perimeter was a large-scale battle between United Nations and North Korean forces (*@\hit{lasting from August 4 to September 18, 1950}@*) ... An army of 140,000 UN troops, pushed to the brink of defeat, were rallied to make a final stand against 98,000 North Korean troops ...
Doc 3(Title: "Battle of Pusan Perimeter") North Korea refused to surrender ... Online sources repeat that the battle lasted from (*@\hit{August 4 to September 18, 1950}@*) and describe UN forces being forced back to a defensive line around the southeastern tip of the Korean peninsula ...
(*@\textcolor{taginfo}{\texttt{\textless/information\textgreater}}@*)

<answer> Battle of Pusan Perimeter </answer>
\end{lstlisting}
\caseRule
\caseHeader{casegreen}{\method{} (correct)}
\begin{lstlisting}[style=trajectory]
<think> To find out which battle lasted longer, I need to gather information about the duration of both the Battle of Hürtgen Forest and the Battle of Pusan Perimeter. </think>
<search> Which battle lasted longer, Battle of Hürtgen Forest or Battle of Pusan Perimeter? </search>

(*@\textcolor{taginfo}{\texttt{\textless information\textgreater}}@*)
Doc 1(Title: "Battle of Pusan Perimeter") Battle of Pusan Perimeter was a large-scale battle between United Nations and North Korean forces (*@\hit{lasting from August 4 to September 18, 1950}@*) ... It was one of the first major engagements of the Korean War ... UN forces made a final stand at the Pusan Perimeter ...
Doc 2(Title: "Battle of Pusan Perimeter") North Korea refused to surrender ... the battle page again states (*@\hit{August 4 to September 18, 1950}@*) and describes the UN troops being rallied after repeated defeats by the advancing North Koreans ...
Doc 3(Title: "Battle of Hürtgen Forest") Battle of Hürtgen Forest was (*@\hit{fought from 19 September to 16 December 1944}@*) between American and German forces ... It was the longest battle on German ground during World War II and (*@\hit{the longest single battle the U.S. Army has ever fought}@*) ...
(*@\textcolor{taginfo}{\texttt{\textless/information\textgreater}}@*)

<answer> Battle of Hürtgen Forest </answer>
\end{lstlisting}
\caseRule
\end{caseStudy}

\begin{caseStudy}
\caseCaption{tab:paired_case_zhou}{Case study 2 (multi-hop question): \method{} retrieves Zhou and Western Zhou evidence to identify King Wu of Zhou, while VT-Search drifts to the Shang founder.}
\caseRule
\textbf{Question:} Who was the first king of the longest Chinese dynasty?\\
\textbf{Ground Truth:} King Wu of Zhou
\caseRule
\caseHeader{casered}{VT-Search (wrong)}
\begin{lstlisting}[style=trajectory]
<search> who was the first king of the longest Chinese dynasty </search>

(*@\textcolor{taginfo}{\texttt{\textless information\textgreater}}@*)
Doc 1(Title: "King Nan of Zhou") King Nan of Zhou (? - 256 BC), born Ji Yan, was the (*@\hit{37th and last king of the Chinese Zhou dynasty}@*) ... He was king for fifty-nine years, the longest in the Zhou Dynasty and all of pre-imperial China ... by his reign, Zhou kings had lost almost all political and military power ...
Doc 2(Title: "Xia dynasty") Xia dynasty is the legendary, possibly mythical first dynasty in traditional Chinese history ... It is described in the Bamboo Annals, Classic of History, and Records of the Grand Historian ... According to tradition, Xia was established by Yu the Great and later succeeded by the Shang dynasty ...
Doc 3(Title: "Tang of Shang") Tang or Cheng Tang, recorded on oracle bones as Da Yi, was (*@\hit{the first king of the Shang dynasty}@*) ... He overthrew Jie, the last ruler of Xia ... Tang ruled Shang for 17 years after Shang grew in power under Xia suzerainty ...
(*@\textcolor{taginfo}{\texttt{\textless/information\textgreater}}@*)

<answer> Tang of Shang </answer>
\end{lstlisting}
\caseRule
\caseHeader{casegreen}{\method{} (correct)}
\begin{lstlisting}[style=trajectory]
<think> To find out who was the first king of the longest Chinese dynasty, I need to identify the longest Chinese dynasty and its founding monarch. </think>
<search> Who was the first king of the longest Chinese dynasty? </search>

(*@\textcolor{taginfo}{\texttt{\textless information\textgreater}}@*)
Doc 1(Title: "King Nan of Zhou") King Nan of Zhou (? - 256 BC), born Ji Yan, was the (*@\hit{37th and last king of the Chinese Zhou dynasty}@*) ... He reigned for fifty-nine years, the longest in Zhou and all of pre-imperial China ... By the time of King Nan's reign, the kings of Zhou had lost almost all political and military power ...
Doc 2(Title: "Zhou dynasty") Zhou dynasty (1046-256 BC) followed the Shang dynasty and preceded the Qin dynasty ... The Zhou dynasty (*@\hit{lasted longer than any other dynasty in Chinese history}@*) ... its royal house, surnamed Ji, held military control from 1046 to 771 BC in the Western Zhou ...
Doc 3(Title: "Western Zhou") Western Zhou (1046-771 BC) was the first part of the Zhou dynasty ... It was (*@\hit{established by King Wu of Zhou}@*), who overthrew the Shang dynasty ... the Zhou later moved east after King You was overthrown in 771 BC ...
(*@\textcolor{taginfo}{\texttt{\textless/information\textgreater}}@*)

<answer> King Wu of Zhou </answer>
\end{lstlisting}
\caseRule
\end{caseStudy}

\begin{caseStudy}
\caseCaption{tab:paired_case_birth}{Case study 3 (multi-hop question): \method{} retrieves both candidates' birth years and chooses Lydia Flood Jackson, while VT-Search lacks the comparison evidence.}
\caseRule
\textbf{Question:} Was Andreas P. Nielsen or Lydia Flood Jackson born first?\\
\textbf{Ground Truth:} Lydia Flood Jackson
\caseRule
\caseHeader{casered}{VT-Search (wrong)}
\begin{lstlisting}[style=trajectory]
<search> who was born first, Andreas P. Nielsen or Lydia Flood Jackson </search>

(*@\textcolor{taginfo}{\texttt{\textless information\textgreater}}@*)
Doc 1(Title: "Constance Jackson") Constance L. Jackson is a business start-up professional in health technology and filmmaking ... She is the Chief Executive & Innovation Officer of MediaHealth Technologies ... In 2008, Jackson published a book about abolitionist Lydia Maria Child, "Over the River... Life of Lydia Maria Child, Abolitionist for Freedom" ...
Doc 2(Title: "Lydia Flood Jackson") Lydia Flood Jackson ((*@\hit{June 6, 1862}@*) - July 8, 1963) was an American businesswoman, suffragist, and clubwoman ... Lydia Flood was born in Oakland, California ... her mother opened schools for black children in Sacramento and Oakland ...
Doc 3(Title: "Constance Jackson") Constance L. Jackson is a business start-up professional in health technology and filmmaking ... She has spoken at national and international conferences ... She is the writer, director, and producer of feature-length documentary films ...
(*@\textcolor{taginfo}{\texttt{\textless/information\textgreater}}@*)

<answer> Andreas P. Nielsen </answer>
\end{lstlisting}
\caseRule
\caseHeader{casegreen}{\method{} (correct)}
\begin{lstlisting}[style=trajectory]
<think> To determine who was born first between Andreas P. Nielsen and Lydia Flood Jackson, I need to gather information about their birth dates. </think>
<search> Was Andreas P. Nielsen or Lydia Flood Jackson born first? </search>

(*@\textcolor{taginfo}{\texttt{\textless information\textgreater}}@*)
Doc 1(Title: "Lydia Flood Jackson") Lydia Flood Jackson ((*@\hit{June 6, 1862}@*) - July 8, 1963) was an American businesswoman, suffragist, and clubwoman ... Lydia Flood was born in Oakland, California ... her father Isaac Flood joined the California Colored Convention Movement to fight segregation in California schools ...
Doc 2(Title: "Fredrikke Nielsen") Fredrikke Louise Nielsen was born Fredrikke Louise Jensen (*@\hit{July 5, 1837}@*) in Haugesund ... She was a Norwegian actress and was married to actor Harald Nielsen ...
Doc 3(Title: "Andreas P. Nielsen") Andreas P. Nielsen (1943-2016) was a Danish politician ... He was (*@\hit{born on February 12, 1943}@*) in Copenhagen ... He represented a Danish constituency and served in national politics ...
(*@\textcolor{taginfo}{\texttt{\textless/information\textgreater}}@*)

<answer> Lydia Flood Jackson </answer>
\end{lstlisting}
\caseRule
\end{caseStudy}

\begin{caseStudy}
\caseCaption{tab:paired_case_species}{Case study 4 (multi-hop question): \method{} compares genus-level species counts for Erodium and Cymbidium, while VT-Search answers against the retrieved evidence.}
\caseRule
\textbf{Question:} Does Erodium or Cymbidium include more species?\\
\textbf{Ground Truth:} Erodium
\caseRule
\caseHeader{casered}{VT-Search (wrong)}
\begin{lstlisting}[style=trajectory]
<search> does Erodium or Cymbidium include more species </search>

(*@\textcolor{taginfo}{\texttt{\textless information\textgreater}}@*)
Doc 1(Title: Erodium) Erodium is a genus of flowering plants in the botanical family Geraniaceae ... The genus (*@\hit{includes about 60 species}@*), native to North Africa, Indomalaya, the Middle East, and Australia ... They are perennials, annuals, or subshrubs with five-petalled flowers in white, pink, and purple ...
Doc 2(Title: Erodium) Erodium includes about (*@\hit{60 species}@*) ... American species are known as filarees or heron's bill, whereas Eurasian ones are called storksbills in English ... some species are used as food plants by Lepidoptera larvae ...
Doc 3(Title: "Cymbidium suave") Cymbidium suave, or the snake orchid, is an Australian orchid species ... It is part of the genus Cymbidium ... The genus consists of (*@\hit{52 species}@*) and this species has flattened, yellowish green pseudobulbs ...
(*@\textcolor{taginfo}{\texttt{\textless/information\textgreater}}@*)

<answer> Cymbidium </answer>
\end{lstlisting}
\caseRule
\caseHeader{casegreen}{\method{} (correct)}
\begin{lstlisting}[style=trajectory]
<think> To determine which genus, Erodium or Cymbidium, includes more species, I need the number of species in each genus. </think>
<search> Does Erodium or Cymbidium include more species? </search>

(*@\textcolor{taginfo}{\texttt{\textless information\textgreater}}@*)
Doc 1(Title: Erodium) Erodium is a genus of flowering plants in Geraniaceae ... The genus (*@\hit{includes about 60 species}@*), native to North Africa, Indomalaya, the Middle East, and Australia ... flowers strongly resemble the better-known Geranium ...
Doc 2(Title: "Cymbidium eburneum") Cymbidium eburneum, the ivory-colored cymbidium, is a species of orchid ... It belongs to the genus Cymbidium and is found in the eastern Himalayas, Assam, and southern China ...
Doc 3(Title: Cymbidium) Cymbidium is a genus of orchids ... There are (*@\hit{approximately 50 species}@*) in the genus ... the species are distributed in tropical and subtropical Asia and northern Australia ...
(*@\textcolor{taginfo}{\texttt{\textless/information\textgreater}}@*)

<answer> Erodium </answer>
\end{lstlisting}
\caseRule
\end{caseStudy}

\begin{caseStudy}
\caseCaption{tab:paired_case_hangon}{Case study 5 (single-hop question): \method{} repairs the ambiguous Hang-On query by searching for the video game, while VT-Search follows music retrieval drift.}
\caseRule
\textbf{Question:} What genre is Hang-On?\\
\textbf{Ground Truth:} racing game
\caseRule
\caseHeader{casered}{VT-Search (wrong)}
\begin{lstlisting}[style=trajectory]
<search> what genre is Hang-On </search>

(*@\textcolor{taginfo}{\texttt{\textless information\textgreater}}@*)
Doc 1(Title: "Hang on the Box") Hang on the Box, also known as HOTB, is (*@\hit{a punk band based in Beijing}@*) ... They were China's first all-female punk band ... The band usually sings about sex and relationship issues in a forward political manner ...
Doc 2(Title: "The Hang Ups") The Hang Ups are (*@\hit{an indie pop rock band}@*) from Minneapolis, Minnesota, formed in 1990 and fronted by vocalist and guitarist Brian Tighe ... They released their debut EP "Comin' Through" in 1993 and later released four studio albums ...
Doc 3(Title: "Hang on the Box") Hang on the Box were featured on the local edition cover of Newsweek six months later ... Their music could be described as cheerful and hilarious, deadly serious, and personally political ... The page again identifies HOTB as (*@\hit{a punk band}@*) ...
(*@\textcolor{taginfo}{\texttt{\textless/information\textgreater}}@*)

<answer> punk </answer>
\end{lstlisting}
\caseRule
\caseHeader{casegreen}{\method{} (correct)}
\begin{lstlisting}[style=trajectory]
<think> To find out the genre of "Hang-On," I need to gather information about the game "Hang-On." </think>
<search> What genre is Hang-On? </search>

(*@\textcolor{taginfo}{\texttt{\textless information\textgreater}}@*)
Doc 1(Title: "Hang on the Box") Hang on the Box, also known as HOTB, is (*@\hit{a punk band based in Beijing}@*) ... They were China's first all-female punk band ... the lineup includes Wang Yue on vocals, Yilina on bass, Li Yan Fan on guitar, and Shenjing on drums ...
Doc 2(Title: "Hang On (Weezer song)") "Hang On" is (*@\hit{a song by the band Weezer}@*) ... It impacted radio on February 15, 2011 and is the seventh track and second single from the album "Hurley" ...
Doc 3(Title: "The Hang Ups") The Hang Ups are (*@\hit{an indie pop rock band}@*) from Minneapolis, Minnesota, formed in 1990 ... Members include Brian Tighe, Jeff Kearns, Stephen Ittner, and John Crozier ...
(*@\textcolor{taginfo}{\texttt{\textless/information\textgreater}}@*)

<search> What genre is Hang On (video game)? </search>

(*@\textcolor{taginfo}{\texttt{\textless information\textgreater}}@*)
Doc 1(Title: "Super Hang-On") Super Hang-On is a motorcycle racing video game by Sega ... it is the sequel to Hang-On and uses arcade-style racing across multiple tracks ...
Doc 2(Title: Hang-On) Hang-On uses a behind-the-motorcycle perspective ... (*@\hit{the player races a linear race track}@*) divided into stages within a limited time ... reaching checkpoints extends the time limit ...
Doc 3(Title: Hang-On) Hang-On was released as an arcade game with motorcycle-style controls ... The title is derived from the biker having to "hang on" to the bike while leaning ... the Master System version was also built into some consoles ...
(*@\textcolor{taginfo}{\texttt{\textless/information\textgreater}}@*)

<answer> racing game </answer>
\end{lstlisting}
\caseRule
\end{caseStudy}

\begin{caseStudy}
\caseCaption{tab:paired_case_chimera}{Case study 6 (single-hop question): \method{} refines the Chimera query to Greek mythology and retrieves Typhon, while VT-Search follows irrelevant entities.}
\caseRule
\textbf{Question:} Who is the father of Chimera?\\
\textbf{Ground Truth:} Typhon
\caseRule
\caseHeader{casered}{VT-Search (wrong)}
\begin{lstlisting}[style=trajectory]
<search> who is the father of Chimera </search>

(*@\textcolor{taginfo}{\texttt{\textless information\textgreater}}@*)
Doc 1(Title: "King Chimera") King Chimera is (*@\hit{a fictional superhero in the DC Comics Universe}@*), a member of the Justice Society of America ... He first appeared in "Justice Society of America" (vol. 3) #24 in April 2009 ...
Doc 2(Title: "King Chimera") King Chimera keeps a mysterious identity and very little is known about his past ... the page says he claims to be (*@\hit{the son of "King" Standish}@*) ...
Doc 3(Title: "Chimera (short story)") Chimera is a comical short story by Korean author Lee Yeongdo based on Dragon Raja ... The king of Caria, (*@\hit{Amisodarus, raised the Chimera}@*) to terrorize the surrounding region ...
(*@\textcolor{taginfo}{\texttt{\textless/information\textgreater}}@*)

<answer> Amisodarus </answer>
\end{lstlisting}
\caseRule
\caseHeader{casegreen}{\method{} (correct)}
\begin{lstlisting}[style=trajectory]
<think> To find out who the father of Chimera is, I need to gather information about Chimera and her parentage. </think>
<search> Who is the father of Chimera? </search>

(*@\textcolor{taginfo}{\texttt{\textless information\textgreater}}@*)
Doc 1(Title: "Chimera (short story)") Chimera is a comical short story by Lee Yeongdo based on Dragon Raja ... The king of Caria, (*@\hit{Amisodarus, raised the Chimera}@*) to terrorize the surrounding region ...
Doc 2(Title: "King Chimera") King Chimera is (*@\hit{a fictional superhero in the DC Comics Universe}@*) ... The character first appeared in Justice Society of America and claims to be the son of "King" Standish ...
Doc 3(Title: "King Chimera") King Chimera is a fictional superhero in the DC Comics Universe ... He is a member of the Justice Society of America and first appeared in "Justice Society of America" (vol. 3) #24 ...
(*@\textcolor{taginfo}{\texttt{\textless/information\textgreater}}@*)

<search> Who is the father of Chimera in Greek mythology? </search>

(*@\textcolor{taginfo}{\texttt{\textless information\textgreater}}@*)
Doc 1(Title: "Chimera (mythology)") Chimera is a monstrous hybrid creature in Greek mythology ... commonly described as fire-breathing and as one of (*@\hit{the offspring of Typhon and Echidna}@*) ...
Doc 2(Title: "Chimera (mythology)") The mythology page discusses variant genealogies ... most traditions attribute the Chimera and related monsters to (*@\hit{Typhon and Echidna}@*) ...
Doc 3(Title: "Chimera of Arezzo") The Chimera of Arezzo depicts the mythological Chimera in bronze ... the object is identified with the creature described as (*@\hit{offspring of Typhon and Echidna}@*) ...
(*@\textcolor{taginfo}{\texttt{\textless/information\textgreater}}@*)

<answer> Typhon </answer>
\end{lstlisting}
\caseRule
\end{caseStudy}

\subsection{Retrieval Behavior across Training Checkpoints}
\label{app:training_cases}

Tables~\ref{tab:training_case_prada_early} and~\ref{tab:training_case_prada_late}
illustrate how the retrieval
behavior changes during training for the same question. Early \method{} checkpoints identify the
episode clue but stop after one broad search, leaving the advertisement clue
unretrieved. After training, \method{} issues a second, actor-specific
search and retrieves the missing evidence, while VT-Search
remains single-search and wrong.

\begin{caseStudy}
\caseCaption{tab:training_case_prada_early}{Training-procedure case study, early checkpoint: both methods retrieve the Cherokee Rose episode clue but not the advertisement clue, then answer from distractors.}
\caseRule
\textbf{Question:} In the Cherokee Rose episode of ``The Walking Dead,'' the character that continues to search for Sophia Peletier is portrayed by an actor who is also famous for his work in what company's advertisements?\\
\textbf{Ground Truth:} Prada
\caseRule
\caseHeader{casered}{\method{} at step 50 (wrong)}
\begin{lstlisting}[style=trajectory]
<search> In the Cherokee Rose episode of "The Walking Dead," the character that continues to search for Sophia Peletier is portrayed by an actor who is also famous for his work in what company's advertisements? </search>

(*@\textcolor{taginfo}{\texttt{\textless information\textgreater}}@*)
Doc 1(Title: "Cherokee Rose (The Walking Dead)") The episode page states that Daryl Dixon (*@\hit{Norman Reedus}@*) continues to search for Sophia Peletier ...
Doc 2(Title: "Sophia Peletier") The page describes Sophia Peletier and reactions to her disappearance, but does not identify Reedus's advertising work ...
Doc 3(Title: "Cherokee Rose (The Walking Dead)") A review compares Glenn to Little Mikey from the (*@\hit{Life advertisements}@*), introducing an advertising distractor unrelated to Reedus ...
(*@\textcolor{taginfo}{\texttt{\textless/information\textgreater}}@*)

<answer> Life </answer>
\end{lstlisting}
\caseRule
\caseHeader{casered}{VT-Search at step 50 (wrong)}
\begin{lstlisting}[style=trajectory]
<search> actor who portrays character in "The Walking Dead" Cherokee Rose episode and is also famous for his work in company's advertisements </search>

(*@\textcolor{taginfo}{\texttt{\textless information\textgreater}}@*)
Doc 1(Title: "Cherokee Rose (The Walking Dead)") The episode page again mentions Daryl Dixon (*@\hit{Norman Reedus}@*) and the unrelated Life-advertisements review ...
Doc 2(Title: "Daryl Wilcher") This different actor appeared in commercials for (*@\hit{KFC}@*), Coca-Cola, and Combos Pretzels ...
Doc 3(Title: "Cherokee Rose (The Walking Dead)") The episode evidence repeats the plot clue but does not retrieve Norman Reedus's advertising history ...
(*@\textcolor{taginfo}{\texttt{\textless/information\textgreater}}@*)

<answer> KFC </answer>
\end{lstlisting}
\caseRule
\end{caseStudy}

\begin{caseStudy}
\caseCaption{tab:training_case_prada_late}{Training-procedure case study, best checkpoint: \method{} repairs the missing evidence with an actor-specific second search, while VT-Search still stops after one broad query.}
\caseRule
\textbf{Question:} Same validation item as Table~\ref{tab:training_case_prada_early}.\\
\textbf{Ground Truth:} Prada
\caseRule
\caseHeader{casegreen}{\method{} at final checkpoint (step 250, correct)}
\begin{lstlisting}[style=trajectory]
<search> In the Cherokee Rose episode of "The Walking Dead," the character that continues to search for Sophia Peletier is portrayed by an actor who is also famous for his work in what company's advertisements? </search>

(*@\textcolor{taginfo}{\texttt{\textless information\textgreater}}@*)
Doc 1(Title: "Cherokee Rose (The Walking Dead)") Daryl Dixon (*@\hit{Norman Reedus}@*) continues to search for Sophia Peletier ...
Doc 2(Title: "Sophia Peletier") The page supplies character context but not the advertising clue ...
Doc 3(Title: "Carol Peletier") The page mentions Daryl and Sophia in related episode context but still lacks the company name ...
(*@\textcolor{taginfo}{\texttt{\textless/information\textgreater}}@*)

<search> Norman Reedus famous for his work in what company's advertisements </search>

(*@\textcolor{taginfo}{\texttt{\textless information\textgreater}}@*)
Doc 1(Title: "Norman Reedus") Reedus has modeled for (*@\hit{Prada}@*), Alessandro Dell'Acqua, Durban, Levi's, Lexus, and Morgenthal Fredrics ...
Doc 2(Title: "Norman B. Norman") Irrelevant Norman entity ...
Doc 3(Title: "Geoff Moore (photographer)") Related photography/fashion context, but the first document contains the answer ...
(*@\textcolor{taginfo}{\texttt{\textless/information\textgreater}}@*)

<answer> Prada </answer>
\end{lstlisting}
\caseRule
\caseHeader{casered}{VT-Search at final checkpoint (step 250, wrong)}
\begin{lstlisting}[style=trajectory]
<search> which actor portrays the character that continues to search for Sophia Peletier in the Cherokee Rose episode of "The Walking Dead" and is also famous for his work in what company's advertisements </search>

(*@\textcolor{taginfo}{\texttt{\textless information\textgreater}}@*)
Doc 1(Title: "Cherokee Rose (The Walking Dead)") Daryl Dixon (*@\hit{Norman Reedus}@*) continues to search for Sophia Peletier ...
Doc 2(Title: "Sophia Peletier") Character context for Sophia Peletier, without Reedus's modeling or advertisements ...
Doc 3(Title: "Cherokee Rose (The Walking Dead)") The retrieved passage again contains the unrelated Life-advertisements review rather than the Prada clue ...
(*@\textcolor{taginfo}{\texttt{\textless/information\textgreater}}@*)

<answer> Bud Light </answer>
\end{lstlisting}
\caseRule
\end{caseStudy}

The later \method{} policy uses the actor identified by the first retrieval
to search for the missing advertising credit. This targeted second hop
illustrates how learned search behavior can close a specific evidence gap.

\clearpage
\begingroup
\raggedbottom
\subsection{Medical QA Case Studies}
\label{app:medical_cases}

{We present three medical case studies comparing \method{} with VT-Search, the second-best method.}

\begin{caseStudy}
\caseCaption{tab:medical_case_brown_tumors}{Case study 7: \method{} retrieves evidence linking brown tumors to hyperparathyroidism; VT-Search retrieves other causes of finger lesions and retains multiple myeloma.}
\caseRule
\textbf{Question (simplified):} A 50-year-old woman has left-hand finger pain for one month, worsening confusion over the past week, and nausea for several months. Examination reveals tender fingers with small palpable tumors. A hand X-ray shows thin bones, brown tumors, and osteolysis of the distal phalanges. What is the most likely cause?\\
\textbf{Options:} A: Hyperparathyroidism; B: Hypocalcemia; C: Osteosarcoma; D: Multiple myeloma.\\
\textbf{Ground Truth:} A: Hyperparathyroidism.
\caseRule
\caseHeader{casered}{VT-Search (wrong)}
\begin{lstlisting}[style=trajectory]
<think> ...
A: Hyperparathyroidism - This condition can cause bone loss and hypercalcemia, but it is not typically associated with finger pain and palpable tumors.
...
<search> What is the most likely cause for a 50-year-old woman with finger pain, confusion, palpable tumors, and osteolysis of the distal phalanges? </search>

(*@\textcolor{taginfo}{\texttt{\textless information\textgreater}}@*)
Doc 1(Title: "An unusual presentation of a ring injury.") A case of chronic erosion of the soft tissues and bone ...
Doc 2(Title: "Metastatic phalangeal osteolysis ...") Metastatic malignancies of the hand are rare ...
Doc 3(Title: "Fibro-osseous pseudotumour ...") ... The X-ray showed evidence of erosion at the tip of the distal phalanx. ...
(*@\textcolor{taginfo}{\texttt{\textless/information\textgreater}}@*)

... the most likely cause for this patient's condition is multiple myeloma.
<answer> Multiple myeloma </answer>
\end{lstlisting}
\caseRule
\caseHeader{casegreen}{\method{} (correct)}
\begin{lstlisting}[style=trajectory]
<think> ... The patient has finger pain, small palpable tumors on the fingers, and an X-ray showing thin bones with brown tumors and osteolysis of the distal phalanges. ... </think>
<search> symptoms of finger pain, small palpable tumors on the fingers, thin bones with brown tumors and osteolysis of the distal phalanges</search>

(*@\textcolor{taginfo}{\texttt{\textless information\textgreater}}@*)
Doc 1(Title: "Osteochondroma") Signs and symptoms
Doc 2(Title: "Brown tumour in phalanx of the ring finger ...") Brown tumours are highly vascular lytic bone lesions found in (*@\hit{primary and secondary hyperparathyroidism}@*). ...
Doc 3(Title: "Benign tumor") Signs and symptoms
(*@\textcolor{taginfo}{\texttt{\textless/information\textgreater}}@*)

...
<think> ... The presence of brown tumors and osteolysis of the distal phalanges is characteristic of this condition. </think>
<answer> Hyperparathyroidism </answer>
\end{lstlisting}
\caseRule
\end{caseStudy}

\begin{caseStudy}
\caseCaption{tab:medical_case_dystonia}{Case study 8: \method{} searches for haloperidol-induced dystonia and selects benztropine; VT-Search uses a broader query and selects baclofen.}
\caseRule
\textbf{Question (simplified):} A 31-year-old man develops episodic neck stiffness and pain after six doses of haloperidol over two days for substance-induced psychosis. Episodes last about 25 minutes, with his neck rotating right; touching his jaw relieves the stiffness. Temperature is 37.3$^\circ$C, pulse 108/min, and blood pressure 128/86 mm Hg. His neck is flexed, rigid, and rotated right with limited movement. What is the most appropriate therapy?\\
\textbf{Options:} A: Benztropine; B: Physical therapy; C: Dantrolene; D: Baclofen.\\
\textbf{Ground Truth:} A: Benztropine.
\caseRule
\caseHeader{casered}{VT-Search (wrong)}
\begin{lstlisting}[style=trajectory]
<think> ... Given these symptoms, it seems that the patient is experiencing neck muscle spasms or stiffness, which could be related to his substance-induced psychosis or the medications he is taking, such as haloperidol. ...
<search> What is the most appropriate therapy for neck muscle spasms or stiffness in patients with substance-induced psychosis or those taking haloperidol? </search>

(*@\textcolor{taginfo}{\texttt{\textless information\textgreater}}@*)
Doc 1(Title: "[On a muscular dystony case from manganism. ...") It is described a muscular dystony case of manganism, treated with good results in two following phases with haloperidol and diazepam and then with L-Dopa + benserazide. ...
Doc 2(Title: "Is haloperidol or a benzodiazepine ...") ... We conclude that haloperidol should be considered the first line drug for agitated patients post cardiac surgery, ...
Doc 3(Title: "InternalMed_Harrison") Optimal treatment of (*@\hit{stiff-person syndrome}@*) requires therapy of the underlying tumor, glucocorticoids, and symptomatic use of drugs that enhance GABA-ergic transmission (diazepam, baclofen, sodium valproate, tiagabine, vigabatrin). ...
(*@\textcolor{taginfo}{\texttt{\textless/information\textgreater}}@*)

... Between the options, dantrolene and baclofen are both muscle relaxants, but baclofen is more commonly used for muscle spasms and stiffness.
...
<answer> Baclofen </answer>
\end{lstlisting}
\caseRule
\caseHeader{casegreen}{\method{} (correct)}
\begin{lstlisting}[style=trajectory]
<think> ... Dystonia can be a side effect of antipsychotic medications like haloperidol. ...
<search> treatment for neck dystonia caused by haloperidol </search>

(*@\textcolor{taginfo}{\texttt{\textless information\textgreater}}@*)
Doc 1(Title: "A patient with haloperidol induced laryngeal dystonia.") ... It can occur after a short treatment of a patient with atypical antipschycotics. ...
Doc 2(Title: "Haloperidol-Induced Dystonia ...") ... Antipsychotic agents sometimes cause (*@\hit{acute dystonia}@*) that can rapidly worsen within a few hours or days. ...
Doc 3(Title: "Treatment of dystonia and tics.") ... In dystonia, while oral agents such as benzodiazepines, baclofen and anticholinergics remain in use, botulinum toxin (BoNT) continues to be regarded as the treatment of choice for focal and segmental dystonia, ...
(*@\textcolor{taginfo}{\texttt{\textless/information\textgreater}}@*)

... the most appropriate therapy for this patient's symptoms is Benztropine, which is an anticholinergic medication used to treat dystonia caused by antipsychotic medications like haloperidol.
<answer> Benztropine </answer>
\end{lstlisting}
\caseRule
\end{caseStudy}

\begin{caseStudy}
\caseCaption{tab:medical_case_canine_rest}{Case study 9: both methods retrieve the same passages in different orders; \method{} selects reduced leverage, while VT-Search acknowledges the rotation-center evidence but chooses the cingulum distractor.}
\caseRule
\textbf{Question:} A properly designed rest on lingual surface of a canine is preferred to a rest on incisal surface because:\\
\textbf{Options:} A: Less Leverage is exerted against the tooth by Lingual rest; B: Enamel is thicker on lingual surface; C: Visibility and access is better on lingual surface; D: The cingulum of canine produces a naturalsurface for recess.\\
\textbf{Ground Truth:} A: Less Leverage is exerted against the tooth by Lingual rest.
\caseRule
\caseHeader{casered}{VT-Search (wrong)}
\begin{lstlisting}[style=trajectory]
<think> ... I will search for information to help me determine the correct answer.
<search> A properly designed rest on lingual surface of a canine is preferred to a rest on incisal surface </search>

(*@\textcolor{taginfo}{\texttt{\textless information\textgreater}}@*)
Doc 1(Title: "The marginal-ridge rest seat.") ... preparing the rest seat on the lingual surface of a mandibular canine risks perforating the enamel. ...
Doc 2(Title: "Evaluation of prefabricated lingual rest seats for removable partial dentures.") ... A prefabricated lingual rest seat was designed and constructed to closely approximate the lingual contours of anterior teeth. ...
Doc 3(Title: "[Evaluation of stresses in different cingulum rest seats].") ... Lingual rest seats are preferred since they are more acceptable esthetically and (*@\hit{located closer to the centre of rotation of the tooth}@*). ...
(*@\textcolor{taginfo}{\texttt{\textless/information\textgreater}}@*)

... The lingual surface is preferred because it is closer to the center of rotation of the tooth and provides a more uniform stress distribution.
...
<answer> D </answer>
\end{lstlisting}
\caseRule
\caseHeader{casegreen}{\method{} (correct)}
\begin{lstlisting}[style=trajectory]
<think> ... I'll consider the biomechanical, anatomical, and aesthetic factors involved in dental restorations. </think>
<search> Why is a properly designed rest on the lingual surface of a canine preferred to a rest on the incisal surface? </search>

(*@\textcolor{taginfo}{\texttt{\textless information\textgreater}}@*)
Doc 1(Title: "The marginal-ridge rest seat.") ... preparing the rest seat on the lingual surface of a mandibular canine risks perforating the enamel. ...
Doc 2(Title: "[Evaluation of stresses in different cingulum rest seats].") ... Lingual rest seats are preferred since they are more acceptable esthetically and (*@\hit{located closer to the centre of rotation of the tooth}@*). ...
Doc 3(Title: "Evaluation of prefabricated lingual rest seats for removable partial dentures.") ... A prefabricated lingual rest seat was designed and constructed to closely approximate the lingual contours of anterior teeth. ...
(*@\textcolor{taginfo}{\texttt{\textless/information\textgreater}}@*)

... This reduces the leverage exerted against the tooth by the rest.

<think> Based on the information gathered, the correct answer is:
A: Less Leverage is exerted against the tooth by Lingual rest
</think>
<answer> A </answer>
\end{lstlisting}
\caseRule
\end{caseStudy}

\clearpage
\endgroup

\end{document}

%% file: math_commands.tex
\usepackage{amsmath,amsfonts,bm}

\def\1{\bm{1}}

\DeclareMathAlphabet{\mathsfit}{\encodingdefault}{\sfdefault}{m}{sl}
\SetMathAlphabet{\mathsfit}{bold}{\encodingdefault}{\sfdefault}{bx}{n}

\newcommand{\E}{\mathbb{E}}

\newcommand{\KL}{D_{\mathrm{KL}}}

\DeclareMathOperator*{\argmin}{arg\,min}

%% file: appendix_theory.tex
\section{Retrieval Gradients and Theoretical Guarantee}
\label{app:theory}

In this section, we derive the retrieval RL gradient and state the convergence guarantee for BRIDGE.

\subsection{Derivation of retrieval gradient}
\label{app:retrieval_gradient}

\noindent\textbf{Gradient of RAG objective with respect to the retrieval parameters.}
At the retrieval turn $t$, with rollout $\tau_i$ for prompt $x$, the query context
is $h_{i,t}=(\tau_{i,<t},a_{i,t})$ and we can access the reference answer $y$ corresponding to question $x$.
Within a batch, RAG objective in \eqref{eq:lrag} can be written as
\[
\mathcal L_{\mathrm{RAG}}(\theta,\eta)
=-\mathbb E_{x,i,t}\!\left[
\log\sum_{d\in\mathcal N_k(h_{i,t})}
\rho_\eta(d\mid h_{i,t})\,
p_\theta(y\mid h_{i,t},d)
\right].
\]
Here, $\mathbb E_{x,i,t}$ averages over training prompts in a batch, sampled rollouts,
and their logged retrieval turns. We retain $(h,y)$ as shorthand below, with
$\mathbb E_{(h,y)}$ denoting this same averaging. 

Following \citep{lewis2021rag}, we treat $\mathcal N_k(h_{i,t})$ fixed and  differentiate the RAG objective in \eqref{eq:lrag} as   
\begin{align}
&\nabla_\eta \mathcal{L}_{\mathrm{RAG}}(\theta,\eta)
=-\mathbb{E}_{(h,y)}
\left[
\sum_{d \in \mathcal{N}_k(h)}
\frac{
    \rho_\eta(d \mid h)\,p_\theta(y \mid h,d)
}{
    \sum_{d' \in \mathcal{N}_k(h)}
    \rho_\eta(d' \mid h)\,p_\theta(y \mid h,d')
}
\nabla_\eta \log \rho_\eta(d \mid h)
\right]
\nonumber\\
&=
\mathbb{E}_{(h,y)}
\Bigg[
\nabla_\eta f_\eta (h)^{\!\top}
\sum_{d \in \mathcal{N}_k(h)}
\underbrace{\rho_\eta(d \mid h)
\left(
1-
\frac{
    p_\theta(y \mid h,d)
}{
    \sum_{d' \in \mathcal{N}_k(h)}
    \rho_\eta(d' \mid h)\,p_\theta(y \mid h,d')
}
\right)}_{\text{RAG score-gradient coefficient}}
g(d)
\Bigg].
\label{eq:app_retrieval_score_gradient}
\end{align}
Intuitively, the RAG score-gradient coefficient compares a document’s current retrieval probability \(\rho_\eta(d\mid h)\) with the reweighted document score by the final answer correctness. Gradient descent with a negative RAG score-gradient coefficient adds a positive multiple of \(g(d)\) to the query-embedding descent direction $\nabla f_\eta(h)$, promoting stronger alignment between the query and the document and thereby providing a signal to improve the document’s ranking within the top-\(k\) retrieved set.

\noindent\textbf{Likelihood computation.}
During training, the reference answer $y$ is available. For each retrieved
document $d\in\mathcal N_k(h)$, we compute
\[
\log p_\theta(y\mid h,d)
=
\sum_{j=1}^{|y|}
\log p_\theta(y_j\mid h,d,y_{<j})
\]
using teacher forcing through forward computation, summing only over answer tokens. These log-likelihoods can be reused
across inner steps while $\theta$, $h$, $d$, and $y$ remain unchanged,
whereas $\rho_\eta(d\mid h)$ is recomputed after each retriever update.
For fixed sequence lengths, LLM scoring costs scale linearly with $k$, but is independent with respect to the corpus size $|\mathcal C|$. We also use document batch processing to reduce the peak scoring memory required for this step. 

\noindent\textbf{Numerical stability.}
We evaluate the RAG objective in log space:
\[
\mathcal L_{\mathrm{RAG}}(\theta,\eta)
=
-\mathbb E_{(h,y)}
\left[
\operatorname{logsumexp}_{d\in\mathcal N_k(h)}
\left(
\log\rho_\eta(d\mid h)
+
\log p_\theta(y\mid h,d)
\right)
\right].
\]
Here, $\log\rho_\eta$ is computed by log-softmax over $\mathcal N_k(h)$.
Log-softmax and log-sum-exp avoid explicit products of small probabilities
or division by an underflowed marginal likelihood, which ensures its numerical stability. The answer-conditioned
document weights in \eqref{eq:app_retrieval_score_gradient} are obtained
by applying softmax to $\log\rho_\eta(d\mid h)+\log p_\theta(y\mid h,d)$.
Each score-gradient coefficient is therefore a difference between two
normalized probabilities. 

\noindent\textbf{RL outcome credit for retrieval.}
To motivate the GRPO like RL surrogate gradient with respect to retrieval parameter $\eta$, fix $\theta$ and consider
a stochastic retrieval relaxation in which
$d_t\sim\rho_\eta(\cdot\mid h_t)$ over fixed top-$k$ candidate set $\mathcal{N}_k(h_t)$.
For the outcome reward, the log-derivative identity gives
\begin{align}
-\nabla_\eta \mathbb E_{x,\tau}[R(x,\tau)]
&=
-\mathbb E_x \int
\nabla_\eta P(\tau\mid\theta,\eta) R(x,\tau)\,d\tau
\nonumber\\
&=
-\mathbb E_x \int
P(\tau\mid\theta,\eta)
\nabla_\eta\log P(\tau\mid\theta,\eta)
R(x,\tau)\,d\tau
\nonumber\\
&=
-\mathbb E_{x,\tau}
\left[
R(x,\tau)\nabla_\eta\log P(\tau\mid\theta,\eta)
\right]
\nonumber\\
&=
-\mathbb E_{x,\tau}
\left[
R(x,\tau)\sum_t
\nabla_\eta\log\rho_\eta(d_t\mid h_t)
\right],
\label{eq:app_relaxed_retrieval_policy_gradient}
\end{align}
where the expectation is under this stochastic relaxation
and the sum is over retrieval turns $t$. The last equality follows because, under the stochastic retrieval
relaxation, $\eta$ enters the trajectory probability only through
the factors $\rho_\eta(d_t\mid h_t)$ at retrieval turns $t$.
With $\theta$ fixed and the recorded trajectory held fixed,
differentiation therefore gives
$\nabla_\eta\log P(\tau\mid\theta,\eta)
=\sum_t\nabla_\eta\log\rho_\eta(d_t\mid h_t)$. 

BRIDGE uses deterministic top-$k$ retrieval during rollouts.
Motivated by the expression above, we construct a
retrieval-credit surrogate on the logged candidate sets,
using $\bar c_{i,t,d}$ to distribute trajectory-level
outcome credit among candidate passages.
We then replace $R_i$ with the GRPO advantage $A_i$
and apply probability-ratio clipping.

Each retrieval turn $t$ of rollout $\tau_i$ has the trajectory-level
GRPO advantage $A_i$ in \eqref{eq:grpo}. We distribute this outcome
credit over the retrieved documents using answer-likelihood-based credit weight assigned to document and the document-level importance ratio as 
\begin{align}
 \bar c_{i,t,d}
 &=
 \frac{\rho_{\etaold}(d\mid h_{i,t})\,
                   p_\theta(y\mid h_{i,t},d)}
      {\sum_{d'\in\mathcal N_k(h_{i,t})}
       \rho_{\etaold}(d'\mid h_{i,t})\,
                   p_\theta(y\mid h_{i,t},d')}, ~~r_{i,t,d}(\eta)=\frac{\rho_\eta(d\mid h_{i,t})}
          {\rho_{\etaold}(d\mid h_{i,t})}
 \label{eq:app_retrieval_credit}
\end{align}
where $\etaold$ denotes the pre-update retrieval parameter. 
The document-level importance ratio extends the token-level ratio in \eqref{eq:grpo}, with document
probabilities replacing token probabilities. Fixing $\theta$ and following ~\citep{schulman2017proximal}, the retrieval
RL surrogate used the following clipped advantage 
\begin{align}
 &{\mathcal J}^{\eta}_{\mathrm{RL}}(\theta,\eta)
 =-\mathbb E_{x,i,t}\!\left[
   \sum_{d\in\mathcal N_k(h_{i,t})}\bar c_{i,t,d}
   \min\!\left\{r_{i,t,d}(\eta)A_i,
       \operatorname{clip}\!\left(r_{i,t,d}(\eta),
              1-\epsilon_\eta,1+\epsilon_\eta\right)A_i\right\}
   \right],
 \label{eq:app_retrieval_rl_loss}
\end{align}
where $\mathbb E_{x,i,t}$ denotes the empirical average over prompts,
rollouts, and retrieval turns, and $\epsilon_\eta$ is the retrieval
clipping threshold. Away from clipping
boundaries, differentiation gives
\begin{align}
 &\nabla_\eta{\mathcal J}^{\theta,\eta}_{\mathrm{RL}}(\eta)=-\mathbb E_{x,i,t}\!\left[
   \sum_{d\in\mathcal N_k(h_{i,t})}
   \bar c_{i,t,d}\,\kappa_{i,t,d} A_i r_{i,t,d}(\eta)\,
   \nabla_\eta\log\rho_\eta(d\mid h_{i,t})
   \right],
 \label{eq:app_retrieval_rl_gradient}
\end{align}
where $\kappa_{i,t,d}=1$ if $A_i\geq0$ and
$r_{i,t,d}(\eta)<1+\epsilon_\eta$, or if $A_i<0$ and
$r_{i,t,d}(\eta)>1-\epsilon_\eta$; otherwise it is zero.
Thus, $A_i$ determines the sign of the outcome credit, while
$\bar c_{i,t,d}$ allocates it according to each document's support for
the correct answer. 

\noindent\textbf{Combining RL and RAG gradients.}
The clipped RL surrogate in \eqref{eq:app_retrieval_rl_gradient} uses the fixed weights $\bar c_{i,t,d}$ to
distribute trajectory-level outcome credit. In contrast, the RAG gradient
in \eqref{eq:app_retrieval_score_gradient} differentiates the current
log-mixture likelihood and uses weights based on $\rho_\eta$. The retriever update combines the RL gradient in 
\eqref{eq:app_retrieval_rl_gradient} with $\gamma$ times this RAG gradient.

\noindent\textbf{LoRA gradient form.}
For each adapted query-encoder weight matrix, write $\eta=\eta_0+AB$
with $\eta_0$ frozen. Holding $\theta$ fixed, the chain rule gives
\begin{align}
 \nabla_A\mathcal L_\gamma(\theta,\eta)
 &=\big[\nabla_\eta\Jrl(\theta,\eta)+\gamma\nabla_\eta\Lrag(\theta,\eta)\big]B^\top,
 \nonumber\\
 \nabla_B\mathcal L_\gamma(\theta,\eta)
 &=A^\top\big[\nabla_\eta\Jrl(\theta,\eta)+\gamma\nabla_\eta\Lrag(\theta,\eta)\big].
 \label{eq:app_lora_gradients}
\end{align}
In implementation, the RL gradient is approximated by
\eqref{eq:app_retrieval_rl_gradient}, and the RAG gradient is given by
\eqref{eq:app_retrieval_score_gradient}. Applying these expressions
layerwise yields the factor updates in \eqref{eq:eta_loss}, while the
base query-encoder weights and document encoder remain frozen.

\subsection{Convergence guarantee of BRIDGE}
\label{app:bridge_convergence}

Following \citep{kwon2023fully,shen2023penalty}, define
\begin{align}
 \Lrag^*(\theta)=\min_\eta\Lrag(\theta,\eta),\qquad
 \mathcal L_\gamma^*(\theta)=\min_\eta\mathcal L_\gamma(\theta,\eta)-\gamma\Lrag^*(\theta).
 \label{eq:app_reduced_penalty}
\end{align}
We analyze direct gradient updates of $\eta$ with exact population gradients.
The following assumptions follow the smoothness, lower-level landscape, and flatness conditions of \citet{jiang2026beyond}.

\begin{assumption}[Lipschitz continuity]
\label{as:bridge_smoothness}
$\Jrl$ and $\Lrag$ are jointly $L_{\mathrm{RL}}$- and $L_{\mathrm{RAG}}$-smooth, and have Lipschitz-continuous Hessians. The minima in \eqref{eq:app_reduced_penalty} are attained, and $\inf_\theta\mathcal L_\gamma^*(\theta)>-\infty$. 
\end{assumption}

\begin{assumption}[Lower-level landscape]
\label{as:bridge_pl}
There exist $\mu,\gamma_0>0$ s.t. $\Lrag(\theta,\cdot)$ and $\gamma^{-1}\mathcal L_\gamma(\theta,\cdot)$ satisfy the $\mu$-Polyak--\L{}ojasiewicz (PL) condition uniformly over $\theta$ and $\gamma\ge\gamma_0$, i.e., for all $(\theta,\eta)$,
\begin{align}
 \tfrac12\|\nabla_\eta\Lrag(\theta,\eta)\|^2
 &\ge\mu\big[\Lrag(\theta,\eta)-\min_\eta\Lrag(\theta,\eta)\big],\nonumber\\
 \tfrac12\|\nabla_\eta\mathcal L_\gamma(\theta,\eta)\|^2
 &\ge\gamma\mu\big[\mathcal L_\gamma(\theta,\eta)-\min_\eta\mathcal L_\gamma(\theta,\eta)\big].
 \label{eq:app_pl}
\end{align}
\end{assumption}

\begin{assumption}[Upper-level flatness]
\label{as:bridge_flatness}
There exist $p\in(1,3/2]$ and $C\ge0$, independent of $\gamma$ and $R$, such that for every $\theta_r$ there is $\delta_r\ge0$ satisfying
\begin{align}
 |\Jrl(\theta_r,\eta)-\Jrl(\theta_r,\eta^*)|
 \le C\|\eta-\eta^*\|^p+\delta_r,\qquad
 \frac1R\sum_{r=0}^{R-1}\delta_r\le\bar\delta
 \label{eq:app_flatness}
\end{align}
for all $\eta$ and all $\eta^*\in\arg\min_\eta\Lrag(\theta_r,\eta)$.
\end{assumption}

\begin{corollary}[Approximate stationary point]
\label{cor:bridge_bilevel_stationarity}
Under Assumptions~\ref{as:bridge_smoothness}--\ref{as:bridge_pl}, suppose the flatness condition holds at $\theta$ with $\delta=O(\epsilon^{p/2})$, where $0<\epsilon\le1$. There exists $\gamma_\epsilon=O(\epsilon^{-(2-p)/2})$ such that, for $\gamma\ge\max\{\gamma_0,\gamma_\epsilon\}$, any pair $(\theta,\eta_\gamma)$ satisfying $\|\nabla\mathcal L_\gamma^*(\theta)\|\le\epsilon$ and $\|\nabla_\eta\mathcal L_\gamma(\theta,\eta_\gamma)\|\le\epsilon$ is $O(\epsilon)$-stationary point to the original bilevel problem \eqref{eq:ideal_bilevel} in the squared KKT-residual sense. 
\end{corollary}

This extends \citep[Lemma 3]{jiang2026beyond} to an approximate penalized minimizer setting, which directly links the approximate stationary point for the penalized problem to the approximate stationary point for the original problem \eqref{eq:ideal_bilevel}. We defer its proof after Theorem \ref{thm:bridge_convergence}. 

The following theorem extends \citep[Theorem 3]{jiang2026beyond} to periodic retriever updates under exact population gradients and direct gradient descent updates of the retriever parameters. This provides an idealized analysis of the central alternating update structure in BRIDGE. Together with the approximate stationary point result in Corollary \ref{cor:bridge_bilevel_stationarity}, it explains how BRIDGE can achieve approximate stationarity for the original bilevel problem, without explicitly computing hypergradients. Extending the analysis to the LoRA parameterization and empirical surrogate gradients used in our implementation is beyond the scope of this work. 

\begin{theorem}[Convergence of BRIDGE with periodic retriever updates]
\label{thm:bridge_convergence}
Under Assumptions~\ref{as:bridge_smoothness}--\ref{as:bridge_flatness}, fix $\gamma\ge\gamma_0$ and $T\ge1$. Periodicly updating the retriever $\eta$ at rounds $r\in\{0,T,2T,\ldots\}$ and the policy $\theta$ at every round gives
\begin{align}
 \eta_{r+1}&=\eta_r-\frac b\gamma\mathbf 1_{\{r\bmod T=0\}}\nabla_\eta\mathcal L_\gamma(\theta_r,\eta_r),
 \qquad \theta_{r+1}=\theta_r-a\nabla_\theta\Jrl(\theta_r,\eta_r).
 \label{eq:app_bridge_periodic_updates}
\end{align}
If we choose 
\begin{align}
b=\left(L_{\mathrm{RAG}}+\frac{L_{\mathrm{RL}}}{\gamma}\right)^{-1},
 \qquad a=\frac{b\mu^2}{64\gamma T(L_{\mathrm{RAG}}+\mu)^2}={\cal O}(T^{-1}) 
 \label{eq:app_bridge_stepsizes}
\end{align}
then for every $R\ge1$, we have 
\begin{align}
 &\frac1R\sum_{r=0}^{R-1}\left[
 \|\nabla\mathcal L_\gamma^*(\theta_r)\|^2+
 \|\nabla_\eta\mathcal L_\gamma(\theta_r,\eta_r)\|^2\right]\nonumber\\
 &\qquad\le O\!\left(
 \frac{\mathcal L_\gamma(\theta_0,\eta_0)-\gamma\Lrag^*(\theta_0)-\inf_\theta\mathcal L_\gamma^*(\theta)}{aR}
 +\gamma^{-\frac{2(p-1)}{2-p}}+\gamma\bar\delta\right).
 \label{eq:app_bridge_convergence}
\end{align}
Moreover, the lower-level objective gap satisfies
\begin{align}
 \frac1R\sum_{r=0}^{R-1}\big[\mathcal L_\gamma(\theta_r,\eta_r)-\min_\eta\mathcal L_\gamma(\theta_r,\eta)\big]
 \le\frac1{2\gamma\mu R}\sum_{r=0}^{R-1}\|\nabla_\eta\mathcal L_\gamma(\theta_r,\eta_r)\|^2.
 \label{eq:app_bridge_lower_gap}
\end{align}
The hidden constants depend on the regularity and flatness constants and $\gamma_0$, but not on $R,T$, or $\gamma$.
\end{theorem}

For fixed $\gamma$ and initialization, the optimization term in
\eqref{eq:app_bridge_convergence} is ${\cal O}(T/R)$, and periodic
retrieval updates do not introduce an additional residual term.
To balance the last two residual terms, we can choose
$\gamma\asymp\bar\delta^{-(2-p)/p}$ with $\gamma\ge\gamma_0$.
For $\bar\delta\asymp\epsilon^{p/2}$ and $0<\epsilon\le1$, this also
satisfies the penalty-size condition
$\gamma\ge\max\{\gamma_0,\gamma_\epsilon\}$, where
$\gamma_\epsilon=O(\epsilon^{-(2-p)/2})$, in
Corollary~\ref{cor:bridge_bilevel_stationarity}. Then, with $R\gtrsim T/\epsilon$ iterations,
Theorem \ref{thm:bridge_convergence} and Corollary~\ref{cor:bridge_bilevel_stationarity}
imply that the idealized BRIDGE updates yield an ${\cal O}(\epsilon^{p-1})$
stationary point of the original bilevel problem
\eqref{eq:ideal_bilevel}.

\begin{proof}
Fix a retriever-update round $r$ and write $r'=\min\{r+T,R\}$. By smoothness and \eqref{eq:app_bridge_stepsizes}, the retriever update satisfies
\begin{align}
 \mathcal L_\gamma(\theta_r,\eta_{r+1})
 &\le\mathcal L_\gamma(\theta_r,\eta_r)
 -\frac b\gamma\left(1-\frac{b(L_{\mathrm{RL}}+\gamma L_{\mathrm{RAG}})}{2\gamma}\right)
 \|\nabla_\eta\mathcal L_\gamma(\theta_r,\eta_r)\|^2\nonumber\\
 &\le\mathcal L_\gamma(\theta_r,\eta_r)-\frac b{2\gamma}\|\nabla_\eta\mathcal L_\gamma(\theta_r,\eta_r)\|^2.
 \label{eq:app_bridge_retriever_descent}
\end{align}
We next bound the error in the policy direction. Smoothness and PL imply the quadratic-growth and error-bound properties \citep[Lemma 4]{jiang2026beyond}. For a nearest RAG minimizer $\eta^*(\theta_t)$ to $\eta_\gamma^*(\theta_t)$, penalized optimality and flatness give
\begin{align}
 \frac{\gamma\mu}{2}\|\eta_\gamma^*(\theta_t)-\eta^*(\theta_t)\|^2
 \le C\|\eta_\gamma^*(\theta_t)-\eta^*(\theta_t)\|^p+\delta_t.
 \label{eq:app_bridge_minimizer_comparison}
\end{align}
Applying Young's inequality yields $\|\eta_\gamma^*(\theta_t)-\eta^*(\theta_t)\|^2
\le O(\gamma^{-2/(2-p)}+\delta_t/\gamma)$. Thus, by smoothness and the envelope identity,
\begin{align}
 \gamma^2\|\nabla_\theta\Lrag(\theta_t,\eta_\gamma^*(\theta_t))-\nabla\Lrag^*(\theta_t)\|^2
 \le O\!\left(\gamma^{-\frac{2(p-1)}{2-p}}+\gamma\delta_t\right).
 \label{eq:app_bridge_omission_error}
\end{align}
For any $\eta$, choose $\eta_\gamma^*(\theta_t)$ nearest to $\eta$. The error bound gives $\|\eta-\eta_\gamma^*(\theta_t)\|\le
(\gamma\mu)^{-1}\|\nabla_\eta\mathcal L_\gamma(\theta_t,\eta)\|$. Combining this with \eqref{eq:app_bridge_omission_error}, smoothness, and Young's inequality gives
\begin{align}
 \|\gamma[\nabla_\theta\Lrag(\theta_t,\eta)-\nabla\Lrag^*(\theta_t)]\|^2
 &\le\frac{2L_{\mathrm{RAG}}^2}{\mu^2}\|\nabla_\eta\mathcal L_\gamma(\theta_t,\eta)\|^2
 +O\!\left(\gamma^{-\frac{2(p-1)}{2-p}}+\gamma\delta_t\right).
 \label{eq:app_bridge_correction_error}
\end{align}
Since $\Lrag^*$ has smoothness constant at most $L_{\mathrm{RAG}}(1+L_{\mathrm{RAG}}/\mu)$, the centered objective $\mathcal L_\gamma-\gamma\Lrag^*$ is smooth in $\theta$ with constant at most $L_{\mathrm{RL}}+\gamma(2L_{\mathrm{RAG}}+L_{\mathrm{RAG}}^2/\mu)$. The policy gradient uses $\eta_t$ rather than $\eta_{t+1}$, and smoothness gives
\[
 \|\nabla_\theta\Jrl(\theta_t,\eta_{t+1})-\nabla_\theta\Jrl(\theta_t,\eta_t)\|
 \le\frac{bL_{\mathrm{RL}}}{\gamma}\mathbf 1_{\{t\bmod T=0\}}
 \|\nabla_\eta\mathcal L_\gamma(\theta_t,\eta_t)\|.
\]
Applying smoothness to the policy update and using $|\langle u,v\rangle|\le\|u\|^2+\|v\|^2/4$ for the omitted correction and $|\langle u,v\rangle|\le2\|u\|^2+\|v\|^2/8$ for this retriever-state discrepancy yields
\begin{align}
 &\mathcal L_\gamma(\theta_{t+1},\eta_{t+1})-\gamma\Lrag^*(\theta_{t+1})
 -\mathcal L_\gamma(\theta_t,\eta_{t+1})+\gamma\Lrag^*(\theta_t)\nonumber\\
 &\quad\le-\frac a2\|\nabla_\theta\Jrl(\theta_t,\eta_t)\|^2
 +\frac{2aL_{\mathrm{RAG}}^2}{\mu^2}\|\nabla_\eta\mathcal L_\gamma(\theta_t,\eta_{t+1})\|^2\nonumber\\
 &\qquad+\frac{2ab^2L_{\mathrm{RL}}^2}{\gamma^2}\mathbf 1_{\{t\bmod T=0\}}
 \|\nabla_\eta\mathcal L_\gamma(\theta_t,\eta_t)\|^2
 +O\!\left(a\gamma^{-\frac{2(p-1)}{2-p}}+a\gamma\delta_t\right).
 \label{eq:app_bridge_policy_descent}
\end{align}

Then, the remaining task is to control the accumulated lower-level gradients while the retriever is fixed. For $r\le t<r'$, $\eta_{t+1}=\eta_{r+1}$ and $\theta_t-\theta_r=-a\sum_{j=r}^{t-1}\nabla_\theta\Jrl(\theta_j,\eta_j)$. Also, smoothness and the retriever update give $\|\nabla_\eta\mathcal L_\gamma(\theta_r,\eta_{r+1})\|\le2\|\nabla_\eta\mathcal L_\gamma(\theta_r,\eta_r)\|$. Therefore,
\begin{align}
 \|\nabla_\eta\mathcal L_\gamma(\theta_t,\eta_{t+1})\|^2
 &\le8\|\nabla_\eta\mathcal L_\gamma(\theta_r,\eta_r)\|^2\nonumber\\
 &\quad+2(L_{\mathrm{RL}}+\gamma L_{\mathrm{RAG}})^2a^2(t-r)
 \sum_{j=r}^{t-1}\|\nabla_\theta\Jrl(\theta_j,\eta_j)\|^2.\nonumber
\end{align}
Summing over $t$ and using $r'-r\le T$ gives
\begin{align}
 \sum_{t=r}^{r'-1}\|\nabla_\eta\mathcal L_\gamma(\theta_t,\eta_{t+1})\|^2
 &\le8T\|\nabla_\eta\mathcal L_\gamma(\theta_r,\eta_r)\|^2\nonumber\\
 &\quad+(L_{\mathrm{RL}}+\gamma L_{\mathrm{RAG}})^2a^2T^2
 \sum_{t=r}^{r'-1}\|\nabla_\theta\Jrl(\theta_t,\eta_t)\|^2.
 \label{eq:app_bridge_period_drift}
\end{align}
The same estimate for $\sum_{t=r}^{r'-1}\|\nabla_\eta\mathcal L_\gamma(\theta_t,\eta_t)\|^2$ holds with $8T$ replaced by $9T$, since only the first retriever state differs.
Adding \eqref{eq:app_bridge_retriever_descent} and \eqref{eq:app_bridge_policy_descent} over the period, and substituting \eqref{eq:app_bridge_period_drift}, yields
\begin{align}
 &\mathcal L_\gamma(\theta_{r'},\eta_{r'})-\gamma\Lrag^*(\theta_{r'})
 -\mathcal L_\gamma(\theta_r,\eta_r)+\gamma\Lrag^*(\theta_r)\nonumber\\
 &\quad\le-\left(\frac b{2\gamma}-\frac{16aTL_{\mathrm{RAG}}^2}{\mu^2}-\frac{2ab^2L_{\mathrm{RL}}^2}{\gamma^2}\right)
 \|\nabla_\eta\mathcal L_\gamma(\theta_r,\eta_r)\|^2\nonumber\\
 &\qquad-a\left(\frac12-\frac{2L_{\mathrm{RAG}}^2(L_{\mathrm{RL}}+\gamma L_{\mathrm{RAG}})^2a^2T^2}{\mu^2}\right)
 \sum_{t=r}^{r'-1}\|\nabla_\theta\Jrl(\theta_t,\eta_t)\|^2\nonumber\\
 &\qquad+O\!\left(a\sum_{t=r}^{r'-1}\left[\gamma^{-\frac{2(p-1)}{2-p}}+\gamma\delta_t\right]\right)\nonumber\\
 &\quad\le-\frac b{4\gamma}\|\nabla_\eta\mathcal L_\gamma(\theta_r,\eta_r)\|^2
 -\frac a4\sum_{t=r}^{r'-1}\|\nabla_\theta\Jrl(\theta_t,\eta_t)\|^2
 +O\!\left(a\sum_{t=r}^{r'-1}\left[\gamma^{-\frac{2(p-1)}{2-p}}+\gamma\delta_t\right]\right),
 \label{eq:app_bridge_period_descent}
\end{align}
where the last inequality follows from \eqref{eq:app_bridge_stepsizes}; in particular, $bL_{\mathrm{RL}}/\gamma\le1$ implies
\[
 \frac{16aTL_{\mathrm{RAG}}^2}{\mu^2}+\frac{2ab^2L_{\mathrm{RL}}^2}{\gamma^2}
 \le\frac b\gamma\frac{16L_{\mathrm{RAG}}^2+2\mu^2/T}{64(L_{\mathrm{RAG}}+\mu)^2}
 \le\frac b{4\gamma}.
\]

Finally, the envelope identity, the error bound, and \eqref{eq:app_bridge_omission_error} imply
\begin{align}
 \|\nabla\mathcal L_\gamma^*(\theta_t)\|^2
 &\le2\|\nabla_\theta\Jrl(\theta_t,\eta_t)\|^2
 +\frac{4L_{\mathrm{RL}}^2}{\gamma^2\mu^2}\|\nabla_\eta\mathcal L_\gamma(\theta_t,\eta_t)\|^2
 +O\!\left(\gamma^{-\frac{2(p-1)}{2-p}}+\gamma\delta_t\right).
 \label{eq:app_bridge_reduced_gradient}
\end{align}
Telescope \eqref{eq:app_bridge_period_descent} over all periods. Its terminal centered objective is bounded below by $\inf_\theta\mathcal L_\gamma^*(\theta)$. By \eqref{eq:app_bridge_stepsizes}, $a\gamma T/b\le1/64$ and $(L_{\mathrm{RL}}+\gamma L_{\mathrm{RAG}})aT\le1/64$; hence \eqref{eq:app_bridge_period_drift} and its $9T$ version control the averaged lower gradients, including the rounds without retriever updates. Combining these bounds with \eqref{eq:app_bridge_reduced_gradient} and $\gamma\ge\gamma_0$ proves \eqref{eq:app_bridge_convergence}. The PL inequality \eqref{eq:app_pl} directly gives \eqref{eq:app_bridge_lower_gap}.
\end{proof}

\begin{proof}[Proof of Corollary~\ref{cor:bridge_bilevel_stationarity}]
Choose a nearest penalized minimizer,
\[
 \eta_\gamma^*(\theta)\in
 \arg\min_{\widetilde\eta\in\arg\min_\eta\mathcal L_\gamma(\theta,\eta)}
 \|\widetilde\eta-\eta_\gamma\|.
\]
By the error bound and \eqref{eq:app_pl},
$\|\eta_\gamma-\eta_\gamma^*(\theta)\|\le\epsilon/(\gamma\mu)$ and
$\mathcal L_\gamma(\theta,\eta_\gamma)-\mathcal L_\gamma(\theta,\eta_\gamma^*(\theta))\le\epsilon^2/(2\gamma\mu)$.
Smoothness and the envelope identity therefore give
\begin{align}
 \|\nabla_\theta\mathcal L_\gamma(\theta,\eta_\gamma)-\gamma\nabla\Lrag^*(\theta)\|
 \le\epsilon+\frac{L_{\mathrm{RL}}+\gamma L_{\mathrm{RAG}}}{\gamma\mu}\epsilon=O(\epsilon).
 \label{eq:app_bridge_inexact_envelope}
\end{align}
For a nearest $\eta^*\in\arg\min_\eta\Lrag(\theta,\eta)$, approximate penalized optimality and flatness imply
\begin{align}
 \frac{\gamma\mu}{2}\|\eta_\gamma-\eta^*\|^2
 \le C\|\eta_\gamma-\eta^*\|^p+\delta+\frac{\epsilon^2}{2\gamma\mu}.
 \label{eq:app_bridge_inexact_distance}
\end{align}
Applying Young's inequality gives $\|\eta_\gamma-\eta^*\|^2
\le O(\gamma^{-2/(2-p)}+\delta/\gamma+\epsilon^2/\gamma^2)$.
With the stated choices of $\gamma$ and $\delta$, this implies
$\|\eta_\gamma-\eta^*\|^2=O(\epsilon)$ and
$\gamma\|\eta_\gamma-\eta^*\|^2=O(\epsilon^{p/2})$.
Taylor expansion at $\eta_\gamma$, the Lipschitz Hessians, and $\nabla_\eta\Lrag(\theta,\eta^*)=0$ yield
\begin{align*}
 \|\nabla_\theta\Jrl+\gamma\nabla_{\theta\eta}^2\Lrag(\eta_\gamma-\eta^*)\|
 &\le O(\epsilon)+O(\gamma\|\eta_\gamma-\eta^*\|^2),\\
 \|\nabla_\eta\Jrl+\gamma\nabla_{\eta\eta}^2\Lrag(\eta_\gamma-\eta^*)\|
 &\le\epsilon+O(\gamma\|\eta_\gamma-\eta^*\|^2),
\end{align*}
where the first bound uses \eqref{eq:app_bridge_inexact_envelope} and the second uses the assumed lower-gradient bound. Their squares are $O(\epsilon^p)$, while smoothness gives
$\|\nabla_\eta\Lrag(\theta,\eta_\gamma)\|^2\le L_{\mathrm{RAG}}^2\|\eta_\gamma-\eta^*\|^2=O(\epsilon)$.
Since $p>1$, summing the three squared residuals proves the result.
\end{proof}

%% file: appendix_zero_reward.tex
\subsection{RAG provides useful learning signals under zero reward}
\label{app:zero_reward_evidence}

\noindent\textbf{Useful evidence can still lead to zero reward.}
In Table~\ref{tab:case_chedid_zero_reward}, we show an example for BRIDGE at the step where it receives zero reward, but still get useful information through the RAG score.

BRIDGE first identifies Matthieu Chedid as the performer of \textit{Je dis aime}, alongside a
distractor about Thomas Dutronc and his father, Jacques Dutronc. A second
search explicitly identifies Louis Chedid as Matthieu Chedid's father.
Yet \method{} answers incorrectly and receives reward zero although the retrieved related document is present.

\begin{caseStudy}
\caseCaption{tab:case_chedid_zero_reward}{\method{}'s failure on MuSiQue Dataset. }
\caseRule
\textbf{Question:} Who is the father of the performer of Je dis aime?\\
\textbf{Ground Truth:} Louis Chedid
\caseRule
\caseHeader{casered}{\method{} (wrong despite useful evidence)}
\begin{lstlisting}[style=trajectory]
<search> Who is the father of the performer of Je dis aime </search>
<information> Doc 1 (Title: "Je dis aime"): ... the second studio album by French singer-songwriter Matthieu Chedid, in his persona as -M- ...
Doc 2 (Title: "Je dis aime"): ... the second studio album by French singer-songwriter Matthieu Chedid, in his persona as -M- ...
Doc 3 (Title: "Thomas Dutronc"): ... his father ... Jacques Dutronc. ... </information>
<search> Matthieu Chedid father </search>
<information> Doc 1 (Title: "Matthieu Chedid"): ... the son of French singer Louis Chedid ...
Doc 2 (Title: "Louis Chedid"): ... Louis Chedid ... the father of Matthieu Chedid (better known as -M-). ...
Doc 3 (Title: "Louis Chedid"): ... Louis Chedid ... the father of Matthieu Chedid (better known as -M-). ... </information>
<answer> Jacques Dutronc </answer>
\end{lstlisting}
\textbf{RL outcome reward:} 0. 
\caseRule
\end{caseStudy}

\noindent\textbf{RAG likelihood can distinguish documents when GRPO cannot.}
At intermediate training step 100, when a retriever update is performed, 85.26\%
audited retrieval turns have zero GRPO advantage. 
For those events, retrieval probabilities $\rho_\eta\left(d \mid h_{i, t}\right)$
alone give a mean maximum of only $0.1093$ over ten
candidates, which means for a failure event, the retrieval scores are almost equally informless. However, after combining with the correct answer likelihood $p_\theta\left(y \mid h_{i, t}, d\right)$, the most favored document gets, on average, $88\%$ of the weight in terms of the answer likelihood-based credit weight $\bar{c}_{i, t, d}$ in \eqref{eq:app_retrieval_credit}, which creates a learning signal for the retrieval model to improve. 

\noindent\textbf{Why the RAG and RL loss complement each other.}
When all rollouts in a group fail, correctness-based GRPO advantages
vanish, but RAG likelihood can still distinguish documents through
\mbox{$p_\theta(y\mid h,d)$}; see \eqref{eq:app_retrieval_score_gradient}. The combined
objectives thus provide document-level learning signals even when successful
rollouts are rare, provided the current model's likelihood scores meaningfully
distinguish useful evidence. This explains the
larger gains over VT-Search with the weaker 3B model than with 7B
(Table~\ref{tab:knowledge_qa_results}): likelihood can remain informative when
correctness-based feedback is sparse for weak reasoning model.

%% file: reference/LLM.bib
@inproceedings{hendrycksmeasuring,
  title = {Measuring Massive Multitask Language Understanding},
  author = {Hendrycks, Dan and Burns, Collin and Basart, Steven and Zou, Andy and Mazeika, Mantas and Song, Dawn and Steinhardt, Jacob},
  booktitle = {Proceedings of the International Conference on Learning Representations},
  year = {2021},
  address = {virtual}
}

@inproceedings{hulora,
  title = {{LoRA}: Low-Rank Adaptation of Large Language Models},
  author = {Hu, Edward J and Shen, Yelong and Wallis, Phillip and Allen-Zhu, Zeyuan and Li, Yuanzhi and Wang, Shean and Wang, Lu and Chen, Weizhu},
  booktitle = {Proceedings of the International Conference on Learning Representations},
  year = {2022},
  address = {virtual}
}

@article{schulman2017proximal,
  title = {Proximal policy optimization algorithms},
  author = {Schulman, John and Wolski, Filip and Dhariwal, Prafulla and Radford, Alec and Klimov, Oleg},
  journal = {arXiv preprint arXiv:1707.06347},
  year = {2017}
}

@inproceedings{shen2024seal,
  title = {{SEAL}: Safety-enhanced Aligned {LLM} Fine-tuning via Bilevel Data Selection},
  author = {Shen, Han and Chen, Pin-Yu and Das, Payel and Chen, Tianyi},
  booktitle = {Proceedings of the International Conference on Learning Representations},
  address={Singapore,Singapore},
  year = {2025}
}

@article{shao2024deepseekmath,
  title={Deepseekmath: Pushing the limits of mathematical reasoning in open language models},
  author={Shao, Zhihong and Wang, Peiyi and Zhu, Qihao and Xu, Runxin and Song, Junxiao and Bi, Xiao and Zhang, Haowei and Zhang, Mingchuan and Li, YK and Wu, Yang and others},
  journal={arXiv preprint arXiv:2402.03300},
  year={2024}
}


%% file: reference/agent.bib
@inproceedings{ahn2024mathematical,
  title = {Large Language Models for Mathematical Reasoning: Progresses and Challenges},
  author = {Ahn, Janice and Verma, Rishu and Lou, Renze and Liu, Di and Zhang, Rui and Yin, Wenpeng},
  booktitle = {Proceedings of the Conference of the European Chapter of the Association for Computational Linguistics: Student Research Workshop},
  address = {St. Julian’s, Malta},
  year = {2024}
}

@inproceedings{chen2026improving,
  title = {Improving retrieval-augmented generation through multi-agent reinforcement learning},
  author = {Chen, Yiqun and Yan, Lingyong and Sun, Weiwei and Ma, Xinyu and Zhang, Yi and Wang, Shuaiqiang and Yin, Dawei and Yang, Yiming and Mao, Jiaxin},
  booktitle = {Advances in Neural Information Processing Systems},
  year = {2025},
  address = {San Diego, CA}
}

@article{chung2024scaling,
  title = {Scaling Instruction-Finetuned Language Models},
  author = {Chung, Hyung Won and Hou, Le and Longpre, Shayne and Zoph, Barret and Tay, Yi and Fedus, William and Li, Yunxuan and Wang, Xuezhi and Dehghani, Mostafa and Brahma, Siddhartha and others},
  journal = {Journal of Machine Learning Research},
  volume = {25},
  number = {70},
  pages = {1--53},
  year = {2024}
}

@inproceedings{gao2025smartrag,
  title = {{SmartRAG}: Jointly Learn {RAG}-Related Tasks From the Environment Feedback},
  author = {Gao, Jingsheng and Li, Linxu and Ji, Ke and Li, Weiyuan and Lian, Yixin and Fu, Yuzhuo and Dai, Bin},
  booktitle = {Proceedings of the International Conference on Learning Representations},
  year = {2025},
  address = {Singapore, Singapore}
}

@article{guo2025deepseekr1,
  title = {{DeepSeek-R1} incentivizes reasoning in {LLMs} through reinforcement learning},
  author = {Guo, Daya and Yang, Dejian and Zhang, Haowei and Song, Junxiao and Wang, Peiyi and Zhu, Qihao and Xu, Runxin and Zhang, Ruoyu and Ma, Shirong and Bi, Xiao and others},
  journal = {Nature},
  volume = {645},
  pages = {633--638},
  year = {2025}
}

@inproceedings{ho2020constructing,
  title = {Constructing A Multi-hop {QA} Dataset for Comprehensive Evaluation of Reasoning Steps},
  author = {Ho, Xanh and Duong Nguyen, Anh-Khoa and Sugawara, Saku and Aizawa, Akiko},
  booktitle = {Proceedings of the International Conference on Computational Linguistics},
  address={virtual},
  year = {2020}
}

@inproceedings{huang2025ragrl,
  title = {Tackling Distractor Documents in Multi-Hop {QA} with Reinforcement and Curriculum Learning},
  author = {Huang, Jerry and Madala, Siddarth and Sidhu, Risham and Niu, Cheng and Peng, Hao and Hockenmaier, Julia and Zhang, Tong},
  booktitle = {Findings of the Association for Computational Linguistics: EACL},
  year = {2026},
  address = {Rabat, Morocco}
}

@inproceedings{jiang2025s3,
  title = {{s3}: You Don't Need That Much Data to Train a Search Agent via {RL}},
  author = {Jiang, Pengcheng and Xu, Xueqiang and Lin, Jiacheng and Xiao, Jinfeng and Wang, Zifeng and Sun, Jimeng and Han, Jiawei},
  booktitle = {Proceedings of the Conference on Empirical Methods in Natural Language Processing},
  year = {2025},
  address = {Suzhou, China}
}

@article{jiang2025verltool,
  title = {{VerlTool}: Towards Holistic Agentic Reinforcement Learning with Tool Use},
  author = {Jiang, Dongfu and Lu, Yi and Li, Zhuofeng and Lyu, Zhiheng and Nie, Ping and Wang, Haozhe and Su, Alex and Chen, Hui and Zou, Kai and Du, Chao and Pang, Tianyu and Chen, Wenhu},
  journal = {Transactions on Machine Learning Research},
  year = {2026}
}

@inproceedings{jin2019pubmedqa,
  title = {{PubMedQA}: A Dataset for Biomedical Research Question Answering},
  author = {Jin, Qiao and Dhingra, Bhuwan and Liu, Zhengping and Cohen, William W. and Lu, Xinghua},
  booktitle = {Proceedings of the Conference on Empirical Methods in Natural Language Processing and the International Joint Conference on Natural Language Processing},
  address = {Hong Kong, China},
  year = {2019}
}

@article{jin2021disease,
  title = {What Disease Does This Patient Have? {A} Large-Scale Open Domain Question Answering Dataset from Medical Exams},
  author = {Jin, Di and Pan, Eileen and Oufattole, Nassim and Weng, Wei-Hung and Fang, Hanyi and Szolovits, Peter},
  journal = {Applied Sciences},
  volume = {11},
  number = {14},
  pages = {6421},
  year = {2021}
}

@article{jin2023medcpt,
  title = {{MedCPT}: Contrastive Pre-trained Transformers with large-scale {PubMed} search logs for zero-shot biomedical information retrieval},
  author = {Jin, Qiao and Kim, Won and Chen, Qingyu and Comeau, Donald C and Yeganova, Lana and Wilbur, W John and Lu, Zhiyong},
  journal = {Bioinformatics},
  volume = {39},
  number = {11},
  pages = {btad651},
  year = {2023}
}

@inproceedings{jin2025searchr1,
  title = {{Search-R1}: Training {LLMs} to Reason and Leverage Search Engines with Reinforcement Learning},
  author = {Jin, Bowen and Zeng, Hansi and Yue, Zhenrui and Yoon, Jinsung and Ar{\i}k, Sercan {\"O}. and Wang, Dong and Zamani, Hamed and Han, Jiawei},
  booktitle = {Proceedings of the Conference on Language Modeling},
  year = {2025},
  address = {Montreal, Canada}
}

@inproceedings{joshi2017triviaqa,
  title = {{TriviaQA}: A Large Scale Distantly Supervised Challenge Dataset for Reading Comprehension},
  author = {Joshi, Mandar and Choi, Eunsol and Weld, Daniel S. and Zettlemoyer, Luke},
  booktitle = {Proceedings of the Annual Meeting of the Association for Computational Linguistics},
  address = {Vancouver, Canada},
  year = {2017}
}

@inproceedings{karpukhin2020dense,
  title = {Dense passage retrieval for open-domain question answering},
  author = {Karpukhin, Vladimir and Oguz, Barlas and Min, Sewon and Lewis, Patrick and Wu, Ledell and Edunov, Sergey and Chen, Danqi and Yih, Wen-tau},
  booktitle = {Proceedings of the Conference on Empirical Methods in Natural Language Processing},
  address={virtual},
  year = {2020}
}

@article{krithara2023bioasq,
  title = {{BioASQ-QA}: A manually curated corpus for Biomedical Question Answering},
  author = {Krithara, Anastasia and Nentidis, Anastasios and Bougiatiotis, Konstantinos and Paliouras, Georgios},
  journal = {Scientific Data},
  volume = {10},
  number = {1},
  pages = {170},
  year = {2023}
}

@article{kwiatkowski2019natural,
  title = {Natural Questions: A Benchmark for Question Answering Research},
  author = {Kwiatkowski, Tom and Palomaki, Jennimaria and Redfield, Olivia and Collins, Michael and Parikh, Ankur and Alberti, Chris and Epstein, Danielle and Polosukhin, Illia and Devlin, Jacob and Lee, Kenton and others},
  journal = {Transactions of the Association for Computational Linguistics},
  volume = {7},
  pages = {453--466},
  year = {2019}
}

@inproceedings{lewis2021rag,
  title = {{Retrieval-Augmented Generation for Knowledge-Intensive NLP Tasks}},
  author = {Lewis, Patrick and Perez, Ethan and Piktus, Aleksandra and Petroni, Fabio and Karpukhin, Vladimir and Goyal, Naman and K{\"u}ttler, Heinrich and Lewis, Mike and Yih, Wen-tau and Rockt{\"a}schel, Tim and others},
  booktitle = {Advances in Neural Information Processing Systems},
  year = {2020},
  address = {virtual}
}

@inproceedings{li2025agentflow,
  title = {In-the-Flow Agentic System Optimization for Effective Planning and Tool Use},
  author = {Li, Zhuofeng and Zhang, Haoxiang and Han, Seungju and Liu, Sheng and Xie, Jianwen and Zhang, Yu and Choi, Yejin and Zou, James and Lu, Pan},
  booktitle = {Proceedings of the International Conference on Learning Representations},
  year = {2026},
  address = {Rio de Janeiro, Brazil}
}

@inproceedings{li2025searcho1,
  title = {{Search-o1}: Agentic Search-Enhanced Large Reasoning Models},
  author = {Li, Xiaoxi and Dong, Guanting and Jin, Jiajie and Zhang, Yuyao and Zhou, Yujia and Zhu, Yutao and Zhang, Peitian and Dou, Zhicheng},
  booktitle = {Proceedings of the Conference on Empirical Methods in Natural Language Processing},
  year = {2025},
  address = {Suzhou, China}
}

@article{li2026dart,
  title = {Reasoning and Tool-use Compete in Agentic {RL}: From Quantifying Interference to Disentangled Tuning},
  author = {Li, Yu and Yi, Mingyang and Li, Xiuyu and Fan, Ju and Jiang, Fuxin and Chen, Binbin and Li, Peng and Song, Jie and Zhang, Tieying},
  journal = {arXiv preprint arXiv:2602.00994},
  year = {2026}
}

@inproceedings{lin2024radit,
  title = {{RA-DIT}: Retrieval-Augmented Dual Instruction Tuning},
  author = {Lin, Xi Victoria and Chen, Xilun and Chen, Mingda and Shi, Weijia and Lomeli, Maria and James, Richard and Rodriguez, Pedro and Kahn, Jacob and Szilvasy, Gergely and Lewis, Mike and others},
  booktitle = {Proceedings of the International Conference on Learning Representations},
  year = {2024},
  address = {Vienna, Austria}
}

@inproceedings{liu2026agenticr,
  title = {{Agentic-R}: Learning to Retrieve for Agentic Search},
  author = {Liu, Wenhan and Ma, Xinyu and Zhu, Yutao and Li, Yuchen and Shi, Daiting and Yin, Dawei and Dou, Zhicheng},
  booktitle = {Findings of the Association for Computational Linguistics: ACL},
  address = {San Diego, CA},
  year = {2026},
}

@inproceedings{luo2023sail,
  title = {Search Augmented Instruction Learning},
  author = {Luo, Hongyin and Zhang, Tianhua and Chuang, Yung-Sung and Gong, Yuan and Kim, Yoon and Wu, Xixin and Meng, Helen and Glass, James},
  booktitle = {Findings of the Association for Computational Linguistics: EMNLP},
  year = {2023},
  address = {Singapore, Singapore}
}

@inproceedings{mallen2023trust,
  title = {When Not to Trust Language Models: Investigating Effectiveness of Parametric and Non-Parametric Memories},
  author = {Mallen, Alex and Asai, Akari and Zhong, Victor and Das, Rajarshi and Khashabi, Daniel and Hajishirzi, Hannaneh},
  booktitle = {Proceedings of the Annual Meeting of the Association for Computational Linguistics},
  address = {Toronto, Canada},
  year = {2023}
}

@inproceedings{pal2022medmcqa,
  title = {{MedMCQA}: A Large-scale Multi-Subject Multi-Choice Dataset for Medical domain Question Answering},
  author = {Pal, Ankit and Umapathi, Logesh Kumar and Sankarasubbu, Malaikannan},
  booktitle = {Proceedings of the Conference on Health, Inference, and Learning},
  address={virtual},
  year = {2022},
}

@inproceedings{press2023measuring,
  title = {Measuring and Narrowing the Compositionality Gap in Language Models},
  author = {Press, Ofir and Zhang, Muru and Min, Sewon and Schmidt, Ludwig and Smith, Noah A. and Lewis, Mike},
  booktitle = {Findings of the Association for Computational Linguistics: EMNLP},
  address = {Singapore, Singapore},
  year = {2023}
}

@article{qwen2024technical,
  title = {{Qwen2.5} Technical Report},
  author = {Yang, An and Yang, Baosong and Zhang, Beichen and Hui, Binyuan and Zheng, Bo and Yu, Bowen and Li, Chengyuan and Liu, Dayiheng and Huang, Fei and Wei, Haoran and others},
  journal = {arXiv preprint arXiv:2412.15115},
  year = {2024}
}

@inproceedings{schick2023toolformer,
  title = {{Toolformer}: Language Models Can Teach Themselves to Use Tools},
  author = {Schick, Timo and Dwivedi-Yu, Jane and Dessi, Roberto and Raileanu, Roberta and Lomeli, Maria and Hambro, Eric and Zettlemoyer, Luke and Cancedda, Nicola and Scialom, Thomas},
  booktitle = {Advances in Neural Information Processing Systems},
  year = {2023},
  address = {New Orleans, LA}
}

@inproceedings{shi2023replug,
  title = {{REPLUG}: Retrieval-Augmented Black-Box Language Models},
  author = {Shi, Weijia and Min, Sewon and Yasunaga, Michihiro and Seo, Minjoon and James, Richard and Lewis, Mike and Zettlemoyer, Luke and Yih, Wen-tau},
  booktitle = {Proceedings of the Conference of the North American Chapter of the Association for Computational Linguistics: Human Language Technologies},
  year = {2024},
  address = {Mexico City, Mexico}
}

@article{shi2025directrao,
  title={Direct retrieval-augmented optimization: Synergizing knowledge selection and language models},
  author={Shi, Zhengliang and Yan, Lingyong and Sun, Weiwei and Feng, Yue and Ren, Pengjie and Ma, Xinyu and Wang, Shuaiqiang and Yin, Dawei and de Rijke, Maarten and Ren, Zhaochun},
  journal={ACM Transactions on Information Systems},
  volume={44},
  number={4},
  pages={94:1--94:30},
  year={2026}
}

@article{song2025r1searcher,
  title = {{R1-Searcher}: Incentivizing the Search Capability in {LLMs} via Reinforcement Learning},
  author = {Song, Huatong and Jiang, Jinhao and Min, Yingqian and Chen, Jie and Chen, Zhipeng and Zhao, Wayne Xin and Fang, Lei and Wen, Ji-Rong},
  journal = {arXiv preprint arXiv:2503.05592},
  year = {2025}
}

@inproceedings{sun2025dynamicrag,
  title = {{DynamicRAG}: Leveraging Outputs of Large Language Model as Feedback for Dynamic Reranking in Retrieval-Augmented Generation},
  author = {Sun, Jiashuo and Zhong, Xianrui and Zhou, Sizhe and Han, Jiawei},
  booktitle = {Advances in Neural Information Processing Systems},
  year = {2025},
  address = {San Diego, CA}
}

@article{team2025kimi,
  title = {{Kimi K2}: Open Agentic Intelligence},
  author = {{Kimi Team} and Bai, Yifan and Bao, Yiping and Charles, Y. and Chen, Cheng and Chen, Guanduo and Chen, Haiting and Chen, Huarong and Chen, Jiahao and Chen, Ningxin and others},
  journal = {arXiv preprint arXiv:2507.20534},
  year = {2025}
}

@article{trivedi2022musique,
  title = {{MuSiQue}: Multihop Questions via Single-hop Question Composition},
  author = {Trivedi, Harsh and Balasubramanian, Niranjan and Khot, Tushar and Sabharwal, Ashish},
  journal = {Transactions of the Association for Computational Linguistics},
  volume = {10},
  pages = {539--554},
  year = {2022}
}

@inproceedings{trivedi2023interleaving,
  title = {Interleaving retrieval with chain-of-thought reasoning for knowledge-intensive multi-step questions},
  author = {Trivedi, Harsh and Balasubramanian, Niranjan and Khot, Tushar and Sabharwal, Ashish},
  booktitle = {Proceedings of the Annual Meeting of the Association for Computational Linguistics},
  address = {Toronto, Canada},
  year = {2023}
}

@article{tsatsaronis2015overview,
  title = {An overview of the {BIOASQ} large-scale biomedical semantic indexing and question answering competition},
  author = {Tsatsaronis, George and Balikas, Georgios and Malakasiotis, Prodromos and Partalas, Ioannis and Zschunke, Matthias and Alvers, Michael R. and Weissenborn, Dirk and Krithara, Anastasia and Petridis, Sergios and Polychronopoulos, Dimitris and others},
  journal = {BMC Bioinformatics},
  volume = {16},
  pages = {138},
  year = {2015}
}

@article{wang2022text,
  title = {Text Embeddings by Weakly-Supervised Contrastive Pre-training},
  author = {Wang, Liang and Yang, Nan and Huang, Xiaolong and Jiao, Binxing and Yang, Linjun and Jiang, Daxin and Majumder, Rangan and Wei, Furu},
  journal = {arXiv preprint arXiv:2212.03533},
  year = {2022}
}

@article{xie2026quest,
  title = {{QUEST}: Training Frontier Deep Research Agents with Fully Synthetic Tasks},
  author = {Xie, Jian and Lin, Tianhe and Wang, Zilu and Ning, Yuting and Yao, Yuekun and Xue, Tianci and Zhang, Zhehao and Li, Zhongyang and Zhang, Kai and Wu, Yufan and others},
  journal = {arXiv preprint arXiv:2605.24218},
  year = {2026}
}

@inproceedings{xiong2024benchmarking,
  title = {Benchmarking retrieval-augmented generation for medicine},
  author = {Xiong, Guangzhi and Jin, Qiao and Lu, Zhiyong and Zhang, Aidong},
  booktitle = {Findings of the Association for Computational Linguistics: ACL},
  address = {Bangkok, Thailand},
  year = {2024}
}

@inproceedings{yang2018hotpotqa,
  title = {{HotpotQA}: A Dataset for Diverse, Explainable Multi-hop Question Answering},
  author = {Yang, Zhilin and Qi, Peng and Zhang, Saizheng and Bengio, Yoshua and Cohen, William W. and Salakhutdinov, Ruslan and Manning, Christopher D.},
  booktitle = {Proceedings of the Conference on Empirical Methods in Natural Language Processing},
  address = {Brussels, Belgium},
  year = {2018}
}

@inproceedings{yao2023react,
  title = {{ReAct}: Synergizing Reasoning and Acting in Language Models},
  author = {Yao, Shunyu and Zhao, Jeffrey and Yu, Dian and Du, Nan and Shafran, Izhak and Narasimhan, Karthik and Cao, Yuan},
  booktitle = {Proceedings of the International Conference on Learning Representations},
  year = {2023},
  address = {Kigali, Rwanda}
}

@inproceedings{zhang2024raft,
  title = {{RAFT}: Adapting Language Model to Domain Specific {RAG}},
  author = {Zhang, Tianjun and Patil, Shishir G. and Jain, Naman and Shen, Sheng and Zaharia, Matei and Stoica, Ion and Gonzalez, Joseph E.},
  booktitle = {Proceedings of the Conference on Language Modeling},
  year = {2024},
  address = {Philadelphia, PA}
}

@article{zhou2025automating,
  title={Automating expert-level medical reasoning evaluation of large language models},
  author={Zhou, Shuang and Xie, Wenya and Li, Jiaxi and Zhan, Zaifu and Song, Meijia and Yang, Han and Espinoza, Cheyenna and Welton, Lindsay and Mai, Xinnie and Jin, Yanwei and others},
  journal={NPJ Digital Medicine},
  volume={9},
  number={1},
  pages={34},
  year={2025}
}

@inproceedings{ma2023query,
  title={Query rewriting in retrieval-augmented large language models},
  author={Ma, Xinbei and Gong, Yeyun and He, Pengcheng and Zhao, Hai and Duan, Nan},
  booktitle = {Proceedings of the Conference on Empirical Methods in Natural Language Processing},
  year={2023},
  address={Singapore}
}

@article{ram2023context,
  title={In-context retrieval-augmented language models},
  author={Ram, Ori and Levine, Yoav and Dalmedigos, Itay and Muhlgay, Dor and Shashua, Amnon and Leyton-Brown, Kevin and Shoham, Yoav},
  journal={Transactions of the Association for Computational Linguistics},
  volume={11},
  pages={1316--1331},
  year={2023}
}


%% file: reference/bilevel.bib
@inproceedings{chen2021closing,
  title = {Closing the Gap: Tighter Analysis of Alternating Stochastic Gradient Methods for Bilevel Problems},
  author = {Chen, Tianyi and Sun, Yuejiao and Yin, Wotao},
  booktitle = {Advances in Neural Information Processing Systems},
  year = {2021},
  address = {virtual}
}

@book{dontchev2014implicit,
  title = {Implicit Functions and Solution Mappings: A View from Variational Analysis},
  author = {Dontchev, Asen L and Rockafellar, R Tyrrell},
  edition = {Second},
  year = {2014},
  publisher = {Springer}
}

@inproceedings{ji2021bilevel,
  title = {Bilevel optimization: Convergence analysis and enhanced design},
  author = {Ji, Kaiyi and Yang, Junjie and Liang, Yingbin},
  booktitle = {Proceedings of the International Conference on Machine Learning},
  year = {2021},
  address = {virtual}
}

@inproceedings{jiang2026beyond,
  title = {Beyond value functions: Single-loop bilevel optimization under flatness conditions},
  author = {Jiang, Liuyuan and Xiao, Quan and Chen, Lisha and Chen, Tianyi},
  booktitle = {Advances in Neural Information Processing Systems},
  year = {2025},
  address = {San Diego, CA}
}

@inproceedings{kwon2023fully,
  title = {A fully first-order method for stochastic bilevel optimization},
  author = {Kwon, Jeongyeol and Kwon, Dohyun and Wright, Stephen and Nowak, Robert D},
  booktitle = {Proceedings of the International Conference on Machine Learning},
  year = {2023},
  address = {Honolulu, HI}
}

@inproceedings{kwon2023penalty,
  title = {On Penalty Methods for Nonconvex Bilevel Optimization and First-Order Stochastic Approximation},
  author = {Kwon, Jeongyeol and Kwon, Dohyun and Wright, Stephen and Nowak, Robert},
  booktitle = {Proceedings of the International Conference on Learning Representations},
  year = {2024},
  address = {Vienna, Austria}
}

@inproceedings{li2018visualizing,
  title = {Visualizing the loss landscape of neural nets},
  author = {Li, Hao and Xu, Zheng and Taylor, Gavin and Studer, Christoph and Goldstein, Tom},
  booktitle = {Advances in Neural Information Processing Systems},
  year = {2018},
  address = {Montreal, Canada}
}

@inproceedings{lorraine2020optimizing,
  title = {Optimizing millions of hyperparameters by implicit differentiation},
  author = {Lorraine, Jonathan and Vicol, Paul and Duvenaud, David},
  booktitle = {Proceedings of the International Conference on Artificial Intelligence and Statistics},
  year = {2020},
  address = {virtual}
}

@article{scellier2017equilibrium,
  title = {Equilibrium propagation: Bridging the gap between energy-based models and backpropagation},
  author = {Scellier, Benjamin and Bengio, Yoshua},
  journal = {Frontiers in Computational Neuroscience},
  volume = {11},
  pages = {24},
  year = {2017},
  publisher = {Frontiers Media SA}
}

@article{shen2023penalty,
  title = {On penalty-based bilevel gradient descent method},
  author = {Shen, Han and Xiao, Quan and Chen, Tianyi},
  journal = {Mathematical Programming},
  volume = {214},
  pages = {539--589},
  year = {2025},
  publisher = {Springer}
}

@inproceedings{xiao2025ldc,
  title     = {{LDC-MTL}: Balancing Multi-Task Learning Through Scalable Loss Discrepancy Control},
  author    = {Xiao, Peiyao and Dong, Chaosheng and Zou, Shaofeng and Ji, Kaiyi},
  booktitle = {Proceedings of the European Conference on Computer Vision},
  address={Malm{\"o}, Sweden},
  year      = {2026}
}

@article{zucchet2022beyond,
  title = {Beyond backpropagation: bilevel optimization through implicit differentiation and equilibrium propagation},
  author = {Zucchet, Nicolas and Sacramento, Jo{\~a}o},
  journal = {Neural Computation},
  volume = {34},
  number = {12},
  pages = {2309--2346},
  year = {2022}
}
